\documentclass{article}

\PassOptionsToPackage{numbers,compress}{natbib}

\usepackage[main, final]{neurips_2026}

\usepackage[utf8]{inputenc} 
\usepackage[T1]{fontenc}    
\usepackage{hyperref}       
\usepackage{url}            
\usepackage{booktabs}       
\usepackage{amsmath}        
\usepackage{amsfonts}       
\usepackage{amsthm}         
\usepackage{nicefrac}       
\usepackage{microtype}      
\usepackage{xcolor}         

\usepackage{algorithm}
\usepackage{algcompatible}
\usepackage{xcolor}

\usepackage{xcolor}
\definecolor{myblue}{HTML}{3585C9}
\newcommand{\blue}[1]{\textcolor{myblue}{#1}}

\definecolor{plotgray}{HTML}{333333}
\definecolor{plotgrayobservations}{HTML}{cccccc}
\definecolor{plotportfolio}{HTML}{97DFFC}
\definecolor{plotorange}{HTML}{F9C47A}
\definecolor{plotblue}{HTML}{3585C9}
\definecolor{plotmagenta}{HTML}{A00442}
\definecolor{plotbaseline}{HTML}{4444b4}

\usepackage[normalem]{ulem} 
\usepackage{tikz}
\usetikzlibrary{arrows.meta,calc}
\usepackage{pgfplots}
\usepackage[utf8]{inputenc}
\usepackage[T1]{fontenc}
\usepackage{subcaption}
\usepackage{wrapfig}
\usepackage{placeins} 
\usepackage{array}
\usepackage{tabularx}
\usepackage{booktabs}
\usepackage{amssymb}
\usepackage{pifont}
\newcommand{\cmark}{\ding{51}}   
\newcommand{\approxmark}{$\sim$}

\newcommand{\oldC}{\mathcal{C}}
\newcommand{\bonf}{\text{bonf}}

\newcommand{\vect}[1]{\mathbf{#1}}

\newcommand{\xtL}{x^\text{touch}_{L}}
\newcommand{\xtR}{x^\text{touch}_{R}}
\newcommand{\xtLd}{x^\text{touch}_{L, d}}
\newcommand{\xtRd}{x^\text{touch}_{R, d}}

\newcommand{\xstart}{x_{\text{start}}}
\newcommand{\xend}{x_{\text{end}}}
\newcommand{\xbetaleft}{x_{L}^\beta}
\newcommand{\xbetaright}{x_{R}^\beta}

\newcommand{\Bbbox}{\oldC^\text{bbox}}
\newcommand{\Bellip}{\oldC^{\ell_p}}
\newcommand{\BellipS}{\oldC^{\ell_p\mathrm{safe}}}

\newcommand{\Bmax}{\oldC^{\max}}
\newcommand{\Bavg}{\oldC^\mathrm{port}}
\newcommand{\Bbonf}{\oldC^\bonf}

\newcommand{\Bbboxellip}{\oldC^{\text{bbox } \cap ~ \ell_p}}
\newcommand{\BbboxellipS}{\oldC^{{\text{bbox } \cap ~ \ell_p\mathrm{safe}}}}

\newcommand{\Kavg}{K^{\mathrm{port}}}
\newcommand{\Kmax}{K^{\max}}

\usepackage[capitalize]{cleveref}
\usepackage[acronym]{glossaries-extra}

\crefname{section}{\S}{\S\S}
\Crefname{section}{\S}{\S\S}
\crefname{subsection}{\S}{\S\S}
\Crefname{subsection}{\S}{\S\S}
\crefname{subsubsection}{\S}{\S\S}
\Crefname{subsubsection}{\S}{\S\S}
\crefrangeformat{section}{\S\S#3#1#4--#5#2#6}
\Crefrangeformat{section}{\S\S#3#1#4--#5#2#6}
\crefrangeformat{subsection}{\S\S#3#1#4--#5#2#6}
\Crefrangeformat{subsection}{\S\S#3#1#4--#5#2#6}
\crefrangeformat{subsubsection}{\S\S#3#1#4--#5#2#6}
\Crefrangeformat{subsubsection}{\S\S#3#1#4--#5#2#6}

\crefname{equation}{eq.}{eqs.}
\Crefname{equation}{Eq.}{Eqs.}
\crefname{figure}{fig.}{figs.}
\Crefname{figure}{Fig.}{Figs.}
\crefname{table}{tab.}{tabs.}
\Crefname{table}{Tab.}{Tabs.}
\crefname{algorithm}{alg.}{algs.}
\Crefname{algorithm}{Alg.}{Algs.}
\crefname{appendix}{app.}{apps.}
\Crefname{appendix}{App.}{Apps.}

\crefname{proposition}{prop.}{props.}
\Crefname{proposition}{Prop.}{Props.}
\crefname{theorem}{thm.}{thms.}
\Crefname{theorem}{Thm.}{Thms.}
\crefname{lemma}{lem.}{lems.}
\Crefname{lemma}{Lem.}{Lems.}
\crefname{corollary}{cor.}{cors.}
\Crefname{corollary}{Cor.}{Cors.}
\crefname{definition}{def.}{defs.}
\Crefname{definition}{Def.}{Defs.}
\crefname{assumption}{asm.}{asms.}
\Crefname{assumption}{Asm.}{Asms.}
\crefname{remark}{rem.}{rems.}
\Crefname{remark}{Rem.}{Rems.}

\theoremstyle{plain}
\newtheorem{theorem}{Theorem}[section]
\newtheorem{proposition}{Proposition}[section]
\newtheorem{lemma}{Lemma}[section]
\newtheorem{corollary}{Corollary}[section]
\theoremstyle{definition}

\newtheorem{assumption}{Assumption}[section]
\theoremstyle{remark}

\crefalias{proposition}{proposition}
\crefalias{lemma}{lemma}
\crefalias{corollary}{corollary}
\crefalias{definition}{definition}
\crefalias{assumption}{assumption}
\crefalias{remark}{remark}

\title{On the Tightness and Computational Tractability of Higher-Dimensional Confidence Sequences}

\author{%
  Fabian Denoodt\thanks{Research conducted during a research visit at the University of Amsterdam and completed at Eindhoven University of Technology.} \\
  Eindhoven University of Technology \\
  \And
  Sibylle Hess \\
  Eindhoven University of Technology \\
  \And
  Joaquin Vanschoren \\
  Eindhoven University of Technology \\
  \And
  Christian A. Naesseth \\
  University of Amsterdam \\
}

\begin{document}

\maketitle

\begin{abstract}
  Modern sequential monitoring problems often involve multiple metrics, where we monitor several data streams simultaneously and may act once the evidence is strong enough.
Confidence sequences (CSs) are a natural tool for such continuous monitoring. However, for bounded \emph{vector} means, existing multivariate CSs are either tight but computationally intractable, or fast to compute but conservative.
To address this, we study three lifts of one-dimensional betting-based CSs to higher dimensions: a weighted Bonferroni region, an equivalent max-wealth form, and a portfolio region. The portfolio is typically much tighter, especially in higher dimensions, but its boundary and properties such as volume are not available in closed form.
To make this tighter construction usable, we propose tractable outer approximations of the portfolio region that preserve statistical validity: a bounding box, an $\ell_p$-ellipsoid, and their intersection. We prove set relations among all constructions and show empirically that these approximations \blue{(i)} achieve regions close to the intractable portfolio, \blue{(ii)} substantially outperform existing tractable multivariate CSs, and \blue{(iii)} enable practical use cases such as multi-metric A/B testing and model comparison.

\end{abstract}

\section{Introduction} \label{sec:AS:introduction}

Many modern decision problems are inherently sequential: we monitor the data as it arrives and want to make a decision as soon as the evidence is strong enough. In A/B testing~\citep{maharaj2023anytime_enterprise_ab, johari2015always}, for example, one may compare two versions of a website and stop collecting data once there is reliable evidence that one variant performs better. This kind of optional stopping requires statistical tools that remain valid regardless of when we stop. \emph{Confidence sequences} (CSs) provide exactly this guarantee: a CS for a mean $\mu$ is a sequence of confidence sets that contains $\mu$ at \emph{all} times $t$ with a prescribed probability~\citep{ryu_gamingbasedconfseq}.

While much of the existing literature focuses on CSs for a \emph{scalar} mean, many modern applications involve a \emph{vector} mean $\boldsymbol{\mu}\in\mathbb{R}^D$. For instance, in website A/B testing, one may monitor several metrics simultaneously, such as browsing time, purchases, and page views. Similarly, when evaluating a new ML model's accuracy against an already deployed one, one may also monitor subgroup accuracies to detect possible subgroup degradations.
To be useful in such settings, a multivariate CS should satisfy two requirements: it should be \emph{tight}, meaning that its sets shrink rapidly as more data are observed (enabling early decisions), and \emph{computationally tractable}, allowing for real-time monitoring in moderate to high dimensions.

Achieving both requirements simultaneously is challenging. If one focuses on tightness, the multivariate CSs of~\citet{ryu_gamingbasedconfseq} provide a strong benchmark: they can be tight, but lack a closed-form solution for the region they define.
They also have a computational cost of $O(t^D)$ for verifying whether a vector belongs to the CS at time $t$~\citep{martinez2024empirical_bernstein_banach_spaces}, limiting their practical use. On the other hand, closed-form regions such as confidence spheres~\citep{chugg2023confsphere, martinez2024empirical_bernstein_banach_spaces} are easy to compute, but they expand and shrink in all directions together, which is wasteful when some coordinates have low variance and others do not, as shown by the baseline in \Cref{fig:AS:overview-approx}. 
Meanwhile, the ellipsoids by \citet{whitehouse2023time_uniform_self_normalized_concentration_vector_valued_processes} can adapt to the per-axis variance but target more general vector processes, whereas we specialize to bounded observations and exploit this structure for smaller regions.

\begin{figure*}[t]
	\centering
	\includegraphics[
		width=\textwidth
	]{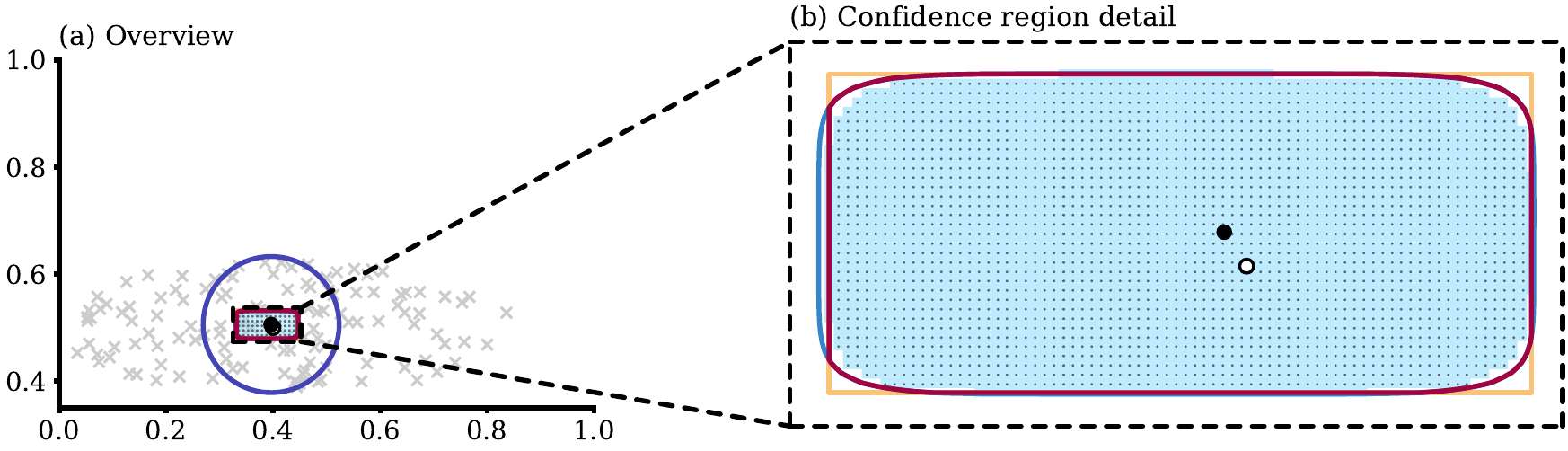}
	\caption{
		Intractable CS $\Bavg$ (\protect\tikz[baseline=0ex]{\fill[plotportfolio] (0,0) rectangle (1.2ex,1.2ex); \fill[plotgray] (0.6ex,0.6ex) circle (0.16ex);}) and its computable superset approximations:
		the bounding box $\Bbbox$ ({\color{plotorange}\rule[0.6ex]{0.9em}{1.0pt}}),
		the $\ell_p$-ellipsoid $\Bellip$ ({\color{plotblue}\rule[0.6ex]{0.9em}{1.0pt}}),
		and their intersection $\Bbboxellip$ ({\color{plotmagenta}\rule[0.6ex]{0.9em}{1.0pt}}).
		The figure also shows
		observations ({\color{plotgrayobservations}$\times$}),
		the population mean ($\circ$),
		the empirical mean ($\bullet$),
		and a previously strong baseline~\protect\citep{martinez2024empirical_bernstein_banach_spaces} ({\color{plotbaseline}\rule[0.6ex]{0.9em}{1.0pt}}).
	}
	\label{fig:AS:overview-approx}
\end{figure*}

Motivated by recent progress on \emph{one-dimensional} CSs~\citep{howard2021timeuniformnonparamnonasmptotic_conf_seq, waudby2024estimatingmean}, we lift the betting-based CS of~\citet{waudby2024estimatingmean} to bounded \emph{vectors}. A simple approach is to build a 1D CS for each coordinate and adjust the significance level using a multiple-testing correction, resulting in a weighted Bonferroni hyperrectangle. This construction is easy to compute, but becomes increasingly conservative as $D$ grows.
We therefore also consider a tighter multivariate construction that aggregates the coordinate-wise wealth processes of~\citep{waudby2024estimatingmean} through a weighted average, which we call a portfolio. The portfolio CS can be substantially tighter, but is less practical since the boundary is not described explicitly, and properties such as volume or maximum width are not easily available in closed form.

\textbf{This paper addresses that gap.} 
\textbf{\blue{(i)}} We treat the portfolio as a tight but expensive target and derive three \emph{outer} approximations that scale linearly with the dimension: a bounding box, an $\ell_p$-ellipsoid, and their intersection.
\textbf{\blue{(ii)}} We prove set relations among all constructions (see \Cref{fig:AS:overview-venn}) and provide the necessary theory to ensure that our approximations remain statistically valid, even \emph{after implementation}.  
\textbf{\blue{(iii)}} Experiments show that our constructions \textbf{consistently outperform existing tractable CSs}, resulting in substantially smaller regions, particularly as $D$ grows.  Two case studies further illustrate their practical value by showing that they can detect statistical significance with substantially fewer observations\footnote{Code and a practical tutorial on using our CSs to monitor ML models are available \href{https://fdenoodt.github.io/higher-dimensional-confidence-sequences/}{\underline{here}}.}.

\textbf{Roadmap.} Building on the univariate CSs described in \Cref{sec:AS:one_dim_conf_seq}, \Cref{sec:AS:high_dim_conf_seq} formalizes three multivariate lifts (a weighted Bonferroni rectangle, an equivalent max-form, and a portfolio construction), proves their relations, and identifies the portfolio as the tightest but computationally intractable variant. \Cref{sec:AS:approximating_the_portfolio_process,sec:AS:computing_regions_in_practice} then develop the three tractable outer approximations to this portfolio and show how to compute them in practice so their regions remain statistically valid.

\section{Background: One-Dimensional Confidence Sequences by Betting} \label{sec:AS:one_dim_conf_seq}

\paragraph{Preliminaries.}
We observe random variables $y_1,y_2,\dots\in[0,1]$ with $\mathbb{E}[y_t | y_{1:t-1}]=\mu$ for all $t \ge 1$, where $\mu \in [0,1]$. This setup also includes the classical i.i.d. setting where $\mathbb{E}[y_t | y_{1:t-1}]=\mathbb{E}[y_t]=\mu$. 
A CS for the mean $\mu$ is a sequence of sets $\mathcal{C}_1, \mathcal{C}_2, \dots$ (denoted compactly as $(\mathcal{C}_t)_{t \ge 1}$) satisfying
\[
\mathbb{P}\!\left( \forall t \ge 1:\ \mu \in \mathcal{C}_t \right) \ge 1 - \alpha,
\]
for a prescribed error level $\alpha$~\citep{ryu_gamingbasedconfseq}. We can also denote it as a $(1-\alpha)$-CS, to emphasize $\alpha$. It can be constructed through a technique that relies on \citeauthor{Ville1939}'s ~\citeyearpar{Ville1939} inequality for supermartingales.
Let $\mathcal{F}_{t-1} := \sigma(Y_{1:t-1})$ denote the \emph{sigma-algebra generated by the past observations} $y_1,\dots,y_{t-1}$, representing all information available up to time $t-1$. A nonnegative process $(M_t)_{t=0}^\infty$ is a \emph{supermartingale} if
\begin{align}
    \mathbb{E}[M_t \mid \mathcal{F}_{t-1}] \le M_{t-1} \label{eq:AS:supermartingale}
\end{align}
for every $t \ge 1$. 
If $M_0=1$, it is a \emph{test} supermartingale (or test martingale when equality holds in \Cref{eq:AS:supermartingale}).
\citeauthor{Ville1939}'s inequality states that for any test (super)martingale,
\[
\mathbb{P}\bigl(\exists t: M_t \ge 1/\alpha \bigr) \;\le\; \alpha.
\]
Hence, if for candidate $m\in[0,1]$ we define $M_t(m)$ such that $(M_t(\mu))_{t\ge0}$ is a test (super)martingale and $M_t(m)$ is $\mathcal{F}_t$-measurable, then the sublevel set
\begin{align}
    \mathcal{C}_t \;=\; \{\,m\in[0,1]\mid M_t(m)\le1/\alpha\} \label{eq:AS:confseq_Mt}
\end{align}
gives a $(1-\alpha)$-CS~\citep{ryu_gamingbasedconfseq, waudby2024estimatingmean}. Ideally, $\mathcal{C}_t$ should shrink rapidly over time. Concretely, this means that for any candidate value $m \neq \mu$, the process $M_t(m)$ should cross the rejection threshold $1/\alpha$ rapidly, so that such values are excluded from $\mathcal{C}_t$ as early as possible.

\paragraph{Hedged Capital Process.} 
\citet{waudby2024estimatingmean} define $M_t(m)$ that is a test martingale for $m = \mu$ through the introduction of the \emph{Hedged Capital Process} \citep{waudby2024estimatingmean}, or \emph{wealth} process. It is based on two functions, respectively non-increasing and non-decreasing in $m \in [0,1]$:
\begin{align*}
k_t^+(m) \;=\;\prod_{i=1}^t \bigl(1 + \lambda_i^+(m)\,(Y_i - m)\bigr), \quad k_t^-(m) \;=\;\prod_{i=1}^t \bigl(1 - \lambda_i^-(m)\,(Y_i - m)\bigr),
\end{align*}
with $k_0^+(m) = k_0^-(m) = 1$, where $\lambda_i^{+}(m) \in[0,1 / m)$ and $\lambda_i^{-}(m) \in[0,1 /(1-m))$ are data-dependent ``stakes''. Concrete definitions are given in \Cref{apdx:AS:hedged_capital_stakes}.
The Hedged Capital Process is defined as the weighted maximum between the two functions,
\begin{align}
    k_t(m) \;=\; \max \{\theta k_t^+(m), \, (1-\theta)k_t^-(m) \},\label{Eq:AS:k_t_hedged_capital}
\end{align}
for a weighting parameter $\theta \in [0, 1]$. 

\paragraph{Remark.}
\citeauthor{waudby2024estimatingmean} initially introduce this wealth process as the convex-combination
\begin{equation}
    \dot k_t(m) \;=\; \theta k_t^+(m) + (1-\theta)k_t^-(m), \label{eq:AS:k_dot_convex_comb}
\end{equation}
and show that $\dot k_t(m)$ is a test martingale when $m=\mu$. Applying Ville’s inequality to $\dot k_t(m)$ therefore forms a valid CS via Eq.~\eqref{eq:AS:confseq_Mt}. Since $k_t(m) \le \dot k_t(m)$ for all $t$ and $m$, the sublevel set obtained by $k_t(m)$ defines a conservative outer CS and is thus also valid.
The preference for the less powerful max-form in \Cref{Eq:AS:k_t_hedged_capital} is motivated by its properties; 
 $k_t(m)$ is nonnegative and quasiconvex, so the sublevel set in Eq.~\eqref{eq:AS:confseq_Mt} is an interval, computable by root finding.

\section{From One-Dimensional Confidence Intervals to Higher-Dimensional Sets} \label{sec:AS:high_dim_conf_seq}

\begin{wrapfigure}{r}{0.32\columnwidth}
	\vspace{-1.5em}
	\centering
	\includegraphics[
	width=\linewidth,
	trim=2.05cm 2cm 2.5cm 1.5cm, 
	clip
	]{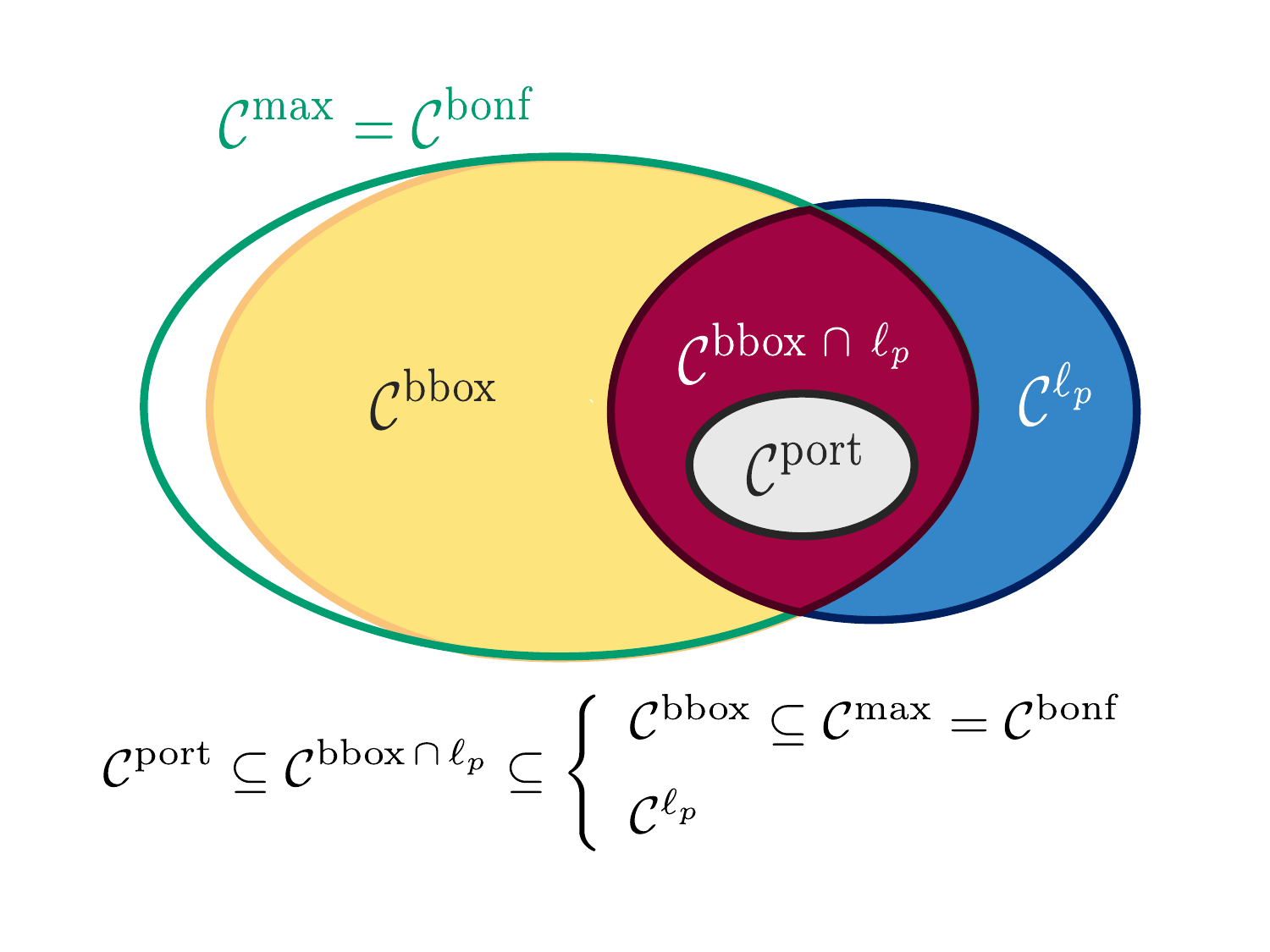}
	\caption{Venn diagram depicting the relations among the various sets.}
	\label{fig:AS:overview-venn}
	\vspace{-0em}
\end{wrapfigure}

Having described univariate CSs in \Cref{sec:AS:one_dim_conf_seq}, we now turn to the multivariate setting. Denote the vector-valued observations $\mathbf{y}_t \in [0,1]^D$ with conditional mean $\mathbb{E}[\mathbf{y}_t \mid \mathbf{y}_{1:t-1}] = \boldsymbol{\mu} \in [0,1]^D$. 
The goal is to construct wealth processes $K_1(\mathbf{m}), K_2(\mathbf{m}), \dots$ such that
\begin{align}
\mathcal{C}_t \;=\; \{\,\mathbf{m}\in[0,1]^D \mid K_t(\mathbf{m}) \le 1/\alpha \}
\end{align}
forms a valid $(1-\alpha)$-CS. The next subsections formalize three such multivariate lifts: \blue{(i)} a Bonferroni rectangle $\Bbonf_t$, \blue{(ii)} an equivalent max-form $\Bmax_t$, and a \blue{(iii)} portfolio construction $\Bavg_t$, and compare their relations. This comparison motivates $\Bavg_t$ as the tightest construction in this family, but also computationally intractable. Wealth processes for scalar $m$ are denoted by lowercase $k_t(\cdot)$, and those for vector $\mathbf{m}$ by uppercase $K_t(\cdot)$.

\subsection{Bonferroni Hyperrectangle}
Suppose we have $D$ wealth processes, one for each coordinate of $\mathbf{m} = (m_1,\ldots,m_D)$, denoted by
$k_t^1(m_1), \ldots, k_t^D(m_D)$. A straightforward multivariate lift is obtained via a weighted Bonferroni correction~\citep{holm_sequential_rejective_multiple_test_proec, kang2009weighted_bonferroni}.

Given a weight vector $\mathbf{w} \in \Delta^D$ on the probability simplex (i.e., $\sum_{d=1}^D w_d = 1$ and $w_d > 0$), we assign an error level $\alpha w_d$ to the $d$-th coordinate. This forms the Bonferroni confidence set
\[
\Bbonf_t
:=\Bigl\{ \mathbf{m} \in [0,1]^D \;\big|\; k_t^d(m_d) \le \tfrac{1}{\alpha w_d}, ~\forall d \in [D]\Bigr\}.
\]
When all weights are equal, i.e., $w_d = 1/D$, this reduces to the classical Bonferroni correction~\citep{Dunn01031961, napierala2012bonferroni}. In \Cref{sec:AS:max_of_wealth_processes} we indirectly show that $\Bbonf_t$ does indeed form a valid (1-$\alpha$)-CS through a proof on its equivalent max-form construction.

\subsection{Maximum of Wealth Processes} \label{sec:AS:max_of_wealth_processes}
The weighted Bonferroni correction enables the construction of a multivariate CS, but does not make the associated wealth process $K_t(\cdot)$ over vectors $\mathbf m$ explicit. Here, we introduce an equivalent formulation in terms of a single wealth process. Consider the max-wealth process
\[
\Kmax_t(\mathbf{m}) = \max_{d\in [D]} w_d k_t^d(m_d),
\]
with $\mathbf{w} \in \Delta^D$, and define the corresponding confidence set
\[
\Bmax_t
:= \left\{ \mathbf{m} \in [0,1]^D \mid K_t^{\max}(\mathbf{m}) \le \frac{1}{\alpha} \right\}.
\]

\begin{proposition}\label{prop:AS:kmax_is_cs}
The set $\Bmax_t$ is a valid $(1-\alpha)$-CS.
\end{proposition}
The proof follows by showing that $\Kmax_t(\boldsymbol{\mu})$ is upper bounded by a test martingale and applying Ville’s inequality; full details are provided in \Cref{apdx:AS:kmax_is_cs}.
The equivalence $\Bbonf_t = \Bmax_t$ is shown in \Cref{apdx:AS:thm_set_relations} in the appendix. This equivalence also (re)confirms the validity of $\Bbonf_t$ as a CS.

The main advantage of (weighted) Bonferroni rectangles (and thus also their equivalent $\Bmax_t$ form) is their simplicity. They are easy to compute, and the confidence set defines a $D$-dimensional hyperrectangle.
However, eliminating candidates $\vect{m}$ from this set is achieved by considering only a single coordinate's wealth function $k_t^d(m_d)$ at a time. The dependence between coordinates is not exploited, typically making them overly conservative as $D$ grows.

\subsection{Portfolio Process}
Inspired by the convex-combination baseline in~\citep{ryu_gamingbasedconfseq}, a stronger alternative that combines the wealths across coordinates is to consider a weighted sum of wealth processes (i.e., a \emph{portfolio process}):
\begin{align}
\Kavg_t(\mathbf{m}) = \sum_{d=1}^D w_d k_t^d(m_d), \label{sec:AS:mixture_of_wealth_Kavg}
\end{align}
with $\mathbf{w} \in \Delta^D$. It is straightforward to verify, in an analogous fashion to \Cref{prop:AS:kmax_is_cs}, that the set 
\begin{align}
    \Bavg_t
:=\bigl\{ \mathbf{m} \in [0, 1]^D \mid \Kavg_t(\mathbf{m}) \le \frac{1}{\alpha} \bigr\} \label{eq:AS:Bavg}
\end{align}
 is a $(1-\alpha$)-CS. The fact that $ \Bavg_t \subset \Bbonf_t$ is a known result~\citep{cho_peeking_with_PEAK}. For a proof, we refer to \Cref{apdx:AS:thm_set_relations} in the appendix.

While $\Bavg_t$ can be substantially smaller than $\Bbonf_t$, its practical use raises a new challenge: how can we represent this confidence set in a form that allows efficient computation of quantities such as volume or maximum width? Although each coordinate-wise wealth function $k_t^d(\cdot)$ is quasiconvex, the weighted sum $\Kavg_t(\cdot)$ does not preserve this property in general, and the resulting sublevel set is not necessarily convex anymore. 

For practical applications, we thus require tractable approximations. A naive approach would be to approximate $\Bavg_t$ by uniformly discretizing $[0,1]^D$ and evaluating $\Kavg_t(\mathbf{m})$ at each grid point. However, this quickly becomes computationally infeasible: a discretization with $N$ points per coordinate requires $N^D$ evaluations of $\Kavg_t(\cdot)$, which scales exponentially with the dimension. 

In the following sections, we therefore develop efficient \emph{outer} approximations to $\Bavg_t$ that have explicit representations, yet remain computationally efficient.

\section{Approximating the Portfolio Process} \label{sec:AS:approximating_the_portfolio_process}
We introduce three \emph{outer} approximations, i.e., supersets, of the region $\Bavg$ that allow for easy computation of volume and maximum width: the bounding box $\Bbbox$, the $\ell_p$-ellipsoid $\Bellip$, and their intersection $\Bbboxellip$. Importantly, since these approximations contain $\Bavg$, they still form valid CSs. We omit the $t$-indexing for simplicity.

\subsection{Bounding Box}
\label{sec:AS:bounding_box}
We define the axis-aligned bounding box
\begin{align}
\Bbbox
&=
\left\{\vect{x}\in[0,1]^D \mid x_d\in[l_d,u_d],
\ ~\forall d \in [D]\right\}, \label{AS:eq:bbox}
\\
[l_d,u_d]
&=
\left\{x_d\in[0,1]\mid
w_d k^d(x_d)
\le
\frac{1}{\alpha}-\sum_{j\neq d} w_j k^j_\star
\right\},
\end{align}
with $k_\star^j$ the global minimum of $k^j(\cdot)$ (derivation in \Cref{apdx:AS:bounding_box}).  It satisfies $ \Bavg \subseteq \Bbbox \subseteq \Bmax = \Bbonf$. Intuitively, $\Bbbox$ defines each coordinate's interval assuming the \emph{minimum possible contribution} of all other coordinates (their global minima), and thus, still contains the full portfolio set $\Bavg$.
At the same time, these coordinate-wise constraints remain tighter than the max/Bonferroni threshold, which gives $\Bbbox\subseteq\Bmax=\Bbonf$. Formal proofs are given in \Cref{apdx:AS:set_relations}.
A discussion on how to compute the bounding box in practice, together with a principled method to estimate the global minimum $k_\star^j$, follows in \Cref{sec:AS:computing_regions_in_practice}.

\subsection{$\ell_p$-Ellipsoid}\label{sec:AS:lp-ellipsoid}
The $\Bbbox$ in \Cref{AS:eq:bbox} is simple and efficiently computable, but can become a loose approximation when $\Bavg$ does not have an approximately rectangular shape. An alternative is to outer-approximate $\Bavg$ through an $\ell_p$-hyper-ellipsoid $\Bellip$. Concretely, we replace each function $k^d(x)$ with a simpler counterpart $\tilde k^d(x)$, resulting in the sublevel set
\begin{align}
	\Bellip := \{ \vect{x} \,|\, \sum_{d=1}^D w_d \tilde k^d(x_d) \le \frac{1}{\alpha} \}, 
	\quad
	\tilde{k}^d(x) := \frac{|x-a_d|^p}{b_d^p} - c_d. 
	\label{eq:AS:Bellip}
\end{align}
with parameters $a_d, b_d, c_d \in \mathbb{R}$, $b_d > 0$, and a shared degree $p \in \mathbb{N}_{>0}$. The practical benefit of this construction is that the sublevel set $\Bellip$ can be expressed as an axis-aligned $\ell_p$-hyper-ellipsoid:
$$
\sum_{d=1}^D \frac{|x_d-a_d|^p}{(b'_d)^p} \;\le\; 1, \quad
b'_d := \frac{b_d\, \sqrt[p]{1/\alpha + \sum_{j=1}^D w_j\,c_j} }{\sqrt[p]{w_d}},
$$
and thus its volume can be computed in closed form \citep{ahmed2018volumes, wang_volumes_of_generalized_unit_balls}: $V(\Bellip)=2^D \left( \prod_{d=1}^D  b_d' \right) \frac{ \Gamma\left(1+\frac{1}{p}\right)^D}{\Gamma\left(1+ \frac{D}{p}\right)}.$
Here, $\Gamma(\cdot)$ denotes the Gamma function. Note that for $p = 2$, we get the conventional $D$-dimensional ellipsoid (see \Cref{apdx:AS:ellipsoid} for details on the derivation and the generalization to its matrix form).

To ensure that $\Bavg \subseteq \Bellip$, it suffices that
$ \tilde k^d(x) \le k^d(x)
\quad \text{for all } d \in [D] \text{ and all } x$ (proven in \Cref{apdx:AS:thm_set_relations}).
The domain constraint can be made more specific:  it is not necessary to enforce this inequality over the entire domain $x \in [0,1]$. Instead, it suffices to satisfy the constraint on a restricted interval $[\xstart^d, \xend^d] \subset [0,1]$, where $(\xstart^d, \xend^d)$ denote the boundary points of the interval $\{ x \mid k^d(x) \le \tfrac{1}{\alpha w_d} \}$. As shown in the proof, outside this interval, the constraint is irrelevant for membership in $\Bavg$, which allows us to restrict attention to a smaller domain when fitting $\tilde k^d$.
How to obtain these parameters such that $\tilde k^d(x) \le k^d(x)$ is described in \Cref{sec:AS:in_practice_ellipsoid}.

\subsection{Intersection: Bounding Box $\cap$ $\ell_p$-Ellipsoid} \label{AS:sec:intersection_bbox_ellip}

Finally, we also consider the intersection $\Bbboxellip := \Bbbox \cap \Bellip.$ Since $\Bavg \subseteq \Bbbox$ and $\Bavg \subseteq \Bellip$, it immediately follows that $\Bavg \subseteq \Bbboxellip$. This completes the set relations shown in \Cref{fig:AS:overview-venn}, with the details and formal proofs described in \Cref{apdx:AS:set_relations}.

Among the proposed approximations, $\Bbboxellip$ is also our smallest superset. However, contrary to $\Bbbox$ and $\Bellip$, to our knowledge, there is no longer an analytical solution for the volume. As a result, volume estimation must rely on hit-or-miss Monte Carlo methods. A practical advantage is that the bounding rectangle from $\Bbbox$ can serve as the sampling region, instead of the naive $[0,1]^D$ hypercube.  
Since $\Bbbox$ is typically much closer to $\Bbboxellip$, a larger fraction of sampled points fall inside the target region, reducing the number of samples needed for a given precision.

Moreover, when membership evaluation in $\Bavg$ is computationally expensive (e.g., due to costly evaluations of $k^d(\cdot)$),  one can first construct $\Bbboxellip$ using only a limited number of such evaluations,  and then perform all subsequent membership tests directly on $\Bbboxellip$, 
which is computationally inexpensive.

Certain other interesting properties of $\Bbboxellip$ are also easy to bound. For instance, for its maximum width, we have: $\mathrm{width}(\Bbboxellip) \le \min\{\mathrm{width}(\Bbbox),\,\mathrm{width}(\Bellip)\}$.

\section{Computing the Regions in Practice} \label{sec:AS:computing_regions_in_practice}
Having defined the outer approximations $\Bbbox$ and $\Bellip$, we now turn to their implementation, with emphasis on preserving statistical validity: the numerical procedures should only enlarge the proposed regions, so the resulting computed sets remain valid CSs.

\subsection{A Lower Bound on the Global Minimum of $k(\cdot)$}
\label{sec:global_minimium}
Several steps in constructing $\Bbbox$ and $\Bellip$ require a \emph{conservative} estimate of the global minimum 
$$
k_\star := \min_{x\in[0,1]} k(x) 
\quad \text{ with } \quad
k(x)=\max\{\theta k^+(x),(1-\theta)k^-(x)\}
$$  
the one-dimensional function defined in \Cref{Eq:AS:k_t_hedged_capital}. Rather than estimating $k_\star$ via dense grid search, we exploit the monotonic structure of $k^+$ and $k^-$ to obtain a lower-bounded estimate through a small number of root-finding calls.
Concretely, we show that (i) minimizers always exist, (ii) they are characterized by the roots of $h(x)=(1-\theta)k^-(x)-\theta k^+(x)$, and (iii) based on this result derive \Cref{alg:AS:bound-global-min-explicit} which returns explicit upper and lower bounds on $k_\star$ (together with minimizer estimate $\hat{x}$). Thus, the resulting procedure is theory-driven and ensures that $\underline{k}^\star \le k_\star$, with the gap controlled by a preset tolerance. See \Cref{apdx:AS:global_minimum_bound_k} for details.

\subsection{Computing the Bounding Box}\label{sec:AS:in_practice_bbox}

Following \Cref{sec:AS:bounding_box}, $\Bbbox$ is given by $\{\mathbf{x} \mid w_d k^d(x_d) \le \tfrac{1}{\alpha} - \sum_{j \neq d} w_j k_\star^{j}, \quad ~\forall d \in [D] \}.$ Since the true minima $k_\star^j$ are unknown, we use lower bounds $\underline{k_\star^j}$ from \textsc{Estimate-$k_\star$} (\Cref{alg:AS:bound-global-min-explicit}), which also returns minimizer estimates $\hat{x}_j$. As $\underline{k_\star^j} \le k_\star^j$, this enlarges the box. Given $\hat{x}_d$, a conservative interval $[l_d,u_d]$ is obtained by solving $w_d k^d(x_d) - \! (\tfrac{1}{\alpha}-\sum_{j\neq d} w_j k^j_\star) - \epsilon = 0$ on $[0,\hat{x}_d]$ and $[\hat{x}_d,1]$, with constant $\epsilon>0$ ensuring enclosure in case of flat regions. Applying this for all $d \in [D]$ yields $\Bbbox$, which remains a valid outer approximation of $\Bavg$ after implementation.

\paragraph{Computational Complexity.}
Using bisection for root finding, the global minimum of each $k^d(\cdot)$ is first approximated using one bisection call per dimension. Next, the endpoints $l_d$ and $u_d$ are computed via two additional bisection calls per dimension. In total, this results in $3D$ bisections. Let $C_{\mathrm{eval}}$ denote the cost of evaluating $k^d(\cdot)$. Since bisection converges in $O\!\left(\log(1/\varepsilon)\right)$ iterations, the total time complexity is $O\!\left(D\, C_{\mathrm{eval}}\, \log(1/\varepsilon)\right),$
which is equivalent to the complexity of the Bonferroni construction $\Bbonf$.

\subsection{Computing the $\ell_p$-Ellipsoid}
\label{sec:AS:in_practice_ellipsoid}
Obtaining an efficient implementation of the $\ell_p$-ellipsoid while preserving validity requires substantially more technical care. Recall that $\Bellip := \{ \vect{x}\; | \;\sum_{d=1}^D w_d\,\tilde{k}^d(x_d)\le \tfrac{1}{\alpha} \}$, where each $\tilde{k}^d(x) := \frac{|x-a_d|^p}{b_d^p}-c_d$ is constructed to satisfy $\tilde{k}^d(x)\le k^d(x)$ for all $x\in[\xstart^d,\xend^d]$. We describe two variants: (i) an approximate construction that enforces this condition only on a grid, and (ii) a conservative modification $\BellipS$ that guarantees $\Bavg\subseteq\BellipS$.

\paragraph{Overview.}
For a fixed degree $p$, we fit parameters $\{(a_d,b_d,c_d)\}_{d=1}^D$ independently per dimension. We then select $p$ by a small grid search over candidates $\mathcal{P}$ (e.g., $\{6,7,\dots,16\}$), choosing the minimum-volume ellipsoid: 
$$p^\star = \arg\min_{p\in\mathcal{P}} V\left( \Bellip(p) \right).$$

\begin{wrapfigure}{r}{0.356\columnwidth}
  \vspace{-0.2\baselineskip}
  \centering
  \begin{tikzpicture}
    \begin{axis}[
      unbounded coords=discard,
      width=1.06\linewidth,
      height=4.05cm,
      grid=none,
      xmin=0.2000, xmax=1.0000,
      ymin=-5.0000, ymax=50.0000,
      label style={font=\scriptsize},
      tick label style={font=\scriptsize},
	xtick={0, 0.25,0.5,0.75,1.0},
      ytick={0,10,20,30,40,50},
      extra x ticks={0.448227,0.812329},
      extra x tick labels={$\xtLd$,$\xtRd$},
      extra x tick style={tick label style={font=\scriptsize,yshift=-1.1ex}},
      extra y ticks={16.817930},
      extra y tick labels={$\gamma_d$},
      extra y tick style={tick label style={font=\scriptsize,xshift=-0.8ex}},
      legend style={
        at={(0.5,-0.24)},
        anchor=north,
        draw=none,
        font=\scriptsize,
        legend columns=2
      },
    ]
      \addplot[very thick, black] table[col sep=comma] {assets/example_tilde_k_fit/figure-d0-hedged.csv};
      \addlegendentry{$k^d(x)$}
      \addplot[very thick, dashed, color=magenta!80!black] table[col sep=comma] {assets/example_tilde_k_fit/figure-d0-parabola.csv};
      \addlegendentry{$\tilde k^d(x)$}

      \addplot[black!70, dashed, forget plot] coordinates {(0.448227,-5.000000) (0.448227,50.000000)};
      \addplot[black!70, dashed, forget plot] coordinates {(0.812329,-5.000000) (0.812329,50.000000)};
      \addplot[black!70, dashed, forget plot] coordinates {(0.200000,16.817930) (1.000000,16.817930)};
    \end{axis}
  \end{tikzpicture}
  \vspace{-0.8\baselineskip}
  \caption{
  Example where $\tilde k^d(x)$ lower-bounds $k^d(x)$, touching $\gamma_d$ at two points; tight near $\gamma_d$ and looser farther away.
  }
  \label{fig:AS:fit_and_extension_examples}
  \vspace{-2.5\baselineskip}
\end{wrapfigure}
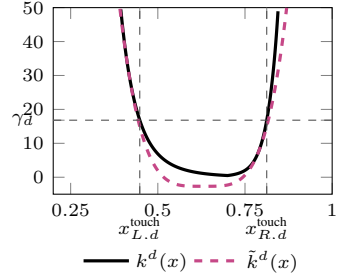

\paragraph{Per-Dimension Fit via Two Touch Points.}
Rather than approximating $k^d(\cdot)$ globally, we enforce a good fit only in the most relevant regions. We select a target level $\gamma_d>k^d_\star$ and fit $\tilde{k}^d(\cdot)$ so that it (approximately) touches $k^d(\cdot)$ at two points. Outside this region, some slack is allowed (\Cref{fig:AS:fit_and_extension_examples}).
Formally, given two touch points $\xtLd$ and $\xtRd$, we enforce
\begin{align}
	\tilde{k}^d(\xtLd)=\tilde{k}^d(\xtRd)=\gamma_d,
	\label{eq:AS:touch_gamma}
	\\
	\tilde{k}^d(x_i)\le k^d(x_i)\ \ \text{for all } x_i \in \{ x_i\}_{i=1}^N,
	\label{eq:AS:dominate}
\end{align}
where $\{ x_i\}_{i=1}^N$ is a fine grid on the restricted domain $[\xstart^d,\xend^d]$. Importantly, the fit is largely analytic: $a_d$ is fixed directly by the two touch-point constraints, and $c_d$ is obtained in closed form once $b_d$ is chosen. Numerically, the procedure reduces to a single pass over $\{x_i\}_{i=1}^N$ to obtain feasibility bounds for $b_d^p$. Full derivations, the choice of $\gamma_d$, and fallback procedures (including adaptive touch-point widening and domain extension) are given in \Cref{apdx:AS:find_abc,apdx:AS:adaptive_touch,apdx:AS:k_extension}.

\paragraph{Discretization and Safe Option.} \label{sec:AS:disc_and_conservative_alternative}
Constraint~\eqref{eq:AS:dominate} enforces $\tilde{k}^d(x_i)\le k^d(x_i)$ only on a grid $\{ x_i\}_{i=1}^N$. In \Cref{apdx:AS:a_bound_on_the_gridding_err}, we derive an upper bound on the worst-case violation between the grid points: 
$$
\max_{x\in[\xstart^d,\xend^d]}\!\bigl(\tilde{k}^d(x)-k^d(x)\bigr)
\;\le\;
\max_{i}\!\left(\max\{\tilde{k}^d(x_i),\tilde{k}^d(x_{i+1})\}-L_i^d\right),
$$

where $L_i^d$ is a computable lower bound on $k^d(\cdot)$ over the interval $[x_i,x_{i+1}]$ (defined in \Cref{thm:phi_bound}). Based on this bound, we introduce the ``safe'' construction $\BellipS$ that modifies the grid points so that $\tilde{k}^d(x)\le k^d(x)$ holds \textbf{for all} $x\in[\xstart^d,\xend^d]$ (\Cref{prop:safe_shift}). This safe mode can then be used when a true superset guarantee is required. \Cref{tab:cs_validity} summarizes all outer approximations and whether they remain valid after implementation.

\paragraph{Computational Complexity.}
Let $N$ denote the grid size per dimension and let $N_p = |\mathcal{P}|$ denote the number of $p$-candidates. The ellipsoid construction has complexity
$O\!\left(D\,N_p\,N\,\left(\log(L/\text{tol}) + C_{\text{eval}}\right)\right),$
where $C_{\mathrm{eval}}$ is the cost of evaluating $k^d(\cdot)$ and the logarithmic term comes from the bisection-based fallback procedures. Algorithms and a detailed breakdown are provided in \Cref{AS:apdx:ellipse_algorithms_and_complexity}.

\begin{wraptable}{r}{0.3\columnwidth}
\vspace{-1.0\baselineskip}
\centering
\caption{Overview of our outer approximations and whether they form a valid CS after implementation. 
$\sim$~indicates an approximation without a formal CS guarantee.}
\label{tab:cs_validity}
\small
\setlength{\tabcolsep}{4pt}
\begin{tabular*}{\linewidth}{@{\extracolsep{\fill}} l c @{}}
\toprule
\textbf{Approximation} & \textbf{Valid CS} \\
\midrule
$\Bbbox$             & \cmark \\
$\Bellip$            & \approxmark \\
$\BellipS$& \cmark \\
$\Bbboxellip$& \approxmark \\
$\BbboxellipS$& \cmark \\
\bottomrule
\end{tabular*}
\vspace{-3.5\baselineskip}
\end{wraptable}

\section{Empirical Results} \label{sec:AS:empirical_results}
We evaluate the tightness and efficiency of the proposed CSs, compare them against recent alternatives, and demonstrate their practical use in two case studies. We show that the portfolio is the tightest but intractable, while our approximations recover much of this gap while remaining efficient.

\paragraph{Baselines.}
We compare with four recent CSs: \emph{Banach Sphere}~\citep{martinez2024empirical_bernstein_banach_spaces}, \emph{Conf Sphere}~\citep{chugg2023confsphere}, and \emph{Sub-$\psi$ ellipsoids}~\citep{whitehouse2023time_uniform_self_normalized_concentration_vector_valued_processes}, all of which have analytical solutions and are inexpensive to compute. We refer to smaller-scale results for the bounded-vector reduction \emph{Bounded UP}~\citep{ryu_gamingbasedconfseq} in \Cref{apdx:AS:extra_experimetns}, since its full comparisons are infeasible due to lack of a closed-form solution and membership checking costs $O(t^{D+1})$.

\paragraph{Synthetic Experiments Setup.}
We use two synthetic multivariate distributions: \textbf{(i)} i.i.d. samples from the unit ball (favors sphere baselines), and \textbf{(ii)} a Bernoulli-mixture construction in $[0,1]^D$ (favors our constructions and Bounded UP). Full details are in \Cref{apdx:AS:extra_experimetns}.

\paragraph{Tightness and Coverage.}
\Cref{fig:AS:combined_experiments} compares the various CSs, and demonstrates how across all settings, our proposed regions achieve substantially smaller volumes than the sphere-based baselines (columns 1, 3, and 4). This also results in less conservative CSs with a per-time-step coverage closer to the $1-\alpha$ target (column 5).

\paragraph{Approximation Quality.}
We observe from \Cref{fig:AS:combined_experiments}, column 2, that the gap between $\Bbonf$ and $\Bavg$ rapidly grows with the dimension. For example, at $D=10$, the volume of $\Bavg$ is up to 50\% smaller than that of $\Bbonf$. Both the $\ell_p$-ellipsoid $\BellipS$ and the intersection $\BbboxellipS$ approximate $\Bavg$ well, showing a clear improvement over the Bonferroni construction, yet at substantially lower computational cost than $\Bavg$.  Finally, the bounding-box construction $\Bbbox$ produces slightly smaller regions than $\Bbonf$ at the exact same computational complexity.

 \begin{figure}[!t]
 	\centering
 	\begin{tikzpicture}
 		\node[anchor=south west, inner sep=0] (img) at (0,0)
		{\includegraphics[width=\linewidth]{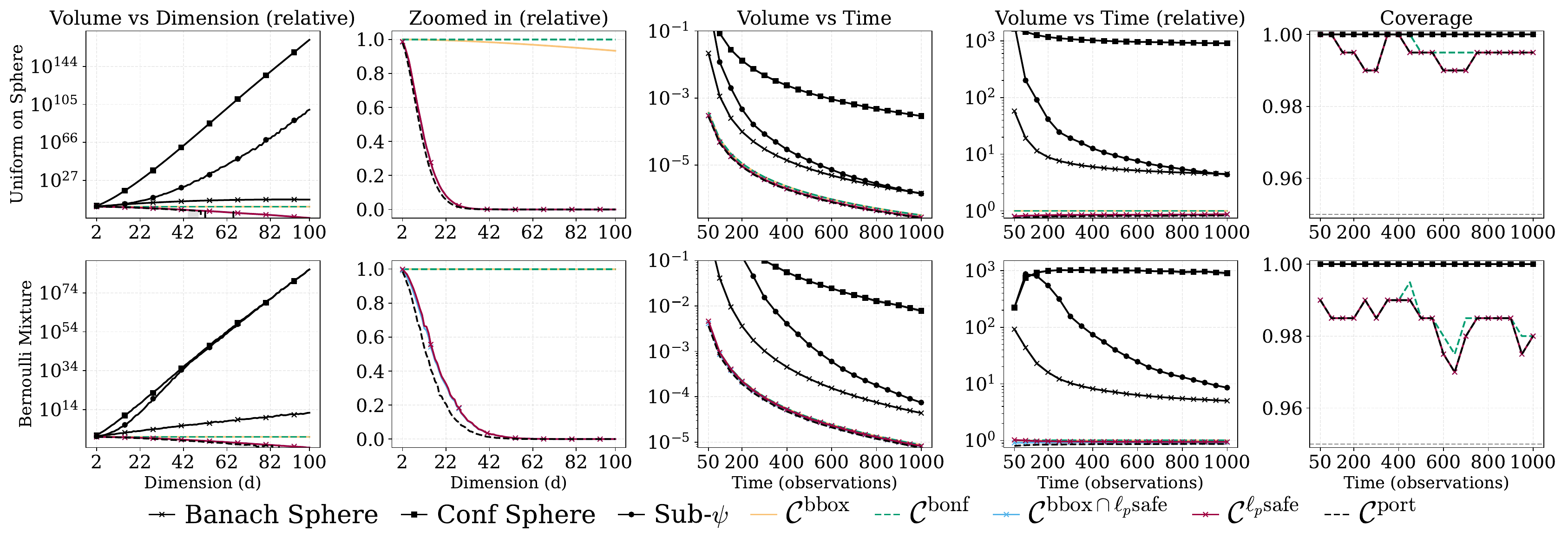}};
		
 		\begin{scope}[x={(img.south east)}, y={(img.north west)}]
			
 			\tikzset{
                zoombox/.style={
                    draw=myblue,
                    line width=0.5pt,
                    rounded corners=1.2pt
                },
                zoomarrow/.style={
                    -{Triangle[length=1.5mm,width=2mm]},
                    draw=myblue,
                    line width=0.5pt
                }
            }
			
 			\coordinate (A1) at (0.052, 0.600); 
 			\coordinate (A2) at (0.208,  0.651); 
 			\draw[zoombox] (A1) rectangle (A2);
			
 			\coordinate (B1) at (0.226, 0.605); 
 			\coordinate (B2) at (0.430, 0.9);
			
 			\coordinate (Aeast) at ($(A2|-A1)!0.5!(A2)$);
 			\coordinate (Bwest) at ($(B1)!0.56!(B1|-B2)$); 
			
 			\draw[zoomarrow]
 			($(Aeast) + (0,0)$) 
 			to[out=8,in=180]
 			($(Bwest) + (0.006,0)$); 
			
 			\coordinate (C1) at (0.052, 0.189); 
 			\coordinate (C2) at (0.208,  0.239); 
 			\draw[zoombox] (C1) rectangle (C2);
			
 			\coordinate (D1) at (0.226, 0.194);
 			\coordinate (D2) at (0.430, 0.489);
			
 			\coordinate (Ceast) at ($(C2|-C1)!0.5!(C2)$);
 			\coordinate (Dwest) at ($(D1)!0.56!(D1|-D2)$);
			
 			\draw[zoomarrow]
 			($(Ceast) + (0,0)$)
 			to[out=8,in=180]
 			($(Dwest) + (0.006,0)$);
			
 		\end{scope}
 	\end{tikzpicture}
     \caption{\textbf{Comparison on synthetic data}. The proposed constructions consistently achieve smaller volumes while maintaining the target coverage.
     ``Relative'' plots report $V(\text{method})/V(\Bbonf)$.
     }
 	\label{fig:AS:combined_experiments}
 \end{figure}
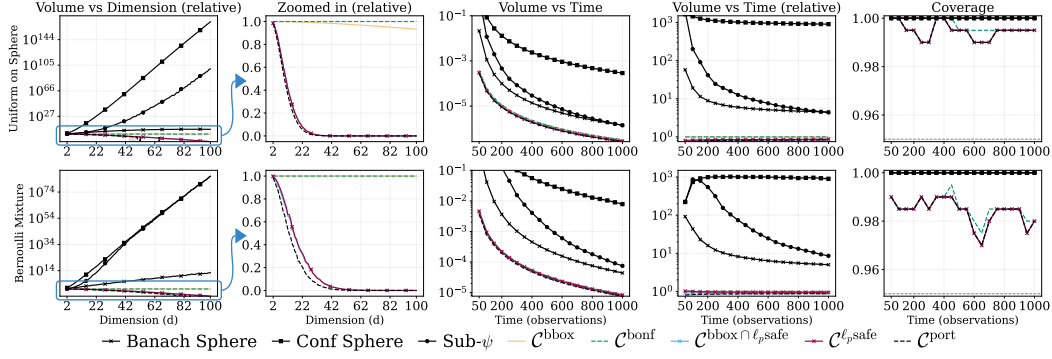

\paragraph{Computational Cost}
\Cref{tab:AS:hedged_volume_comparison_safe} reports the number of $k_t^d(\cdot)$ evaluations required to obtain volumes at $t=100$. Our approximations are orders of magnitude cheaper than direct $\Bavg$ gridding ($N^D$), with $N=2001$. For $\BbboxellipS$, we count only the construction-time $k_t^d(\cdot)$ calls; although its volume has no closed-form expression, once constructed, it can be estimated using inexpensive membership tests on the bounding box and the $\ell_p$-ellipsoid, rather than direct evaluations of $k_t^d(\cdot)$. The external baselines are not included since they do not define $k^d_t(\cdot)$. We instead report wall-clock evaluation times in \Cref{tab:AS:evaluation_time}.

 \begin{table}[!t]
	\centering
	\caption{Number of $k^d_t(\cdot)$ calls to compute the volume at $t=100$ (sorted from most to least efficient). 
    The outer approximations are orders of magnitude more efficient than $\Bavg$'s prohibitive $N^D$ cost.}
	\label{tab:AS:hedged_volume_comparison_safe}
	\scriptsize
	\setlength{\tabcolsep}{2.5pt}
	\renewcommand{\arraystretch}{0.95}
	{%
		\resizebox{\columnwidth}{!}{%
		\begin{tabular}{c c c c c c}
			\toprule
			$D$ 
			& $\Bbbox$ 
			& $\Bbonf$ 
			& $\BellipS$ 
			& $\BbboxellipS$ 
            & $\Bavg$  \\
            \midrule
            2  & $5.01\!\times\!10^{2}$ & $5.05\!\times\!10^{2}$ & $1.60\!\times\!10^{5}$ & $1.61\!\times\!10^{5}$ & $8.00\!\times\!10^{6}$ \\
            3  & $7.78\!\times\!10^{2}$ & $7.78\!\times\!10^{2}$ & $2.40\!\times\!10^{5}$ & $2.41\!\times\!10^{5}$ & $2.40\!\times\!10^{10}$ \\
            4  & $1.02\!\times\!10^{3}$ & $1.02\!\times\!10^{3}$ & $3.20\!\times\!10^{5}$ & $3.21\!\times\!10^{5}$ & $6.40\!\times\!10^{13}$ \\
            5  & $1.30\!\times\!10^{3}$ & $1.30\!\times\!10^{3}$ & $4.00\!\times\!10^{5}$ & $4.01\!\times\!10^{5}$ & $1.61\!\times\!10^{17}$ \\
            6  & $1.54\!\times\!10^{3}$ & $1.55\!\times\!10^{3}$ & $4.80\!\times\!10^{5}$ & $4.82\!\times\!10^{5}$ & $3.85\!\times\!10^{20}$ \\
            7  & $1.85\!\times\!10^{3}$ & $1.85\!\times\!10^{3}$ & $5.61\!\times\!10^{5}$ & $5.63\!\times\!10^{5}$ & $8.96\!\times\!10^{23}$ \\
            8  & $2.11\!\times\!10^{3}$ & $2.12\!\times\!10^{3}$ & $6.41\!\times\!10^{5}$ & $6.43\!\times\!10^{5}$ & $2.06\!\times\!10^{27}$ \\
            9  & $2.40\!\times\!10^{3}$ & $2.40\!\times\!10^{3}$ & $7.22\!\times\!10^{5}$ & $7.24\!\times\!10^{5}$ & $4.63\!\times\!10^{30}$ \\
            10 & $2.68\!\times\!10^{3}$ & $2.68\!\times\!10^{3}$ & $8.03\!\times\!10^{5}$ & $8.06\!\times\!10^{5}$ & $1.03\!\times\!10^{34}$ \\
			\bottomrule
			Complexity 
			& $O\!\left(D\,C_{\mathrm{eval}}\,\log(1/\varepsilon)\right)$
			& $O\!\left(D\,C_{\mathrm{eval}}\,\log(1/\varepsilon)\right)$
			& $O\!\bigl(D\,N_p\,N\,(\log(L/\text{tol}) + C_{\text{eval}})\bigr)$
			& $\displaystyle
			O\!\Bigl(
			\begin{aligned}
				& D\,C_{\mathrm{eval}}\,\log(1/\varepsilon) \, + \\
				& {} D\,N_p\,N\bigl(\log(L/\text{tol}) + C_{\text{eval}}\bigr)
			\end{aligned}
			\Bigr)$
			& $O\!\left(DC_{\mathrm{eval}}N^D \right)$\\
		\end{tabular}
		}
	}
	\vspace{-2em}
\end{table}

\paragraph{Paired A/B-Testing Experiment of Simulated Web Metrics.}
We compare two versions of a website across ten simulated metrics, with the goal of determining whether one version shows a statistically significant difference for a particular metric. 

This is achieved by simulating batches of web-shop sessions for variants $A$ and $B$, with each session described by ten metrics, including time-on-site, clicks, scroll depth, purchase behavior, and two noise variables. The simulation captures several aspects of real-world settings: the number of observations assigned to each variant may differ across batches, and correlations between metrics are induced through a shared latent variable. To simulate a treatment effect, variant $B$ is optionally improved on a subset of metrics by slightly increasing this latent variable. \Cref{fig:AS:ab_batch_distributions} visualizes the resulting per-session metric distributions for a representative batch.

The metrics are fed to the CS as follows: let $\mathbf{y}_t^{(A)}$ and $\mathbf{y}_t^{(B)}$ denote the metric vectors for variants $A$ and $B$ in batch $t$, respectively, and define $\boldsymbol{\Delta}_t := \mathbf{y}_t^{(B)}-\mathbf{y}_t^{(A)}$. A CS then gives a region $\mathcal{C}_t$ for $\mathbb{E}[\boldsymbol{\Delta}_t]$ at every time $t$. If the interval for a particular coordinate lies entirely above or below $0$, we can conclude that the variants differ significantly on that metric. On the other hand, if the region becomes sufficiently small without excluding $0$ (measured through volume), we may conclude that there is no practically meaningful difference and stop collecting data. See \Cref{apdx:AS:extra_experimetns} for details.

\Cref{tab:AS:ab_results} shows that our methods achieve \textbf{much earlier stopping times} than existing tractable CSs when an effect is present, and in the no-effect setting, the $\ell_p$-ellipsoid-based variants stop the fastest.

\begin{wraptable}{r}{0.43\textwidth}
		\vspace{-1.2\baselineskip}
	\caption{\textbf{A/B stopping outcomes.}
		We stop for \emph{significance} once a coordinate interval of $\boldsymbol{\Delta}_t$ excludes $0$. Otherwise, we stop and report \emph{unsubstantial} once volume has shrunken to a preset threshold. If at the end of the $10{,}000$ batches no decision was made, we report \emph{no decision}.}
	\label{tab:AS:ab_results}
	\scriptsize
	\setlength{\tabcolsep}{2pt}
	\renewcommand{\arraystretch}{0.95}
	\begin{tabularx}{\linewidth}{@{}>{\raggedright\arraybackslash}p{0.22\linewidth} c >{\raggedright\arraybackslash}X c >{\centering\arraybackslash}p{0.24\linewidth}@{}}
		\toprule
		& \multicolumn{2}{c}{With effect} & \multicolumn{2}{c}{No effect} \\
		\cmidrule(lr){2-3}\cmidrule(lr){4-5}
		Method & Stop$_t$& Decision& Stop$_t$ & Decision \\
		\midrule
		ConfSphere        & 10,000& No decision    & 10,000& No decision \\
		BanachSphere      & 5,018& Significant    & 8,052& Unsubstantial \\
		Sub-$\psi$& 5,198& Significant    & 6,650& Unsubstantial \\
		$\Bbonf$                   & 178& Significant    & 399& Unsubstantial \\
		$\Bbbox$                   & 178& Significant    & 398& Unsubstantial \\
		$\BellipS$& 178& Significant    & 356& Unsubstantial \\
		$\BbboxellipS$& \textbf{178}& Significant    & \textbf{356}& Unsubstantial \\
		\bottomrule
	\end{tabularx}
	\vspace{-2\baselineskip}
\end{wraptable}

\paragraph{Multi-metric Model Comparison Experiment.} 
To illustrate the value on a real-world dataset, we compare a challenger and baseline classifier on the Adult~\citep{adult_dataset} dataset across six metrics: overall accuracy, Brier score, and four subgroup accuracies. The goal is to certify an improvement in overall accuracy while simultaneously monitoring subgroup degradations with statistical guarantees.

We let model $A$ denote the baseline and model $B$ the challenger. Both classifiers use features such as gender, education, and age to predict whether a person's income exceeds \$50{,}000. The setup is similar to the A/B experiment: each $\boldsymbol{\Delta}_t$ is defined as the difference between two metric vectors, with $\mathbf{y}_t^{(A)}$ representing the metrics of model $A$ (fully defined in \Cref{AS:eq:model_comparision_defin_y_A}) and $\mathbf{y}_t^{(B)}$ those of model $B$. Within these vectors, overall accuracy is represented by the indicator $\mathbf{1}\{f_A(\mathbf{x}_t)=f_t\}$ where $f_t$ is the true label, while subgroup performance is represented through joint indicators, e.g., $\mathbf{1}\{f_A(\mathbf{x}_t)=f_t \text{ and } \mathbf{x}_t \text{ is woman}\}$. The main technical challenge is handling conditional accuracies, such as $\Pr(f_B(\mathbf{x}_t)=\text{correct} \mid \mathbf{x}_t \text{ is woman})$, while preserving anytime validity. We achieve this by expressing each conditional accuracy as a ratio of two means: the corresponding joint accuracy and an additional coordinate representing the subgroup proportion, e.g., $\Pr(\mathbf{x}_t \text{ is woman})$. Technical details are provided in \Cref{apdx:AS:model_comparison}.

Our CSs detect an improvement in overall accuracy after about 356 batches, whereas existing tractable multivariate CS baselines remain inconclusive even after observing the entire test dataset of 509 batches. 
The intervals in \Cref{fig:AS:model_comparison_intervals} also flag possible subgroup degradation.
While the new model shows a statistically significant improvement in overall accuracy, its accuracy on women shows a plausible performance decrease of up to 4\%. This suggests a concrete next step, such as collecting more representative data or improving performance on that subgroup before deployment.

\begin{figure}
	\centering
	\includegraphics[width=\linewidth]{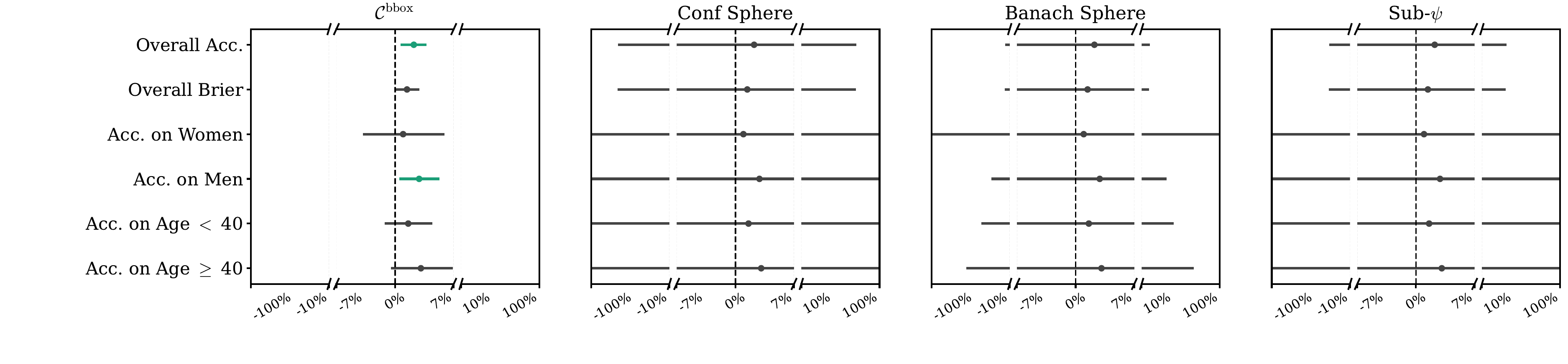}
	\caption{
    \textbf{Model-gain intervals on six Adult metrics} after 509 batches. Only $\mathcal{C}^{\mathrm{bbox}}$ certifies overall-accuracy improvement (0.6\%-3.7\%), reaching significance at 356 batches; baselines remain inconclusive. The other hedged-capital variants behave similarly to $\mathcal{C}^{\mathrm{bbox}}$ (\Cref{fig:AS:model_comparison_intervals_all}). 
	}
	\label{fig:AS:model_comparison_intervals}
\end{figure}

\section{Conclusion} \label{sec:AS:Conclusion}
We study three approaches for lifting one-dimensional confidence sequences (CSs), defined as sublevel sets of a wealth process, to higher-dimensional means. These include a weighted Bonferroni construction $\Bbonf$, an equivalent maximum-wealth formulation $\Bmax$, and a portfolio-based construction $\Bavg$, which leads to the tightest CS but is computationally intractable. For $\Bavg$, we propose three \emph{outer} approximations: 
a bounding box $\Bbbox$, an $\ell_p$-ellipsoid $\Bellip$, and their intersection.

We theoretically prove set relations among all constructions and obtain several interesting insights: we show that $\Bmax$ is equivalent to $\Bbonf$, and that the region of $\Bbbox$ is always contained in the corresponding Bonferroni region. At exactly the same computational cost, one can thus replace $\Bbonf$ by $\Bbbox$ and obtain an at least as tight valid CS. We also find that $\Bellip$ tracks $\Bavg$ closely in practice while allowing for a closed-form expression for its volume. Importantly, our proposed approximations produce regions that maintain the statistical validity of the CS, even after implementation.

Overall, the proposed CSs and their outer approximations lead to substantially smaller regions than existing tractable alternatives while remaining practical to compute in higher dimensions. The practical value is further highlighted in both the A/B testing and model-comparison case studies, which lead to detecting significance after substantially fewer observations, allowing practitioners to make reliable decisions earlier and reduce data-collection costs.

\paragraph{Limitations and Future Work}
While our outer approximations define valid CSs, the $\ell_p$-ellipsoid relies on per-dimension discretization, which can be costly for dense grids. Also, the bound used for its safe construction is very general, requiring only that the underlying function is quasiconvex, continuous, and globally lower bounded by a scalar $\phi$. However, this generality can make the bound conservative, requiring a dense grid to compensate. Moreover, our smallest outer approximation, given by the intersection $\Bbbox \cap \Bellip$, no longer has a closed-form volume expression and must therefore be estimated using Monte Carlo sampling. This estimate becomes less accurate as the dimension increases, requiring more samples to maintain the same level of accuracy. Several directions for future work remain. First, one could study whether other wealth processes (e.g.,~\citet{ryu_gamingbasedconfseq}) can be integrated into the proposed framework. Second, the computational efficiency of our procedures could be further improved through more advanced root-finding or adaptive discretization strategies. Third, one could investigate whether dependencies between coordinates can be exploited to construct tighter CSs. Finally, since CSs remain valid under intersections over time, $\mathcal{C}_t := \cap_{i\le t} \,\mathcal{C}_i$, it would be interesting to incorporate this structure into the approximations as well.

\FloatBarrier

\section*{Acknowledgments and Disclosure of Funding}
We thank Dario Coscia and Rajeev Verma for the discussions that led to this work. We also thank José Oramas, Roel Hulsman, Alexander Timans, and Mátyás
Schubert for their valuable input and discussions during the early stages of the project. 
This project was supported by the University of Antwerp, the University of Amsterdam, Eindhoven University of Technology, the Research Foundation - Flanders (FWO) under grant agreement V429625N, and the MOSAIC project under Grant Agreement No. 101194414. 

Funded by the European Union. Views and opinions expressed are however those of the author(s) only and do not necessarily reflect those of the European Union or the European Commission. Neither the European Union nor the European Commission can be held responsible for them.

\bibliographystyle{unsrtnat}
\bibliography{_biblio}


\newpage
\appendix
\section*{Appendix}
This appendix is divided into three parts. Part~I presents the core validity and set relation proofs as well as several auxiliary theoretical results. Part~II focuses on the proposed outer approximations: it first derives the bounding-box construction, and then develops the $\ell_p$-ellipsoid construction, together with its certified between-grid control and safe variant. Part~III provides experimental setup details, together with ablations and additional case studies.

\makeatletter
\let\origaddcontentsline\addcontentsline
\def\addcontentsline#1#2#3{%
  \origaddcontentsline{#1}{#2}{#3}%
  \def\@tempa{#1}%
  \def\@tempb{toc}%
  \ifx\@tempa\@tempb
    \addtocontents{atoc}{\protect\contentsline{#2}{#3}{\thepage}{}\protected@file@percent}%
  \fi
}
\makeatother
{\setcounter{tocdepth}{2}
\renewcommand{\contentsname}{Contents}
\makeatletter
\renewcommand{\l@section}{\@dottedtocline{1}{1.5em}{2.3em}}
\@starttoc{atoc}
\makeatother}

\section*{Part I: Theory and Proofs}
\addcontentsline{toc}{part}{Part I: Theory and Proofs}

\section{Proof of \Cref{prop:AS:kmax_is_cs}}
\label{apdx:AS:kmax_is_cs}

\begin{proposition}\label{apdx:prop:AS:kmax_is_cs}
The set $\Bmax_t$ is a valid $(1-\alpha)$-CS.
\end{proposition}
\begin{proof}
First, note that for any $\vect{m} \in [0,1]^D$,
\begin{align}
    \Kmax_t(\mathbf m)
& = \max_d \, w_d k_t^d(m_d) \;\le\; \sum_{d=1}^D w_d k_t^d(m_d) \nonumber \\
&\;\le\; \sum_{d=1}^D w_d \dot k_t^d(m_d), \label{eq:AS:test_martingale_to_show}
\end{align}
where the inequalities use the non-negativity of $w_d$, $k_t^d(\cdot)$, and $\dot k_t^d(\cdot)$ with $\dot k(\cdot)$ defined in Eq.~\eqref{eq:AS:k_dot_convex_comb}. 

Since each $\dot k_t^d(m_d)$ is a test martingale (for $m_d = \mu_d$), their weighted sum is also a test martingale:
\begin{align*}
\mathbb{E} \! \left[  \sum_{d=1}^D w_d \dot k_t^d(m_d) \, \middle|\, \mathcal F_{t-1} \right]
&= \sum_{d=1}^D w_d \mathbb{E}\!\left[ \dot k_t^d(m_d) \,\middle|\, \mathcal F_{t-1} \right] \\
&= \sum_{d=1}^D w_d \dot k_{t-1}^d(m_d).
\end{align*}
By Ville’s inequality~\eqref{eq:AS:confseq_Mt}, the sublevel set
\[
\Bigl\{ \mathbf m : \sum_{d=1}^D w_d \dot k_t^d(m_d) \le \tfrac1\alpha \Bigr\}
\]
is therefore a valid $(1-\alpha)$-CS. Finally, since $\Kmax_t(\mathbf m)$ is upper bounded by a test martingale, its sublevel set
\[
\Bigl\{ \mathbf m : \Kmax_t(\mathbf m) \le \tfrac1\alpha \Bigr\}
\]
is a conservative outer CS and hence also valid.

\end{proof}

\section{Set Relations} \label{apdx:AS:set_relations}

\begin{theorem}
\label{apdx:AS:thm_set_relations}
Assume $\alpha \in (0, 1], D \in \mathbb{N}_{> 0}$, and $\vect{w} \in \Delta^D$ fixed with $w_d > 0$ for all $d$, then the following relations hold: 
\begin{align*}
    \Bavg \subseteq  \Bbboxellip &\subseteq \Bbbox \subseteq
        \Bmax = \Bbonf, \\
        \Bbboxellip &\subseteq \Bellip.
\end{align*}
\end{theorem}

\begin{proof}
We show the relations:

\begin{enumerate}

\item         $\Bmax = \Bbonf$.

Denote 
$$\Bmax = \left\{ \vect{x} \,\middle|\, \max_{d\in[D]} \, w_d\,k^d(x_d) \le \frac{1}{\alpha} \right\} 
\quad \text{ and } 
\quad \Bbonf = \left\{ \vect{x} \,\middle|\, \forall d \in [D]: k^d(x_d) \le \frac{1}{\alpha w_d} \right\}.
$$

We have, for any $\mathbf{x}\in[0,1]^D$,
\begin{align}
    \mathbf{x}\in\Bmax
& \;\iff\;
\max_{d\in[D]} w_d k^d(x_d) \le \tfrac{1}{\alpha}
\;\iff\;
\forall d\in[D]:\; w_d k^d(x_d) \le \tfrac{1}{\alpha} \\
& \;\iff\;
\forall d\in[D]:\; k^d(x_d) \le \tfrac{1}{\alpha w_d}
\;\iff\;
\mathbf{x}\in\Bbonf,
\end{align}
which shows that $\Bmax = \Bbonf$.

        \item  $\Bbbox \subseteq\Bmax$.

From the definition in \Cref{AS:eq:bbox},
\begin{align*}
    [l_d,u_d]
    =
    \left\{x_d\in[0,1]\mid
    w_d k^d(x_d)
    \le
    \frac{1}{\alpha}-\sum_{j\neq d} w_j k_\star^j
    \right\}.
\end{align*}
Hence
\begin{align*}
    \Bbbox
    =
    \Bigl\{ \vect{x}\in[0,1]^D \,\Big|\, \forall d:\;
    w_d k^d(x_d)\le \frac{1}{\alpha}-\sum_{j\neq d} w_j k_\star^j \Bigr\}.
\end{align*}

        Let $\vect{x} \in \Bbbox$. Then for each $d$,
        \[
            w_d k^d(x_d) \le \tfrac{1}{\alpha} - \sum_{j\neq d} w_j k_\star^j \le \tfrac{1}{\alpha}.
        \]
        Since this holds for every $d$,
        \[
            \max_d \bigl(w_d k^d(x_d)\bigr) \le \tfrac{1}{\alpha}.
        \]
        Hence $\vect{x} \in \Bmax$.

Note that the reverse direction is not necessarily true, as shown by the example in \Cref{fig:AS:apdx_counter_example_bbox}. For an equality between the two sets to hold, $k^d( \cdot )$ should have a minimum = 0 ($k_\star^d = 0$) for all $d$:

    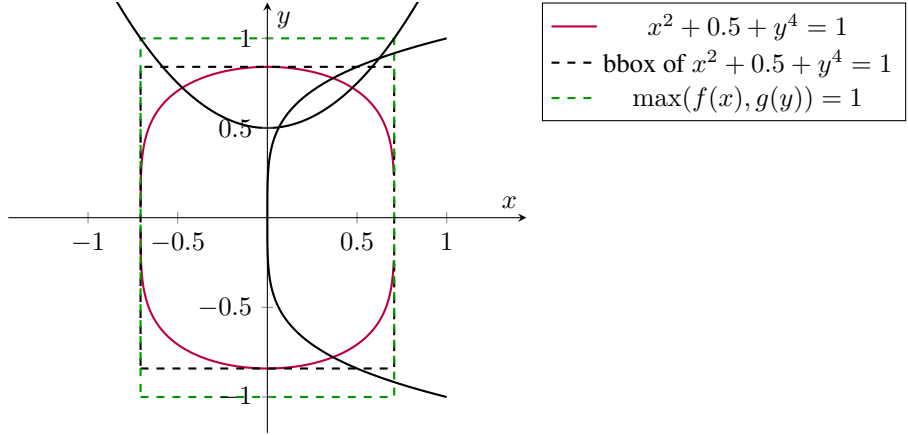
\begin{figure}[h]
\centering

\begin{tikzpicture}
\begin{axis}[
    axis lines=middle,
    xmin=-1.2, xmax=1.2,
    ymin=-1.2, ymax=1.2,
    axis equal,
    xlabel={$x$},
    ylabel={$y$},
    grid=none,
    axis on top,
    legend pos=outer north east,
]

\pgfmathsetmacro{\xr}{sqrt(0.5)}          
\pgfmathsetmacro{\yr}{pow(0.5,0.25)}      

\addplot[domain=-\xr:\xr, samples=600, thick, purple]
{(0.5 - x^2)^(1/4)};
\addlegendentry{$x^2+0.5+y^4=1$}

\addplot[domain=-\xr:\xr, samples=600, thick, purple, forget plot]
{-(0.5 - x^2)^(1/4)};

\addplot[dashed, thick, black!60!black]
coordinates {(-\xr,-\yr) (\xr,-\yr) (\xr,\yr) (-\xr,\yr)}
\closedcycle;
\addlegendentry{bbox of $x^2+0.5+y^4=1$}

\addplot[dashed, thick, green!60!black]
coordinates {(-\xr,-1) (\xr,-1) (\xr,1) (-\xr,1)}
\closedcycle;
\addlegendentry{$\max(f(x),g(y)) = 1$}

\addplot[domain=-1:1, samples=400, thick, black!70!black]
{x^2 + 0.5};

\addplot[domain=-1:1, samples=400, thick, black!70!black]
({x^4}, x);

\end{axis}
\end{tikzpicture}

\caption{Base functions: $f(x) = x^2 + 0.5$, $g(y) = y^4$. The figure depicts the set $f(x) + g(y) = 1$, its bounding-box and $\max \left( f(x), \, g(y) \right) = 1$ superset. We see that the bounding box and max are not necessarily equivalent.}
\label{fig:AS:apdx_counter_example_bbox}
\end{figure}

\item  $\Bavg \subseteq \Bbbox$.

    Let $\vect x \in \Bavg$ so that
    \begin{align*}
        &\sum_{d=1}^D w_d k^d(x_d) \le \frac{1}{\alpha} \iff \\  
        &\sum_{d\neq j}^D w_d k^d(x_d) + w_j k^j(x_j)  \le \frac{1}{\alpha} \implies\\   
   & w_j k^j(x_j) \le \frac{1} {\alpha} - \sum_{d \neq j} w_d k^d(x_d) \le \tfrac{1}{\alpha} - \sum_{d\neq j} w_d k_\star^d.
    \end{align*}
Hence, we have obtained that $ \forall j: w_j k^j(x_j) \le \tfrac{1}{\alpha} - \sum_{d\neq j} w_d k_\star^d$, i.e., $x_j \in [l_j,u_j]$ by \Cref{AS:eq:bbox}, such that $\vect x \in \Bbbox$.

\item $\Bavg \subseteq \Bellip$.

Define $\Bellip = \left\{ \vect{x} | \sum_d^D w_d \tilde k^d(x_d) \le \frac1\alpha \right\}$, with $\tilde k^d(x) \le k^d(x)$ for all $x \in [0, 1]$.

Let $\vect{x} \in \Bavg$ such that
\begin{align*}
    \sum_d^D w_d k^d(x_d) \le \frac1\alpha \Rightarrow \sum_d^D w_d \tilde k^d(x_d) \le \frac1\alpha,
\end{align*}
since $\tilde k^d(x) \le k^d(x)$, so $\vect{x} \in \Bellip$.

Note that when we only have that $\tilde k^d(x) \le k^d(x)$ for all $x \in [\xstart^d, \xend^d]$, with $(\xstart^d, \xend^d)$ defined as the end-point solutions of the interval: $\{ x_d | k^d(x_d) \le \frac{1}{\alpha w_d } \}$, the relation still holds. Indeed, since for $\vect{x} \in \Bavg$ we have that 
\begin{align*}
    \sum_d w_d k^d(x_d) \le 1/\alpha \Rightarrow w_d k^d(x_d) \le 1/\alpha \Rightarrow  k^d(x_d) \le \frac{1}{\alpha w_d}.
\end{align*}
I.e., for $\vect{x} \in \Bavg$, all of its coordinates must satisfy $k^d(x_d) \le \frac{1}{\alpha w_d}$ (and hence we have that $x_d \in [\xstart^d, \xend^d]$).

\item $\Bbboxellip \subseteq \Bbbox$ and $ \Bbboxellip \subseteq \Bellip$.

This is evident from its definition:
\begin{align*}
    \Bbboxellip = \Bbbox \cap \Bellip.
\end{align*}

\item $\Bavg \subseteq \Bbboxellip$.

Since $\Bavg \subseteq \Bbbox$ and $\Bavg \subseteq \Bellip$, it follows that
\begin{align*}
    \Bavg \subseteq \Bbbox \cap \Bellip = \Bbboxellip.
\end{align*}

    \end{enumerate}
\end{proof}

\section{Useful Lemmas}
\subsection{Quasiconvexity of $\tilde{k}$}
\begin{lemma}\label{lem:AS:qconvex_tilde}
$\tilde{k}(x)=\frac{|x-a|^p}{b^p}-c$ is quasiconvex on $\mathbb{R}$ for any $p \ge 1, \, b > 0$.
\end{lemma}

\begin{proof}
For $\gamma\in\mathbb{R}$ the sublevel set is
\[
\{x:\tilde{k}(x)\le \gamma\}=\{x:|x-a|^p\le |b|^p(\gamma+c)\}.
\]
If $\gamma+c<0$ it is empty; if $\gamma+c=0$ it is $\{a\}$; if $\gamma+c>0$ it is the interval
\[
[a-|b|(\gamma+c)^{1/p},\ a+|b|(\gamma+c)^{1/p}].
\]
Hence, every sublevel set is convex and $\tilde{k}$ is quasiconvex
\end{proof}

\subsection{Continuity of $k(x)$}
\begin{lemma} \label{lem:AS:continuity_k}
Assume the constants $\theta \in [0, 1]$, $c \in \mathbb{R}_{>0}$ and the series of fixed constants $\tilde\lambda_i\in(0,\infty)$ and $x_i \in [0, 1]$ for $i\le t$. Define, for $m\in[0,1]$,
\[
\lambda^{+}_i(m):=\min\Big(\tilde\lambda_i\,;\frac{c}{m}\Big),\qquad
\lambda^{-}_i (m):=\min\Big(\tilde\lambda_i\,;\frac{c}{1-m}\Big),
\]
with the boundary conventions $\lambda^{+}_i(0):=\tilde\lambda_i$ and $\lambda^{-}_i(1):=\tilde\lambda_i$. Then the following three functions are continuous on $[0, 1]$:
\[
k_t^{+}(m):=\prod_{i=1}^t\bigl(1+\lambda^{+}_i(m)(x_i-m)\bigr),\qquad
k_t^{-}(m):=\prod_{i=1}^t\bigl(1-\lambda^{-}_i(m)(x_i-m)\bigr),
\]
and
\[
k_t(m):=\max\big( \theta k_t^{+}(m),\, (1-\theta) k_t^{-}(m)\big).
\]
\end{lemma}

\begin{proof}
\begin{enumerate}
    \item For each $i$, the function $m\mapsto c/m$ is continuous on $(0,1]$ and diverges to $+\infty$ as $m \to 0^+$. The function $\lambda^{+}_i(m) = \min(\tilde\lambda_i\,;c/m)$ is therefore continuous on $(0,1]$ and has the limit $\tilde\lambda_i$ as $m\to 0^+$. So by defining $\lambda^{+}_i(0):=\tilde\lambda_i$ we obtain a function continuous on the whole $[0,1]$. The same argument (with $c/{1-m}$ and the convention at $m=1$) shows $\lambda^{-}_i(\cdot)$ is continuous on $[0,1]$.

\item For each $i$ the function $m\mapsto x_i-m$ is continuous and $\lambda^{+}_i(m)$ is continuous, hence their product $\lambda^{+}_i(m)(x_i-m)$ is continuous. Adding the constant $1$ gives the continuous factor
  \[
   f_i^+(m):=1+\lambda^{+}_i(m)(x_i-m).
   \]
   Analogously, define $f_i^-(m):=1-\lambda^{-}_i(m)(x_i-m)$; it is continuous on $[0,1]$.

\item Each $k_t^{+}(m)$ is a finite product of continuous factors $f_i^+(m)$, therefore continuous on $[0,1]$. The same holds for $k_t^{-}(m)$.

\item Finally, the function $m\mapsto\max(\theta k_t^{+}(m), \, (1-\theta)k_t^{-}(m))$ is the pointwise maximum of two continuous functions and is therefore continuous. This proves the claim.
\end{enumerate}
\end{proof}

\section{Bounding the Global Minimum of $k(x)$}
\label{apdx:AS:global_minimum_bound_k}
The construction of $\Bbbox$ in \Cref{sec:AS:bounding_box} requires knowledge of the global minimum of each function $k(x) \;=\; \max \{\theta k^+(x), \, (1-\theta)k^-(x) \}$. In this section, we introduce an algorithm to approximate this minimum and provide explicit guarantees on its correctness. Specifically, given an estimate $\tilde{k}_\star$, the algorithm produces bounds $[\underline{k}^\star, \overline{k}_\star]$ that are guaranteed to contain the true minimum $k_\star$. Note that we assume continuity of $k^+$ and $k^-$ which follows from \Cref{lem:AS:continuity_k}.

This discussion is also used later to certify the discretization error in the $\ell_p$-ellipsoid construction.

\begin{lemma}
Let $[a,b]\subseteq[\xstart, \xend]$ with $a<b$, and define
\[
k(x)=\max\big(\theta k^+(x),\,(1-\theta)k^-(x)\big),
\]
where $k^+$ and $k^-$ are continuous on $[\xstart,\xend]$, with $k^+$ non-increasing and $k^-$ non-decreasing, and $\theta \in [0, 1]$. Let $x^\star$ be a global minimizer of $k$ on $[a,b]$. Then
\begin{itemize}
    \item if $\theta k^+$ and $(1-\theta)k^-$ meet inside $[a, b]$ (i.e., there exists $x \in [a, b]$ with $\theta k^+(x) = (1-\theta)k^-(x)$), then
    $$
    \underbrace{\max \left( \theta k^+(b) \, ,  (1-\theta)k^-(a) \right)}_{\text{lower bound } \underline{k_\star}} \le k(x^\star) \le \underbrace{\min \left( \theta k^+(a), (1-\theta)k^-(b) \right)}_{\text{upper bound }\overline{k_\star}},
    $$
    \item if $\theta k^+$ and $(1-\theta)k^-$ do not meet on $[a, b]$, then the global minimum occurs at an endpoint and
    $$
    k(x^\star) = \min \left(k(a), k(b) \right).
    $$
\end{itemize}
\end{lemma}

\begin{proof}
Define $h(x):=(1-\theta)k^-(x)-\theta k^+(x)$. Because $k^-$ is non-decreasing and $k^+$ is non-increasing, the function $-k^+$ is non-decreasing, so $h$ is non-decreasing on $[a, b]$.
We distinguish three cases, based on the sign of $h(x)$ on $[a, b]$:
\begin{enumerate}
    \item $h(x)>0$ for all $x \in[a, b]$,
    \item $h(x)<0$ for all $x \in[a, b]$,
    \item $h(x)=0$ for some $x \in[a, b]$.
\end{enumerate}
Note that in case 3, because of the continuity of $h$ a root does indeed exist. Because of its monotonicity, this is either a single point or an interval.

We treat the cases.
\begin{enumerate}
    \item Case A - No intersection on $[a, b]$ (cases 1 and 2). \\
If $h>0$ on $[a, b]$ then $(1-\theta)k^-(x)>\theta k^+(x)$ for every $x \in[a, b]$, so $k(x)=(1-\theta)k^-(x)$ on $[a, b]$. Since $k^-$ is non-decreasing, $k$ is non-decreasing and thus attains its minimum at $x=a$. Hence $k\left(x^{\star}\right)=k(a)= \min (k(a), k(b))$. Analogously, if $h<0$ on $[a, b]$ then $k(x)=\theta k^+(x)$ on $[a, b]$. Since $\theta k^+$ is non-increasing, $k$ is non-increasing and attains its minimum at $x=b$. Hence $k\left(x^{\star}\right)=k(b)=\min (k(a), k(b))$. This proves the endpoint statement.

\item Case B - Let $h(x) = 0$ have a solution (case 3). \\
We show the bound in two steps: 1) we show that the roots of $h(x) = (1-\theta)k^-(x) - \theta k^+(x)$ correspond to the global minima of $k(x)$.  2) From this intermediate result, we can derive a lower and upper bound.

Let $x_0 \in [a, b]$ s.t. $h(x_0) = 0$, we thus have that $(1-\theta)k^-(x_0) = \theta k^+(x_0) = k(x_0)$. 
But since $k(x) = \max \left(\theta k^+(x), (1-\theta)k^-(x) \right)$, we have:
\begin{itemize}
   	\item for $x \leq x_0: k(x)=\theta k^+(x)$ \quad (since for $x \le x_0, h(x) \le 0 \Rightarrow (1-\theta)k^-(x) \le \theta k^+(x).)$
   	\item for $x \geq x_0: k(x)=(1-\theta)k^-(x)$.
\end{itemize}
But if $k(x)$ is non-increasing to the left of $x_0$, and non-decreasing to the right of $x_0$, surely, $x_0$ must be a global minimizer of $k(x)$. 

Now the bounds:
\begin{itemize}
   	\item Upper bound. \\
   	Since we just derived that the roots of $h(x)$ are global minimizers for $k(x)$, pick an $x^\star$ from this zero set. It then follows that $k(x^\star) = (1-\theta)k^-(x^{\star}) = \theta k^+(x^\star) $. Now, from the non-decreasing property of $(1-\theta)k^-$, we get $(1-\theta)k^-\left(x^{\star}\right) \leq (1-\theta)k^-(b)$, and analogously for $k$, we get $\theta k^+\left(x^{\star}\right) \leq \theta k^+(a)$. Combining these two inequalities, together with the equality at $k(x^\star)$, gives us the upper bound:
    $$
    k\left(x^{\star}\right) \leq \min \left((1-\theta)k^-(b), \theta k^+(a)\right).
    $$
Note that for any global optimizer $x^{\star}_2 \in[a, b]$ lying outside the zero set of $h(x)$, the bound still applies, since by definition of a global optimizer we have $k\left(x^{\star}\right)=k\left(x^{\star}_2\right)$.

    \item Lower bound. \\
    Because $k^-$ is non-decreasing and $a \leq x^{\star}$ , we have $(1-\theta)k^-(a) \leq (1-\theta)k^-\left(x^{\star}\right)$. Because $k^+$ is non-increasing and $x^{\star} \leq b$ we have $\theta k^+(b) \leq k\left(x^{\star}\right)$. Finally, taking the maximum gives us our bound: 
    
   $$
    \max \left((1-\theta)k^-(a), \theta k^+(b)\right) \leq \max\left((1-\theta)k^-(x^\star), \theta k^+(x^\star)\right) = k\left(x^{\star}\right).
    $$
\end{itemize}
\end{enumerate}
\end{proof}

The theorem implies two practical consequences. First, a conservative estimate $\underline{k}^\star$ of $k(x^\star)$ can be computed via a root-finding procedure applied to $h(x)$; since $h$ is non-decreasing, its roots form an interval. Second, the approximation error on this minimum can be explicitly controlled using the derived bounds. Algorithm \textsc{Estimate-$k_\star$} (\Cref{alg:AS:bound-global-min-explicit}) describes this procedure in detail.

\begin{algorithm}[tb]
  \caption{\textsc{Estimate-$k_\star$}}
  \label{alg:AS:bound-global-min-explicit}
  \begin{algorithmic}
    \STATE {\bfseries Input:} functions $k^+(\cdot), k^-(\cdot)$ on $[\xstart,\xend]$; $\theta \in[0, 1]$; tolerances $\epsilon_x,\epsilon_y$
    \STATE {\bfseries Output:} $(\underline{k_\star}, \overline{k_\star}, \hat x)$
    \STATE $a \gets \xstart,\quad b \gets \xend$
    \STATE $h(x) \gets (1-\theta) k^-(x)-\theta k^+(x)$
    \IF{$h(a)>0$ {\bfseries or} $h(b)<0$}
        \STATE $k(x) = \max \{ \theta k^+(x), (1-\theta)k^-(x) \}$
        \STATE $x_\star \gets \arg\min_{x \in \{a,b\}} k(x)$
        \STATE $k_\star \gets k(x_\star)$
      \STATE {\bfseries return} $(k_\star, k_\star, x_\star)$
    \ENDIF
    \STATE $\underline{k_\star} \gets -\infty, \quad \overline{k_\star} \gets \infty$
    \WHILE{$(b-a>\epsilon_x)$ {\bfseries or} $(\overline{k_\star}-\underline{k_\star}>\epsilon_y)$}
      \STATE $\underline{k_\star} \gets \max\!\big((1-\theta) k^-(a),\,\theta k^+(b)\big), \quad \overline{k_\star} \gets \min\!\big((1-\theta) k^-(b),\,\theta k^+(a)\big)$
      \STATE $m \gets \tfrac12(a+b)$
      \IF{$h(m)<0$}
        \STATE $a \gets m$
      \ELSIF{$h(m)>0$}
        \STATE $b \gets m$
      \ELSE
        \STATE $x_\star \gets m, \quad k_\star \gets k(m)$
        \STATE {\bfseries return} $(k_\star, k_\star, x_\star)$
      \ENDIF
    \ENDWHILE
    \STATE $\underline{k_\star} \gets \max\!\big((1-\theta) k^-(a),\,\theta k^+(b)\big), \quad \overline{k_\star} \gets \min\!\big((1-\theta) k^-(b),\,\theta k^+(a)\big)$
    \STATE $\hat x \gets \tfrac12(a+b)$
    \STATE {\bfseries return} $(\underline{k_\star}, \overline{k_\star}, \hat x)$
  \end{algorithmic}
\end{algorithm}

\section*{Part II: Outer Approximations and Practical Details}
\addcontentsline{toc}{part}{Part II: Outer Approximations and Practical Details}
\section{Bounding box} \label{apdx:AS:bounding_box}
We derive the construction of the axis-aligned bounding box $\Bbbox$ of $\Bavg$. This is achieved by projecting $\Bavg$ onto each coordinate $i$ and deriving its interval $[l_i,u_i]$.
To compute it, write
\[
\sum_{d=1}^D w_d k^d(x_d) \le \frac{1}{\alpha}
\iff
w_i k^i(x_i) + \sum_{d\neq i} w_d k^d(x_d) \le \frac{1}{\alpha},
\]
such that isolating $k^i(x_i)$ leads to
\[
w_i k^i(x_i)
\le
\frac{1}{\alpha}-\sum_{d\neq i} w_d k^d(x_d).
\]
Remark that maximizing the right-hand side gives the largest admissible level for $k^i(\cdot)$, which is achieved when $\sum_{d\neq i} w_d k^d(x_d)$ is minimized. Denote the global minimum for each $k^d(\cdot)$ as $k_\star^d \in \mathbb{R}_{\ge 0}$, then the largest level set is achieved by the constant: $\frac{1}{\alpha}-\sum_{d\neq i} w_d k_\star^d$. Now, because $k^i(\cdot)$ is quasiconvex, its sublevel set is a (possibly degenerate) interval, and the endpoints $l_i,u_i$ are therefore obtained by the outer two solutions of
\begin{align}
    w_i k^i(x_i) = \frac{1}{\alpha}-\sum_{d\neq i} w_d k_\star^d. \label{eq:AS:bbx_end_point_solutions}
\end{align}

Under the quasiconvexity alone, $k^i(\cdot)$ could in principle contain flat segments, such that more than two $x_i$ exist satisfying \Cref{eq:AS:bbx_end_point_solutions}. As such, a safe choice is to solve for the equation:
\begin{equation}
    w_i k^i(x_i) = \frac{1}{\alpha}-\sum_{d\neq i} w_d k_\star^d + \epsilon, \label{eq:AS:bbx_end_point_solutions_epsilon}
\end{equation}
where $\epsilon$ is a small positive constant (e.g., $1\times 10^{-12}$) ensuring that any two solutions found are guaranteed to contain the interval of \Cref{eq:AS:bbx_end_point_solutions}.

In practice, the two solutions to \Cref{eq:AS:bbx_end_point_solutions_epsilon} can be computed numerically, for example via two bisection searches applied to
$$
w_i k^i(x_i) - \left(\frac{1}{\alpha}-\sum_{d\neq i} w_d k_\star^d\right) - \epsilon= 0,
$$
performed on intervals on either side of a minimizer of $k^i(\cdot)$. In our implementation, we use the minimizer $x^\star_i$ (corresponding to the global minimum $k_\star^i$) as the splitting point. All global minima are obtained using \Cref{alg:AS:bound-global-min-explicit} and selecting the theoretically ``worst case'' minimum. Repeating this construction for all coordinates results in the bounding box
\[
\Bbbox = \prod_{d=1}^D [l_d,u_d],
\qquad
\Bavg \subseteq \Bbbox.
\]

\section{$\ell_p$-Ellipsoid}
\subsection{$\ell_p$-Ellipsoid Derivation}
\label{apdx:AS:ellipsoid}
We derive the $\ell_p$-ellipsoid representation for the set $\Bellip$ .
\begin{align}
    \sum_{d=1}^D w_d\!\left(\frac{|x_d-a_d|^p}{b_d^p} - c_d\right) \;\le\; \frac{1}{\alpha}.
\end{align}
Rearranging gives
\begin{align}
    \sum_{d=1}^D w_d\,\frac{|x_d-a_d|^p}{b_d^p}
    &\;\le\; \underbrace{\frac{1}{\alpha} + \sum_{d=1}^D w_d\,c_d}
    _{\;=:\; R},
\end{align}
where $R > 0$ is required for the set to be nonempty. Noting that
\begin{align}
    w_d\,\frac{|x_d-a_d|^p}{b_d^p}
    = \frac{|x_d-a_d|^p}{(b_d / \sqrt[p]{w_d})^p},
\end{align}
we can rewrite the inequality as
\begin{align}
    \sum_{d=1}^D \frac{|x_d-a_d|^p}{(b_d / \sqrt[p]{w_d})^p} \;\le\; R.
\end{align}
Dividing through $R$ results in
\begin{align}
    \sum_{d=1}^D \frac{|x_d-a_d|^p}{\left( b_d \sqrt[p]{R} / \sqrt[p]{w_d} \right)^p} \;\le\; 1.
\end{align}
We therefore identify the semi-axis lengths
\begin{align}
    b'_d := \frac{b_d\, \sqrt[p]{R}}{\sqrt[p]{w_d}},
    \quad
    R = \frac{1}{\alpha} + \sum_{j=1}^D w_j\,c_j,
\end{align}
and the $\ell_p$-ellipsoid can be written as
\begin{align}
    \sum_{d=1}^D \frac{|x_d-a_d|^p}{(b'_d)^p} \;\le\; 1.
\end{align}
The result is an axis-aligned $\ell_p$-ellipsoid, and can also be written in matrix form as follows:
$$
\|B^{-1}(\vect{x}-\vect{a})\|_p \leq 1,
$$
with
\begin{align}
    B := \mathrm{diag}\! \big( b'_1, \dots, b'_D \big)
\end{align}
with $\vect{a} \in \mathbb{R}^D$, consisting of the elements $a_1 \dots a_D$.

\subsection{Fitting $\tilde{k}(x)$ from Two Touch Points}
\label{apdx:AS:find_abc}
Given a threshold $\gamma > k_\star$, we are given two distinct points \(\xtL,\xtR\) that lie on the left and right of the global optimizer $x^\star$ so that:
\[
    k(\xtL)=k(\xtR)=\gamma + \epsilon,
\]
for a small positive constant $\epsilon$. The motivation for this small constant is discussed in \Cref{analysis_of_rx}, but it generally also helps with the feasibility of $\tilde k$.

For a fixed integer $p \ge 1$, we seek parameters \((a,b,c)\) such that
\[
    \tilde{k}(\xtL)=\tilde{k}(\xtR)=\gamma
    \qquad\text{and}\qquad
    \forall x \in[x_{\mathrm{start}},x_{\mathrm{end}}]:\ \tilde{k}(x)\le k(x),
\]
where
\[
    \tilde{k}(x)=\frac{\left|x-a\right|^p}{b^p}-c
\]

with $b > 0$ . Note that once $a, b ,$ and $c$ are obtained, $p$ can be calculated by considering a series of candidates over a small fixed grid of integers (e.g.,  $\{ 6 \dots 16 \}$ ) and choosing the $p$ that minimizes the volume of the respective $\ell_p$-ellipsoid. 

\paragraph{Deriving (a, b c)}
\begin{enumerate}
    \item \textbf{Finding \(a\).}
    From \(\tilde k(\xtL)=\tilde k(\xtR)=\gamma\) we have
    \[
        \frac{|\xtL-a|^p}{b^p}-c=\gamma,\qquad
        \frac{|\xtR-a|^p}{b^p}-c=\gamma.
    \]
    Subtracting the two equations results in
    \[
        |\xtL-a|^p=|\xtR-a|^p \Rightarrow |\xtL-a|=|\xtR-a|.
    \]
    \item And since, \(\xtL\neq \xtR\) we obtain   \[
        a=\frac{\xtL+\xtR}{2},
    \]
    i.e., \(a\) is the midpoint of \(\xtL\) and \(\xtR\).

    \item \textbf{Finding \(c\) (in terms of \(b\)).}
    Plugging one touching point (say \(\xtL\)) into \(\tilde k(\xtL)=\gamma\) gives
    \begin{align}
        \label{eq:c_from_b}
        \frac{|\xtL-a|^p}{b^p}-c=\gamma
        \quad\Rightarrow\quad
        c=\frac{|\xtL-a|^p}{b^p}-\gamma.
    \end{align}
    Thus once $b$ is chosen, \(c\) is also determined.

    \item \textbf{Finding \(b\).} \label{finding_b}
    Substituting \(c\) from \eqref{eq:c_from_b} into \(\tilde k(x)\) gives
    \[
        \tilde k(x)
        =\frac{|x-a|^p}{b^p}-\frac{|\xtL-a|^p}{b^p}+\gamma
        =\frac{|x-a|^p-|\xtL-a|^p}{b^p}+\gamma.
    \]
    The requirement \(\tilde k(x)\le k(x)\) for all \(x\in[x_{\mathrm{start}},x_{\mathrm{end}}]\) is therefore
    \[
        \frac{|x-a|^p-|\xtL-a|^p}{b^p}+\gamma \le k(x)
        \quad\Longleftrightarrow\quad
        |x-a|^p-|\xtL-a|^p \le (k(x)-\gamma)\,b^p.
    \]
    Define the function
    \[
        D(x):=k(x)-\gamma.
    \]
    We must satisfy, for every \(x\) in the domain, the inequality
    \[
        |x-a|^p-|\xtL-a|^p \le D(x)\,b^p.
    \]
    Now, consider three separate cases depending on the sign of \(D(x)\):
    \begin{itemize}
        \item If \(D(x)=0\) (i.e. \(k(x)=\gamma\)), the constraint reduces to
        \[
            |x-a|^p-|\xtL-a|^p \le 0
            \quad\Longleftrightarrow\quad
            |x-a| \le |\xtL-a|.
        \]
Thus, every point \(x\) where \(k(x)=\gamma\) must lie within the closed interval around \(a\) of radius \(|\xtL-a|\). If there exists an \(x\) with \(k(x)=\gamma\) but \(|x-a|>|\xtL-a|\), no \(\tilde k\) of the chosen form can simultaneously touch at \(\xtL,\xtR\) and remain below \(k\) at that \(x\).
        \item If  $D(x)>0$ then dividing by \(D(x)\) (positive) gives the lower bound
        \[
            \frac{|x-a|^p-|\xtL-a|^p}{D(x)} \le b^p.
        \]
        \item If \(D(x)<0\) then dividing by \(D(x)\) (negative) reverses the inequality, resulting in the upper bound
        \[
            \frac{|x-a|^p-|\xtL-a|^p}{D(x)} \ge b^p.
        \]
    \end{itemize}
    For easier notation, we define this fraction as $ R(x):=\frac{|x-a|^p-|\xtL-a|^p}{D(x)}$ (and denote it as $\underline{R}$ when it defines a lower bound and $\overline{R}$ otherwise).  A feasible \(b\) must therefore satisfy
    \begin{equation}
        \label{eq:feasible_bp}
        b^p \ge \sup_{x:\,D(x)>0} \underline{R}(x)
        \qquad\text{and}\qquad
        b^p \le \inf_{x:\,D(x)<0} \overline{R}(x),
    \end{equation}
    while every \(x\) with \(D(x)=0\) must satisfy \(|x-a|\le|\xtL-a|\). If
   $ \sup_{x:\,D(x)>0} \underline{R}(x) > \inf_{x:\,D(x)<0} \overline{R}(x),$ then the feasible interval for \(b^p\) is empty and no \(\tilde k\) exists that meets the inequality on the whole interval (under the given $p$).
\end{enumerate}

\paragraph{Practical remark.}
Numerically, one evaluates the functions on a fine grid \(x_i\) with \(i=1,\dots,n\), covering \([x_{\mathrm{start}},x_{\mathrm{end}}]\). For example choose $x_i = x_{\mathrm{start}} + (i-1)\,h,\qquad h=\frac{x_{\mathrm{end}}-x_{\mathrm{start}}}{n-1}$. Define 
\[
    \underline{b^p}  = \max_{i \in \{1 \dots n \, | \, D(x_i)>0\} } \underline{R}(x_i),
    \qquad
    \overline{b^p} = \min_{i \in \{1 \dots n \, | \, D(x_i)<0\} } \overline{R}(x_i).
\]

When the interval $[\max \{0, \underline{b^p} \}, \overline{b^p}] \neq \emptyset$ and all $i$ with  \(|D(x_i)| \le \epsilon\) satisfy \(|x_i-a|\le|\xtL-a|\), a solution exists, and a simple choice for $b^p$ is to take the midpoint:
\[
    b^p \gets \frac{\underline{b^p}  + \overline{b^p}}{2},
    \qquad
    b \gets \bigl(b^p\bigr)^{1/p}.
\]
Finally obtain \(c\) from~\eqref{eq:c_from_b}.

\paragraph{Choosing $\gamma$}
The parameter $\gamma$ controls which areas of $k(x)$ receive the tightest fit. We generally found that a simple rule, which depends on the dimension $D$ and the original sublevel-set's threshold $1/\alpha$, seemed to work well in practice. We suggest setting it to: $\gamma_d = \frac{\sqrt[p]{w_d}}{\alpha}$ as a solid default choice. When all $w_1 \dots w_D$ are equal, this corresponds to setting each $\gamma_d = \frac{1}{\alpha \sqrt[p]{D}}$.

\paragraph{Adaptive widening when the touch constraints are infeasible}
\label{apdx:AS:adaptive_touch}
When the touch-point equalities at $(\xtL, \xtR)$  do not have a solution, we relax the tightness by widening the touch points:
\begin{align}
    \xtL \gets \max(\xtL-\delta,\ x_{\text{start}}),  \quad \xtR \gets \min(\xtR+\delta,\ x_{\text{end}}), \label{eq:apdx:AS:widen_touch_pts}
\end{align}
and re-run the feasibility test. Since sufficiently large $\delta$ eventually makes the constraints feasible in practice, we find the smallest feasible $\delta$ by bisection on $\delta \in [0,\,x_{\text{end}}-x_{\text{start}}]$. \Cref{alg:AS:fit-parameters-tilde-k} implements this search. The algorithm repeatedly splits the interval into two halves, searching for the transition point from no solution to a solution. This procedure is repeated until the resulting interval has reduced to a predefined tolerance. The parameter $\gamma$ remains fixed. 

\paragraph{Domain extension to guarantee two touch points}
\label{apdx:AS:k_extension}

For a continuous quasiconvex $k$ on $[0,1]$, the equation $k(x)=\gamma + \epsilon$ may have only a solution on one side of $x^\star$. As such, the procedure from \Cref{apdx:AS:find_abc} will fail as either $\xtL$ or $\xtR$ does not exist. 
Here, we construct an extended function $k_{\text{ext}}$  by mirroring a segment of $k$, creating a second touch point at the level $\gamma + \epsilon$ while preserving quasiconvexity and ensuring $k_{\text{ext}}(x)=k(x)$ for all $x\in[0,1]$ (so all set-containment guarantees on the original domain are unchanged).

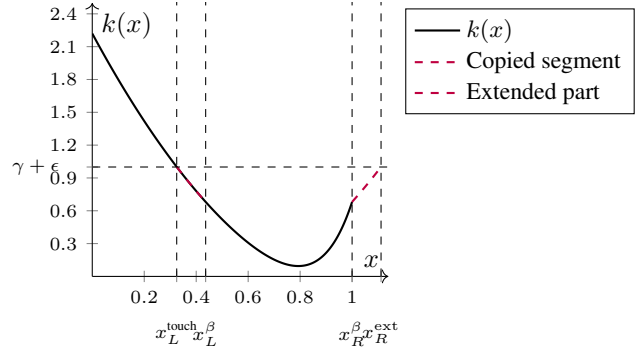
\begin{figure}[h]
	\centering
	
	\begin{tikzpicture}
		\begin{axis}[
			axis lines=middle,
			axis line style={->},
			width=5.5cm,
			height=5.2cm,
			grid=none,
			xmin=-0.0, xmax=1.14,
			ymin=0.0,  ymax=2.5,
			xlabel={$x$},
			ylabel={$k(x)$},			
			tick label style={font=\scriptsize},
			ytick distance=0.3,
			legend style={
                xshift=3mm,
				anchor=north west,
				font=\small
			},
			legend cell align=left,
			xtick distance=0.2,
			extra x ticks={0.325,0.4364466655238867,1.0,1.1114466655238866},
			extra x tick labels={$\xtL$,$\xbetaleft$,$\xbetaright$,$x_{R}^{\mathrm{ext}}$},
			extra x tick style={
				tick label style={font=\scriptsize, yshift=-2.5ex}
			},
			extra y ticks={1.0},
			extra y tick labels={$\gamma + \epsilon$},
			extra y tick style={tick label style={xshift=-1.5ex}},
			]
			
			\pgfmathsetmacro{\xAlphaL}{0.325}
			\pgfmathsetmacro{\xBetaL}{0.4364466655238867}
			\pgfmathsetmacro{\xBetaR}{1.0}
			\pgfmathsetmacro{\xAlphaR}{1.1114466655238866}
			\pgfmathsetmacro{\xMapConst}{\xBetaR + \xBetaL} 
			
			\addplot[thick, domain=0:1, samples=700]
			({x},{2*(pow(x-0.1,8) + pow(x-1.1,2)) - 0.2});
			\addlegendentry{$k(x)$}
			
			\addplot[thick, purple, dashed, domain=\xAlphaL:\xBetaL, samples=250]
			({x},{2*(pow(x-0.1,8) + pow(x-1.1,2)) - 0.2});
			\addlegendentry{Copied segment}
			
			\addplot[thick, purple, dashed, domain=\xBetaR:\xAlphaR, samples=350]
			({x},{
				2*(pow((\xMapConst - x)-0.1,8) + pow((\xMapConst - x)-1.1,2)) - 0.2
			});
			\addlegendentry{Extended part}
			
			\addplot[black, dashed, forget plot] coordinates {(\xAlphaL,0) (\xAlphaL,3)};
			\addplot[black, dashed, forget plot] coordinates {(\xBetaL,0)  (\xBetaL,3)};
			\addplot[black, dashed, forget plot] coordinates {(\xBetaR,0)  (\xBetaR,3)};
			\addplot[black, dashed, forget plot] coordinates {(\xAlphaR,0) (\xAlphaR,3)};
			
			\addplot[black, dashed, forget plot] coordinates {(0,1) (1.5,1)};
			
		\end{axis}
	\end{tikzpicture}
	
	\caption{Right part extension.}
	\label{fig:AS:k_extended}
\end{figure}

\begin{itemize}
    \item \textbf{Right extension (missing right touch point).}
Assume $k$ crosses $\gamma + \epsilon$ on the left of $x^\star$ but not on the right. I.e., there exists $x_L^{\text{touch}}\in[0,x^\star]$ (so that $k(x_L^\text{touch}) = \gamma + \epsilon$) but no $x_L^{\text{touch}}\in[x^\star,1]$. We additionally assume that $k$ is not monotone on $[0,1]$; the monotone case is discussed separately below. 

Let $x_{\beta,R}:=1$ and define $\beta := k(x_{\beta,R}) < \gamma + \epsilon$. Since $k$ is continuous and quasiconvex but not monotone, there exists $x_{\beta,L}\in(x_L^{\text{touch}},x^\star)$ with $k(x_{\beta,L})=\beta$.

Define the extended endpoint $x_R^{\text{ext}}:=x_{\beta,R}+(x_{\beta,L}-x_L^{\text{touch}})$ and
\[
k_{\text{ext}}(x) :=
\begin{cases}
k(x), & x\le x_{\beta,R},\\
k(-x + x_{\beta,R}+x_{\beta,L}), & x> x_{\beta,R}.
\end{cases}
\]
Then $k_{\text{ext}}(x_R^{\text{ext}})=\gamma + \epsilon$ and $k_{\text{ext}}$ is quasiconvex.
The process is visualized in \Cref{fig:AS:k_extended}.


\end{itemize}

\begin{itemize}
    \item \textbf{Left extension (missing left touch point).}
    If instead there exists \(\xtR\) but no \(\xtL\), the above procedure is applied symmetrically by reflecting a segment of \(k\) from the right side onto the left. The construction is entirely analogous.
    \item \textbf{Edge case:}
Finally, we consider the case where $k(x) = \max \{ \theta k^+(x), (1-\theta) k^-(x)\}$ is non-decreasing or non-increasing over its entire domain (with $k^{+}$ non-increasing and $k^{-}$ non-decreasing on $[0,1]$). To check whether these two functions ever intersect, define $h(x)=(1-\theta)k^{-}(x)-\theta k^{+}(x)$, which is non-decreasing. Hence, the curves do not intersect on $[0,1]$ iff either $h(0)>0$ (so $(1-\theta)k^{-}(x)> \theta k^{+}(x)$ for all $x$ ) or $h(1)<0$ (so $(1-\theta)k^{-}(x)<\theta k^{+}(x)$ for all $x$ ). In this ``no-intersection'' case, $k$ is monotone on $[0,1]$. We then extend the domain by mirroring: if $k$ is non-decreasing, mirror $[0,1]$ to the left to obtain $[-1,1]$; otherwise mirror it to the right to obtain $[0,2]$.
\end{itemize}

\subsection{Algorithms}\label{AS:apdx:ellipse_algorithms_and_complexity}
The algorithm to construct the $l_p$-ellipsoid is presented in \textsc{FindEllipsoid} (\Cref{alg:AS:find-ellipsoid}). It iterates over a candidate set $\mathcal{P}$, and for each $p \in \mathcal{P}$ computes an $l_p$-ellipsoid by fitting $D$ functions $\tilde k^d(\cdot)$ through \textsc{FitParameters-$\tilde{k}$} (\Cref{alg:AS:fit-parameters-tilde-k}). This algorithm performs a bisection search to find two valid touch points, and  in turn, uses the \textsc{FitParamsAtTouch}~(\Cref{alg:AS:fit-params-at-touch}) to fit the actual $\tilde k(x)$ given the required touching points. 

The \textsc{FindInterval} call (used by  \textsc{FindEllipsoid}), given a threshold $\gamma$, returns the two distinct points ($a$ and $b$) where $k$ approximately touches $\gamma$, i.e., $k(a)\approx \gamma$. This can be achieved through two root-finding calls (e.g., bisection), on $f(x) = k(x) - \gamma$, one call on the interval $[a, \hat{x})$ and the second on $(\hat{x}, b]$.  
Here, $\hat{x}$ corresponds to an estimated global minimizer of $k(x)$, obtained from \textsc{Estimate-$k_\star$} (\Cref{alg:AS:bound-global-min-explicit}).

\begin{algorithm}[tb]
  \caption{\textsc{FindEllipsoid} (grid-search over $p$)}
  \label{alg:AS:find-ellipsoid}
  \begin{algorithmic}
    \STATE {\bfseries Input:} Per-dimension functions $\{k^d(\cdot)\}_{d=1}^D$, weights $\vect{w} \in \Delta^D$, error level $\alpha$, candidate set $\mathcal{P}$, grid size $n$, $\epsilon$, $\text{max\_iters}$, $\text{tol}$
    \STATE {\bfseries Output:} Best parameters $\{(a_d,b_d,c_d)\}_{d=1}^D$ and best $p\in\mathcal{P}$

    \STATE best\_params $\gets \varnothing$, best\_p $\gets \varnothing$ , best\_vol $\gets +\infty$

    \FOR{each $p \in \mathcal{P}$}
    
      \STATE current\_params $\gets$ empty list
      \STATE feasible $\gets$ \textsc{True}

      \FOR{$d=1$ {\bfseries to} $D$}
        \STATE $\gamma_d \gets \frac{\sqrt[p]{w_d}}{\alpha}$ 

        \STATE $( k^d, x_{\text{start}}^d, x_{\text{end}}^d) \gets$ \textsc{Extend}$(k^d, \gamma_d + \epsilon)$
        \STATE $(x_{\text{start}}^d, x_{\text{end}}^d) \gets$ \textsc{FindInterval}$( k^d, \text{target}=\frac{1}{\alpha w_d} + \epsilon;\ x_{\text{start}}^d, x_{\text{end}}^d)$ 
        \STATE $S^d \gets \{(x_j, k^d(x_j))\}_{j=1}^n$ \COMMENT{$x_j$ uniform on $[x_{\text{start}}^d,x_{\text{end}}^d]$}

        \STATE $\gamma_{d} = \min\{ \gamma_d, k^d(x^d_\text{start}), k^d(x^d_\text{end}) \}$
        \STATE $(x_{L,d}^{\text{touch}}, x_{R,d}^{\text{touch}}) \gets$ \textsc{FindInterval}$( k^d, \gamma_d + \epsilon;\ x_{\text{start}}^d, x_{\text{end}}^d)$ \COMMENT{$\epsilon$ comes from \Cref{analysis_of_rx}}

        \STATE params$_d \gets$ \textsc{FitParameters-$\tilde{k}$}$(S^d, x_{L,d}^{\text{touch}}, x_{R,d}^{\text{touch}}, [x_{\text{start}}^d,x_{\text{end}}^d], p, \gamma_d, \epsilon, \text{max\_iters}, \text{tol})$

        \IF{params$_d$ is \textsc{No Solution}}
          \STATE feasible $\gets$ \textsc{False}
          \STATE {\bfseries break}
        \ENDIF

        \STATE append params$_d$ to current\_params
      \ENDFOR

      \IF{feasible = \textsc{False}}
        \STATE {\bfseries continue}
      \ENDIF

      \STATE vol $\gets$ \textsc{EllipsoidVolume}$(\{(a_d,b_d,c_d)\}_{d=1}^D,\ \vect{w},\ \alpha,\ p)$
      \IF{vol $<$ best\_vol}
        \STATE best\_vol $\gets$ vol
        \STATE best\_params $\gets$ current\_params
        \STATE best\_p $\gets p$
      \ENDIF
    \ENDFOR

    \STATE {\bfseries return} (best\_params, best\_p)
  \end{algorithmic}
\end{algorithm}

\begin{algorithm}[tb]
  \caption{\textsc{FitParameters-$\tilde{k}$} (binary search over touch-distance)}
  \label{alg:AS:fit-parameters-tilde-k}
  \begin{algorithmic}
    \STATE {\bfseries Input:} Samples $S=\{(x_i,k(x_i))\}_{i=1}^n$, initial touch points $x_L^{\text{touch}}$, $x_R^{\text{touch}}$, domain $[x_{\text{start}},x_{\text{end}}]$, $p$, $\gamma$, $\epsilon$, $\text{max\_iters}$, $\text{tol}$
    \STATE {\bfseries Output:} $(a,b,c)$ or \textsc{No Solution}

    \STATE params $\gets$ \textsc{FitParamsAtTouch}$(S, x_L^{\text{touch}}, x_R^{\text{touch}}, p, \gamma, \epsilon)$
    \IF{params exist}
      \STATE {\bfseries return} params
    \ENDIF

    \STATE $\delta_{\text{low}} \gets 0$, $\delta_{\text{high}} \gets (x_{\text{end}} - x_{\text{start}})$
    \STATE $\text{best\_params} \gets \varnothing$

    \FOR{$t=1$ {\bfseries to} $\text{max\_iters}$}
      \STATE $\delta_{\text{mid}} \gets 0.5(\delta_{\text{low}} + \delta_{\text{high}})$
      \STATE $x_L^{\text{test}} \gets \max(x_L^{\text{touch}} - \delta_{\text{mid}},\, x_{\text{start}})$
      \STATE $x_R^{\text{test}} \gets \min(x_R^{\text{touch}} + \delta_{\text{mid}},\, x_{\text{end}})$

      \STATE params $\gets$ \textsc{FitParamsAtTouch}$(S, x_L^{\text{test}}, x_R^{\text{test}}, p, \gamma, \epsilon)$
      \IF{params exist}
        \STATE $\delta_{\text{high}} \gets \delta_{\text{mid}}$
        \STATE $\text{best\_params} \gets$ params
      \ELSE
        \STATE $\delta_{\text{low}} \gets \delta_{\text{mid}}$
      \ENDIF

      \IF{$\delta_{\text{high}} - \delta_{\text{low}} \le \text{tol}$ {\bfseries and} $\text{best\_params} \neq \varnothing$}
        \STATE {\bfseries return} $\text{best\_params}$
      \ENDIF
    \ENDFOR

    \STATE {\bfseries return} \textsc{No Solution}
  \end{algorithmic}
\end{algorithm}

\begin{algorithm}[tb]
  \caption{\textsc{FitParamsAtTouch}}
  \label{alg:AS:fit-params-at-touch}
  \begin{algorithmic}
    \STATE {\bfseries Input:} Samples $S=\{(x_i,k(x_i))\}_{i=1}^n$, $x_L^{\text{touch}}$, $x_R^{\text{touch}}$, $p$, $\gamma$, $\epsilon$
    \STATE {\bfseries Output:} $(a,b,c)$ or \textsc{No Solution}

    \STATE $a \gets (x_L^{\text{touch}} + x_R^{\text{touch}})/2$ \COMMENT{Midpoint of touch points}
    \STATE $\underline{b^p} \gets 0$, $\overline{b^p} \gets +\infty$

    \FOR{$i=1$ {\bfseries to} $n$}
      \STATE $N_i \gets |x_i - a|^p - |x_L^{\text{touch}} - a|^p$
      \STATE $D_i \gets k(x_i) - \gamma$
      \IF{$|D_i| < \epsilon$ {\bfseries and} $N_i > 0$}
        \STATE return \textsc{No Solution}
      \ENDIF
      \IF{$D_i > 0$}
        \STATE $\underline{R_i} \gets N_i / D_i$ \COMMENT{Lower bound}
        \STATE $\underline{b^p} \gets \max(\underline{b^p}, \underline{R_i})$
      \ELSE
        \STATE $\overline{R_i} \gets N_i / D_i$ \COMMENT{Upper bound}
        \STATE $\overline{b^p} \gets \min(\overline{b^p}, \overline{R_i})$
      \ENDIF
    \ENDFOR
    \IF{$\underline{b^p} > \overline{b^p}$}
      \STATE return \textsc{No Solution}
    \ENDIF
    \IF{$\overline{b^p} = \infty$}
      \STATE $b^p \gets \underline{b^p}$
    \ELSE
      \STATE $b^p \gets (\underline{b^p} + \overline{b^p})/2$ \COMMENT{Average bounds}
    \ENDIF
    \STATE $c \gets |\xtL - a|^p / b^p - \gamma$ \COMMENT{Compute parameter $c$}
    \STATE return $(a, b, c)$
  \end{algorithmic}
\end{algorithm}

\paragraph{Computational complexity.} \label{AS:apdx:ellipse_complexity}
\textsc{FitParamsAtTouch}:  Given grid evaluations $S^d$ and a pair of touch points $(x_{L,d}^{\text{touch}},x_{R,d}^{\text{touch}})$, the algorithm performs a single pass over the grid and thus runs in $O(N)$ time.

\textsc{FitParameters-$\tilde{k}$}: repeatedly calls \textsc{FitParamsAtTouch} while performing a bisection search over the touch-distance parameter $\delta$ on an interval of length
$L := \delta_{\text{high}}-\delta_{\text{low}} \le x_{\text{end}}-x_{\text{start}}$. With termination tolerance \texttt{tol}, the number of bisection iterations is $T_{\text{bisect}} = O\!\left(\log\frac{L}{\texttt{tol}}\right)$. Since $L \le 1$ when $k(x)$ was not extended (since its domain is $[0, 1]$), or $L\le 2$ when $k$-extension is enabled (as in the worst case, the entire domain is copied).  Consequently, the dominant cost of fitting a single 1D $\tilde k^d$ is
\[
O\!\left(N \log\frac{L}{\texttt{tol}}\right) = O\!\left(N \log\frac{1}{\texttt{tol}}\right) 
\]
operations on precomputed samples $S$. 

\textsc{FindEllipsoid}: finally, \textsc{FindEllipsoid} repeats this procedure for each coordinate $d\in\{1,\dots,D\}$ and each candidate $p$ in a set $\mathcal{P}$ of size $N_p$.  Additionally, for each $d$ and $p$ the grid samples $S^d$ are computed, with each sample's cost denoted as $C_{\text{eval}}$ (i.e., the cost of a single evaluation of $k^d(\cdot)$).
Note that $S^d$ depends on \textsc{Extend} which requires $\gamma_d = \frac{\sqrt[p]{w_d}}{\alpha}$  (\Cref{alg:AS:find-ellipsoid}).
This is why $S^d$ is recomputed for the various $p$, but there are certainly more efficient alternatives which we do not cover here. 
Returning to \textsc{FindEllipsoid}'s complexity analysis, we obtain a total of
\[
O\!\left(D\,N_p\,N  ( \log\frac{L}{\texttt{tol}} + C_{\text{eval}} ) \right)\, 
\]
where we have ignored the comparatively small overhead of \textsc{FindInterval} and \textsc{Extend}.

\subsection{Analysis of $R(x)$ when dividing by zero} \label{analysis_of_rx}

The construction in \Cref{apdx:AS:find_abc} obtains the parameters $(a,b,c)$ such that
$k(x_i) \ge \tilde k(x_i)$ holds for all points $x_i$. Recall that the feasibility condition for $b^p$ is expressed in terms of
\[
R(x) := \frac{|x-a|^p - |\xtL-a|^p}{D(x)}, 
\qquad D(x) := k(x) - \gamma.
\]
Points with $D(x)>0$ impose a lower bound on $b^p$, while points with $D(x)<0$
impose an upper bound. When $D(x)=0$, the constraint reduces to
\[
|x-a| \le |\xtL-a|.
\]
In practice, testing the condition $D(x)=0$ exactly is numerically unstable. Therefore, the implementation (see \Cref{alg:AS:fit-params-at-touch}) treats all points with
$|D(x)|<\epsilon$ as belonging to the $D(x)=0$ case.
Here, we analyze whether this relaxed handling could eliminate critical constraints on $b^p$ and thereby violate the condition $k(x_i) \ge \tilde k(x_i)$.

\subsubsection{Analytical analysis}
\paragraph{Setup.}
Let $x^\star$ satisfy $D(x^\star)=0$, i.e., $k(x^\star)=\gamma$ (not to be confused with the global minimizer of $k(x)$). Following the choice for the touch points in~\Cref{apdx:AS:find_abc}, we assume that the two touch points $\xtL,\xtR$ are chosen such that
\[
k(\xtL) > \gamma, \quad \text{ and } \quad k(\xtR) > \gamma,
\]
such that by the quasiconvexity of $k$, we obtain the strict inequalities $\xtL < x^\star_L$ and $x^\star_R < \xtR$. In total, we have:
\begin{equation}
	x_{\mathrm{start}} \le \xtL < x^\star_L < a < x^\star_R < \xtR \le x_{\mathrm{end}}. \label{eq:apdx:AS:total_pts_structure}
\end{equation}

\paragraph{Behavior of $R(x)$ near $D(x)=0$.}
We focus on the right-side root $x^\star_R$ and analyze the limiting behavior of $R(x)$ as $x \to x^\star_R$.
Since $D(x^\star_R)=0$, $R(x)$ is undefined at this point, so we consider its one-sided limits.

The numerator
\begin{equation}
	|x-a|^p - |\xtL-a|^p \label{eq:apdx:AS:x-a-xtouch-a}
\end{equation}
evaluated at $x = x^\star_R$ is strictly negative and non-zero, because $|x^\star_R-a| < |\xtL-a|$ by construction (follows from Eq.~\eqref{eq:apdx:AS:total_pts_structure} and the fact that $|\xtL - a| = |\xtR - a|$).
We denote Eq.~\eqref{eq:apdx:AS:x-a-xtouch-a} by $-c$ for some $c > 0$. 
As $x \to {x^\star_R}^-$ we have $D(x)<0$, while as $x \to {x^\star_R}^+$ we have
$D(x)>0$. Therefore,
\[
\lim_{x \to {x^\star_R}^-} \overline{R}(x) = \lim_{x \to {x^\star_R}^-} \frac{-c}{D(x)} = +\infty,
\qquad
\lim_{x \to {x^\star_R}^+} \underline{R}(x) = \lim_{x \to {x^\star_R}^+} \frac{-c}{D(x)} = -\infty.
\]

\paragraph{Implication for feasibility bounds.}
Recall that points with $D(x)<0$ contribute upper bounds on $b^p$ via $\inf_{D(x)<0} \overline{R}(x)$, while points with $D(x)>0$ contribute lower bounds via $\sup_{D(x)>0} \underline{R}(x)$.
The divergent limits above, therefore, do not lead to finite or active constraints:
\begin{itemize}
    \item As $x \to {x^\star_R}^-$, $\overline{R}(x)\to+\infty$, which cannot tighten the
    upper bound on $b^p$.
    \item As $x \to {x^\star_R}^+$, $\underline{R}(x)\to-\infty$, which cannot tighten the lower bound on $b^p$.
\end{itemize}

Consequently, for sufficiently small $\epsilon$, treating points with $|D(x)|<\epsilon$ as belonging to the $D(x)=0$ case does not eliminate any finite, critical bounds on $b^p$. We also validate this empirically below.

\subsubsection{Empirical analysis}
\begin{figure}[h!]
    \centering
    \begin{minipage}[t]{0.49\textwidth}
        \centering
        \includegraphics[width=\linewidth]{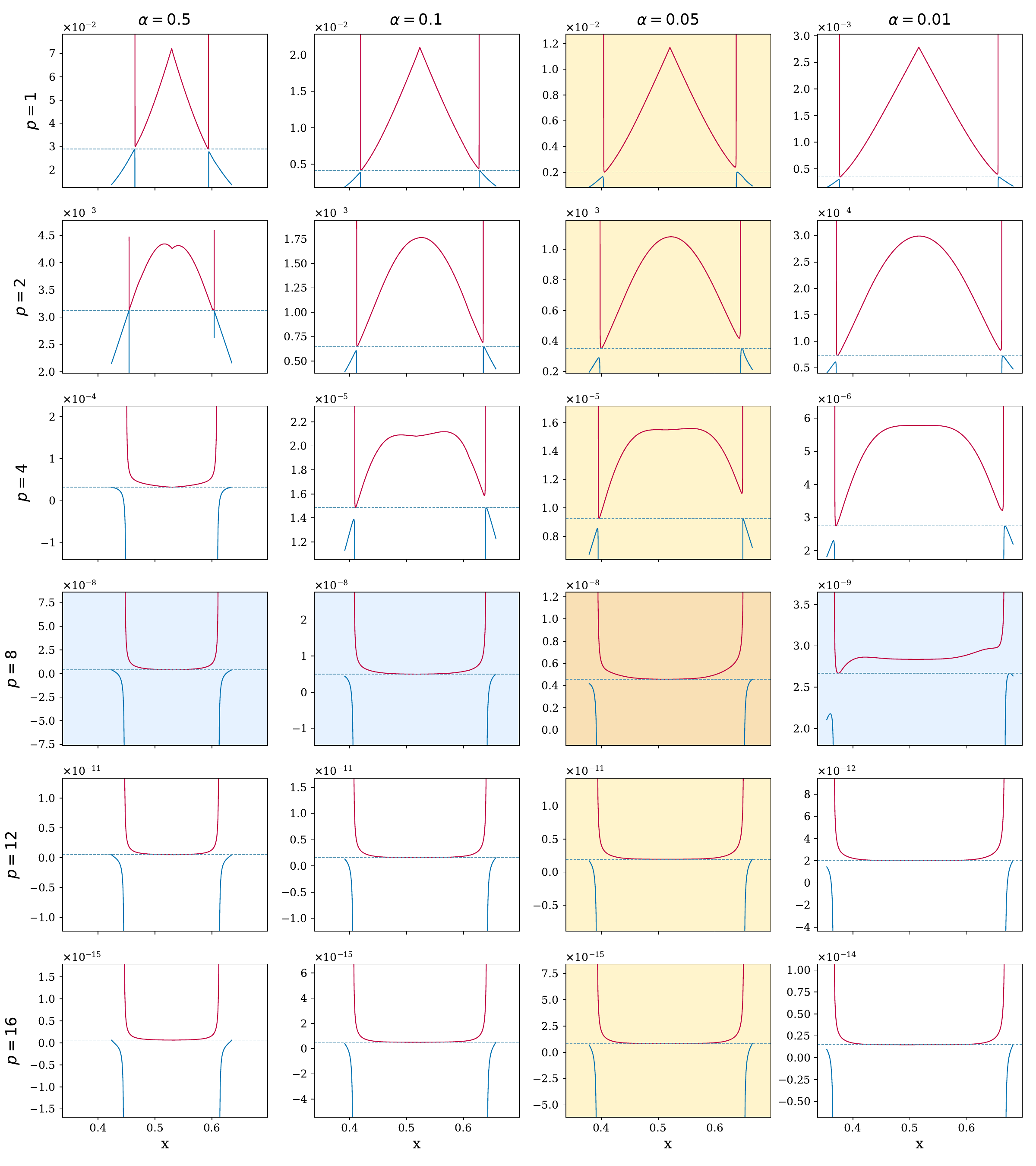}
        \caption*{(a) 50 samples}
    \end{minipage}\hspace{-0em}%
    \begin{minipage}[t]{0.49\textwidth}
        \centering
        \includegraphics[width=\linewidth]{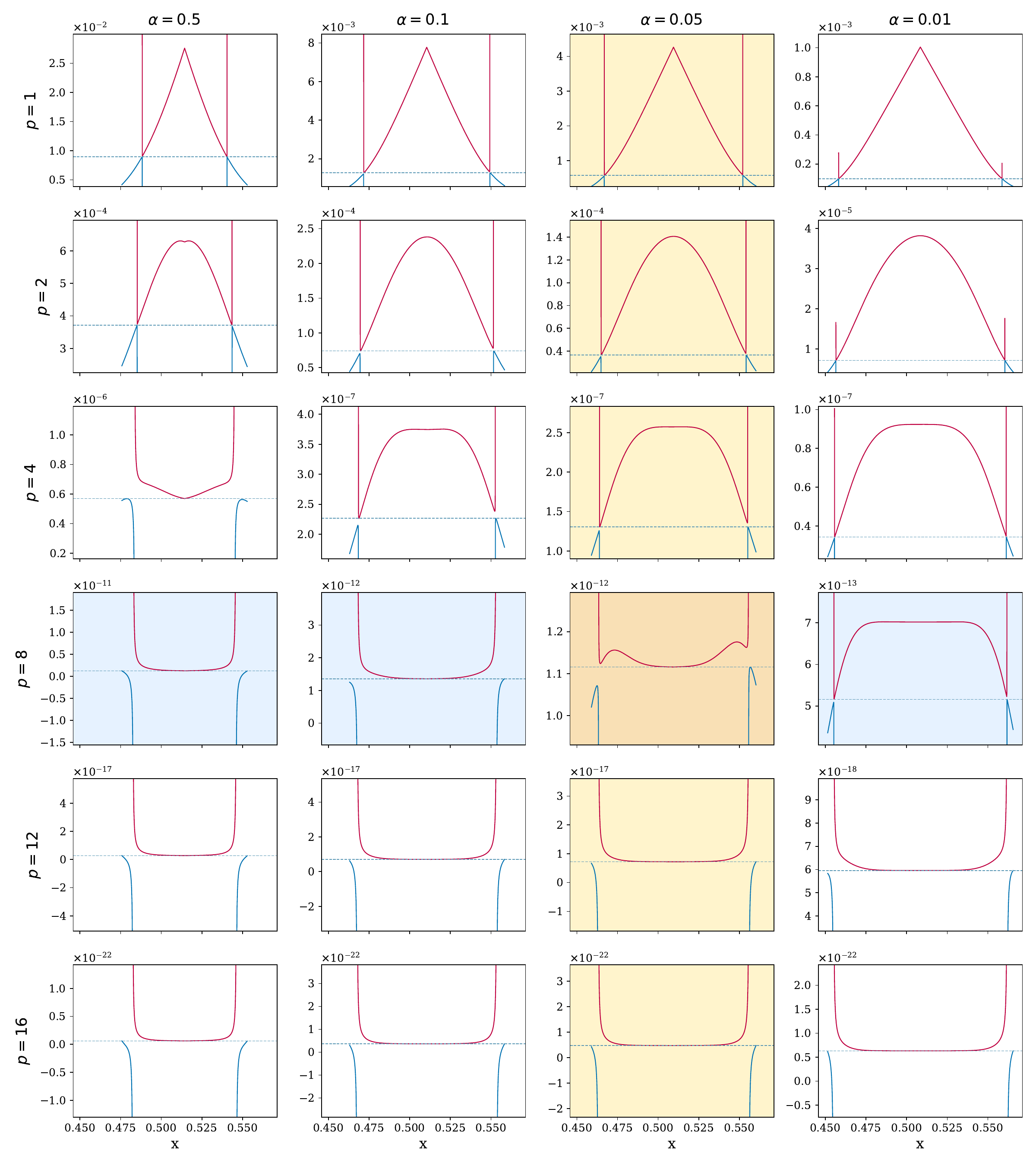}
        \caption*{(b) 500 samples}
    \end{minipage}

    \caption{Plots of $R(x)=\frac{|x-a|^p-|x_1-a|^p}{D(x)}$ for various $\alpha$ and $p$ combinations. 
    The blue part of $R(x)$ corresponds to the lower bound of $b^p$ (i.e., $D(x)>0$), and the red part to the upper bound (i.e., $D(x)<0$).}
    \label{fig:AS:b_range_comparison}
\end{figure}

\Cref{fig:AS:b_range_comparison} depicts the upper (red) and lower bound (blue) for legal $b^p$ candidates. Based on the equation:
$$
R(x) = \frac{|x-a|^p-|\xtL-a|^p}{D(x)}
$$
from \Cref{apdx:AS:find_abc}. We find that when $D(x) = k(x) - \gamma$ approaches zero, $R(x)$ diverges either to $\pm \infty$. For the upper bound $R(x)$ approaches $+ \infty$, resulting in a meaningless bound on $b^p$, and similarly, the lower bound is $- \infty$ near $D(x) = 0$. Hence, at the limit, no crucial upper/lower bound values are being ignored by treating $|D(x)| \le \delta$ as case 1 in \Cref{finding_b} (from the three items).  
Interestingly, as $p$ increases, the key boundary points move away from the points where $D(x)=0$. We also observe that the sub-regions containing these crucial boundary points appear to be somewhat convex, suggesting that future algorithms could exploit this structure to derive upper and lower bounds for $b$ more efficiently than the uniform grid evaluation used here. \\

Details to reproduce the figure:
Underlying data distribution: $y \sim \mathcal{U}(0.5, 0.8)$, where $k(x)$ is based on either 50 or 500 observations. The grid size $N$ is set to 1,000,001 pts.

\subsection{Side Note - Function Approximation with Two Distinct Touch Levels}
So far, we assumed that the level $\gamma$ is touched on both sides of the global minimizer of $k$, but this does not always hold in practice, motivating the extended construction $k_{\text{ext}}$ introduced in \Cref{apdx:AS:k_extension}.

An alternative approach is to directly work with a more general formulation in which the two sides are allowed to correspond to different levels. Specifically, we consider two points $\xtL$ and $\xtR$ satisfying
\[
k(\xtL)=\alpha, \qquad k(\xtR)=\beta,
\]
with $\alpha$ and $\beta$ not necessarily equal, and construct an approximation $\tilde{k}$ that touches $k$ at these locations.

We derive this construction for the quadratic case
\[
\tilde{k}(x)=\frac{(x-a)^2}{b^2}-c,
\]
corresponding to $p=2$. For $p>2$, the resulting parameter equations become substantially more involved, and we are not aware of closed-form solutions. Below, we present the full derivation for $p=2$ and highlight the technical difficulties that appear for $p>2$.

Assume that $\xtL, \xtR$ and the corresponding levels $\alpha, \beta$ are given. Our goal is to determine parameters $(a,b,c)$ such that
\[
\tilde{k}(\xtL)=\alpha, \qquad \tilde{k}(\xtR)=\beta,
\]
while additionally ensuring that
\[
\forall x \in [x_{\text{start}}, x_{\text{end}}]: \quad \tilde{k}(x) \le k(x).
\]
We first solve for $a$ and $c$ in terms of $b$, and subsequently derive the constraints that $b$ must satisfy.

    \begin{enumerate}
        \item \textbf{Solving for $a$ (as a function of $b$ ).}
        Setting $\tilde{k}\left(\xtL\right)=\alpha$ and $\tilde{k}\left(\xtR\right)=\beta$ and subtracting gives
        $$
        \frac{\left(\xtR-a\right)^2}{b^2}-\beta=\frac{\left(\xtL-a\right)^2}{b^2}-\alpha \Longleftrightarrow \frac{\left(\xtR-a\right)^2-\left(\xtL-a\right)^2}{b^2}=\beta-\alpha .
        $$
        Using $v^2-w^2=(v-w)(v+w)$ the numerator becomes
        $$
        \left(\xtR-a\right)^2-\left(\xtL-a\right)^2=\left(\xtR-\xtL\right)\left(\xtL+\xtR-2 a\right) .
        $$
        Hence
        $$
        \left(\xtR-\xtL\right)\left(\xtL+\xtR-2 a\right)=(\beta-\alpha) b^2,
        $$
        so
        \begin{equation}
        	a=\frac{\xtL+\xtR}{2} - 
        	\frac{(\beta-\alpha) b^{2}}{2 (\xtR-\xtL)}. \label{eq:apdx:AS:assymetric_tilde_k_isolated_a}
        \end{equation}
        For compactness define the constants
        $$
        A_0:=\frac{\xtL+\xtR}{2}, \quad K:=\frac{\beta-\alpha}{2\left(\xtR-\xtL\right)},
        $$
        so that
        $$
        a=A_0-K b^2.
        $$
        (When $\alpha=\beta$ we obtain $a=A_0$, the same midpoint as we obtained before.)
        \item \textbf{Solving for $c$ (as function of $a, b$).}
        From $\tilde{k}\left(\xtL\right)=\alpha$ we get
        $$
        \frac{\left(\xtL-a\right)^2}{b^2}-c=\alpha \quad \Longrightarrow \quad c=\frac{\left(\xtL-a\right)^2}{b^2}-\alpha .
        $$

        Thus, once $b$ is known, we can compute $a$ and then $c$.
        \item \textbf{Constraint on $b$ from $\tilde{k}(x) \leq k(x).$}
        Write $\tilde{k}(x)$ using the expression for $c$ :
        \begin{align*}
            \tilde{k}(x)
            &= \frac{(x-a)^2}{b^2}-\left(\frac{(\xtL-a)^2}{b^2}-\alpha\right)
            = \frac{(x-a)^2-(\xtL-a)^2}{b^2}+\alpha.
        \end{align*}
        Denote
        $$
        N(x):=(x-a)^2-\left(\xtL-a\right)^2 .
        $$
        Factor $N(x)$ as
        $$
        N(x)=\left(x-\xtL\right)\left(x+\xtL-2 a\right) .
        $$
        Substitute $a=A_0-K b^2$. Since $2 A_0=\xtL+\xtR$ we have
        $$
        x+\xtL-2 a=x+\xtL-2 A_0+2 K b^2=x-\xtR+2 K b^2
        $$
        and therefore
        $$
        N(x)=\left(x-\xtL\right)\left(x-\xtR\right)+2 K b^2\left(x-\xtL\right) .
        $$
        The requirement $\tilde{k}(x) \leq k(x)$ becomes
        $$
        \frac{N(x)}{b^2}+\alpha \leq k(x) \quad \Longleftrightarrow \quad N(x) \leq b^2(k(x)-\alpha) .
        $$
        Plugging the expression for $N(x)$ gives
        $$
        \left(x-\xtL\right)\left(x-\xtR\right)+2 K b^2\left(x-\xtL\right) \leq b^2(k(x)-\alpha),
        $$
        or equivalently
        $$
        b^2\left(k(x)-\alpha-2 K\left(x-\xtL\right)\right) \geq\left(x-\xtL\right)\left(x-\xtR\right) .
        $$
        Define
        $$
        D(x):=k(x)-\alpha-2 K\left(x-\xtL\right) .
        $$
        Then for each $x$ the inequality reads
        $$
        b^2 D(x) \geq\left(x-\xtL\right)\left(x-\xtR\right)
        $$

        Now consider three cases depending on the sign of $D(x)$ :
        \begin{itemize}
            \item If $D(x)=0$, the constraint reduces to $0 \geq\left(x-\xtL\right)\left(x-\xtR\right)$, which must hold for that $x$.
            \item If $D(x)>0$, dividing by $D(x)$ results in
            $$
            b^2 \geq \frac{\left(x-\xtL\right)\left(x-\xtR\right)}{D(x)}
            $$
            \item If $D(x)<0$, dividing by the negative $D(x)$ flips the inequality:
            $$
            b^2 \leq \frac{\left(x-\xtL\right)\left(x-\xtR\right)}{D(x)}
            $$
        \end{itemize}
        To satisfy $\tilde{k}(x) \leq k(x)$ for all $x \in\left[x_{\text {start }}, x_{\text {end }}\right]$ we therefore need to find $b$ such that the above inequalities hold simultaneously for every $x$ in the domain. Equivalently, letting
       \begin{align*}
            \underline{R}(x):=\frac{\left(x-\xtL\right)\left(x-\xtR\right)}{D(x)} \quad \text { for } D(x)>0, \\ 
            \overline{R}(x):=\frac{\left(x-\xtL\right)\left(x-\xtR\right)}{D(x)} \quad \text { for } D(x)<0,
       \end{align*}
        a feasible $b^2$ must satisfy
        $$
        b^2 \geq \sup _{x: D(x)>0} \underline{R}(x) \quad \text { and } \quad b^2 \leq \inf _{x: D(x)<0} \overline{R}(x)
        $$
        and additionally all $x$ with $D(x)=0$ must satisfy $\left(x-\xtL\right)\left(x-\xtR\right) \leq 0$. 
        Hence, our $k_\text{ext}$ construction instead, that works for any $p$.
    \end{enumerate}

\paragraph{Discussion on higher-degree polynomials}
When $p > 2$, however, the problem is more challenging and we are unaware on how to solve this problem in  these settings. For instance, in Eq.~\eqref{eq:apdx:AS:assymetric_tilde_k_isolated_a}, $a$ was isolated by explicitly exploiting that $p = 2$, and for larger $p$ this isolation of $a$ in terms of $b$ is no longer evident.

\subsection{A bound on the Between-Grid Violation of $\tilde k$}
\label{apdx:AS:a_bound_on_the_gridding_err}
This section belongs with the $\ell_p$-ellipsoid construction because it turns the grid-based fit into a guaranteed outer approximation. We first derive a local between-grid bound when a global minimizer lies on the grid, then generalize it to the practical case where only a lower bound $\phi$ on the minimum is known, and finally derive the safe construction used for $\BellipS$.

\subsubsection{Setup and Assumptions}
Let $k:[x_{\mathrm{start}},x_{\mathrm{end}}]\to\mathbb{R}$ and let
$x_1<\dots<x_n$ be a uniform grid with spacing $h$:
\[
x_i=x_{\mathrm{start}}+(i-1)h,\qquad h=\frac{x_{\mathrm{end}}-x_{\mathrm{start}}}{n-1}.
\]
Let $\tilde{k}(x)=\frac{|x-a|^p}{b^p}-c$ with $a\in\mathbb{R}$, $b > 0$, $c\in\mathbb{R}$ and integer $p\ge 1$.
Assume:
\begin{enumerate}
\item[\textbf{(A1)}] $k$ is quasiconvex and continuous on $[x_{\mathrm{start}},x_{\mathrm{end}}]$ (follows from \Cref{lem:AS:continuity_k}).
\item[\textbf{(A2)}] $\tilde{k}$ is quasiconvex (\Cref{lem:AS:qconvex_tilde}).
\item[\textbf{(A3)}] Grid lower bound: $\tilde{k}(x_i)\le k(x_i)$ for all $i\in\{1,\dots,n\}$.
\end{enumerate}
Define the pointwise error $\epsilon(x):=\tilde{k}(x)-k(x)$.

\subsubsection{Error Bound when a Global Minimizer Lies on the Grid}
\begin{assumption}\label{ass:grid_minimizer}
Assume there exists a global minimizer $x^\star \in [\xstart, \xend]$ that corresponds with a grid point $x_m$.
\end{assumption}

\begin{theorem}[Local interval bound]\label{thm:local_bound}
Under \textbf{(A1)--(A3)} and \Cref{ass:grid_minimizer}, for any $i\in\{1,\dots,n-1\}$ and any $x\in[x_i,x_{i+1}]$,
\[
\epsilon(x)\le
\max\{\tilde{k}(x_i),\tilde{k}(x_{i+1})\}
-
\begin{cases}
k(x_{i+1}), & i < m,\\
k(x_i), & i \ge m.
\end{cases}
\]
\end{theorem}

\begin{proof}
Fix an interval $I=[x_i,x_{i+1}]$ and an arbitrary $x\in I$. Since $x^\star=x_m$ is a global minimizer and $k$ is quasiconvex, $k$ is non-increasing on $[x_{\mathrm{start}},x_m]$ and non-decreasing on $[x_m,x_{\mathrm{end}}]$.
We distinguish two cases based on the grid index $m$.

\begin{enumerate}
\item \textbf{Case A: $i\ge m$ (interval to the right of $x_m$).}
Then $I\subseteq[x_m,x_{\mathrm{end}}]$ and $k$ is non-decreasing on $I$, hence $k(x)\ge k(x_i)$. Therefore,
\begin{align*}
\epsilon(x)
&= \tilde k(x) - k(x) \\
&\le \tilde k(x) - k(x_i) \quad\text{since } k(x)\ge k(x_i)\text{ on }I\\
&\le \max\{\tilde k(x_i), \tilde k(x_{i+1})\} - k(x_i)
\quad\text{since }\tilde k\text{ is quasiconvex.}
\end{align*}

\item \textbf{Case B: $i<m$ (interval to the left of $x_m$).}
Then $I\subseteq[x_{\mathrm{start}},x_m]$ and $k$ is non-increasing on $I$, hence $k(x)\ge k(x_{i+1})$. Thus, analogously
\[
\epsilon(x)\le \tilde k(x)-k(x_{i+1})
\le \max\{\tilde k(x_i),\tilde k(x_{i+1})\}-k(x_{i+1}).
\]
\end{enumerate}

This proves the bound with $k(x_i)$ for $i\ge m$ and $k(x_{i+1})$ for $i<m$.
\end{proof}

\subsubsection{When the Minimizer is Not a Grid Point}
The bound in \Cref{thm:local_bound} assumes that a global minimizer $x^{\star}$ corresponds to a grid point. If $x^{\star}$ is unknown, but a global lower bound $\phi$ on $k$ is available (i.e.\ $\forall x:\ k(x)\ge\phi$), we can still obtain a valid upper bound on $\epsilon(x)$. In practice, this lower bound comes directly from the global-minimum procedure in \Cref{apdx:AS:global_minimum_bound_k}, by setting $\phi \gets \underline{k_\star}$.

\begin{theorem}[Bound using a global lower bound]\label{thm:phi_bound}
Assume \textbf{(A1)--(A3)} and that a global lower bound $\phi$ is known such that $k(x)\ge \phi$ for all $x\in[x_{\mathrm{start}},x_{\mathrm{end}}]$.
Let $m := \arg\min_i k(x_i)$ be any index achieving the grid minimum.
Let $l<m$ be the largest index with $k(x_l)>k(x_m)$ (if it exists) and $r>m$ the smallest index with $k(x_r)>k(x_m)$ (if it exists).
Then every minimizer $x^\star$ lies in $(x_l,x_r)$ whenever both $l$ and $r$ exist.
Moreover, for any interval $[x_i,x_{i+1}]$,
\[
\epsilon(x)\le
\max\{\tilde{k}(x_i),\tilde{k}(x_{i+1})\}
-
L_i
\qquad\forall x\in[x_i,x_{i+1}],
\]
where
\[
L_i :=
\begin{cases}
k(x_{i+1}), & i \le l-1,\\[2pt]
k(x_i),     & i \ge r,\\[2pt]
\phi,       & l \le i \le r-1.
\end{cases}
\]
\end{theorem}

\begin{proof}
By construction, we have $x_l<x_m<x_r$ and $k(x_m)<k(x_l),k(x_r)$. The quasiconvexity of $k(\cdot)$ then implies that all true minimizers $x^{\star}$ must lie in the open interval $(x_l,x_r)$: otherwise, if $x^{\star}$ were to the right of $x_r$, applying the quasiconvexity inequality (that for any $x,y$ and any $\lambda \in [0, 1]$, one has $k(\lambda x + (1-\lambda)y) \le \max \{ k(x), k(y) \}$) to $x_m$ and $x^{\star}$ and choosing the convex combination that gives $x_r$ would force $k(x_r)\le\max\{k(x_m),k(x^{\star})\}$, which is a contradiction.
The argument for $x^{\star}$ being left of $x_l$ is analogous.

Consequently, for left-side intervals: $[x_1, x_2], [x_2, x_3], \dots, [x_{l-1}, x_l]$,  the function $k$ is non-increasing, and for the right-side intervals, $[x_r, x_{r+1}], \dots$, it is non-decreasing, so the local bound derived in \Cref{thm:local_bound} applies there. This addresses the cases for $L_i$ where $i \le l-1$ or $i \ge r$.
For intervals with $l \le i \le r-1$ (i.e., inside $[x_l, x_r]$), the monotonicity assumption fails, but we can use the global lower bound $\phi$ to obtain a weaker bound: since $k(x)\ge\phi$ for all $x$,
\[
\epsilon(x)=\tilde k(x)-k(x) \le \tilde k(x)-\phi,
\]
and consequently for an interval $[x_i,x_{i+1}]\subset(x_l,x_r)$ we may bound
\[
\epsilon(x)\le \max\{\tilde k(x_i),\tilde k(x_{i+1})\}-\phi.
\]
Hence, the case for $L_i$ with $l \le i \le r-1$ is also addressed. This concludes the proof.

\end{proof}

\begin{corollary}[Global bound]\label{cor:global_bound}
Under the assumptions of \Cref{thm:phi_bound},
\[
\epsilon_{\max}:=\max_{x\in[x_{\mathrm{start}},x_{\mathrm{end}}]}\epsilon(x)
\le
\max_{i=1,\dots,n-1}
\left(
\max\{\tilde{k}(x_i),\tilde{k}(x_{i+1})\}
-
L_i
\right).
\]
\end{corollary}

This bound is easily computable from the grid evaluations and allows verifying whether $\tilde k \le k$ holds everywhere. It also suggests an adaptive refinement scheme: whenever the bound at an interval exceeds a prescribed tolerance, subdivide that interval by adding grid points and recomputing $\tilde k$ on the refined grid.

\subsubsection{A Conservative Construction with Zero Between-Grid Violation}
\begin{proposition}[One-step shifted safe values]\label{prop:safe_shift}
Assume \textbf{(A1)--(A3)} and that a global lower bound $\phi$ is known such that $k(x)\ge \phi$ for all $x\in[x_{\mathrm{start}},x_{\mathrm{end}}]$.
Define modified grid values
\[
k_i^{\mathrm{s}}:=
\begin{cases}
k(x_{i+1}), & i< l,\\
\phi, & l\le i \le r,\\
k(x_{i-1}), & i>r,
\end{cases}
\]
(with $k_1^{\mathrm{s}}:=k(x_2)$ if $i>1$ and $k_n^{\mathrm{s}}:=k(x_{n-1})$ if $r<n$).
If $\tilde{k}$ is constructed to satisfy $\tilde{k}(x_i)\le k_i^{\mathrm{s}}$ for all $i$, then the bound on $\epsilon_{\max}$ from \Cref{cor:global_bound} is zero or less:
\[
\max\{\tilde{k}(x_i),\tilde{k}(x_{i+1})\} - L_i
\le 0
\qquad\forall i,
\]
and therefore $\epsilon_{\max}\le 0$.
\end{proposition}

\begin{proof}
We treat the cases:
\begin{enumerate}
    \item If $i\ge r$, then the bound from \Cref{cor:global_bound} is given by
    $$
    \epsilon(x) \le \max\{\tilde k(x_i), \tilde k(x_{i+1})\} - k(x_i).
    $$
    Then, since by construction $\tilde k(x_i) \le k_i^s$ and $\tilde k(x_{i+1}) \le k_{i+1}^s$, we get:
    \begin{align*}
        \max\{\tilde k(x_i), \tilde k(x_{i+1})\} - k(x_i)
        & \le  \max \{ k_i^s, k_{i+1}^s\} - k(x_i).
        \end{align*}
        And since $k_i^{\mathrm{s}}=k(x_{i-1})$ and $k_{i+1}^{\mathrm{s}}=k(x_i)$, we obtain
        \begin{align*}
            \max \{ k_i^s, k_{i+1}^s\} - k(x_i) & = \max \{ k(x_{i-1}), k(x_{i})\} - k(x_i) \quad \text{By definition of } k^s  \\
            &= 0 \quad \text{Since } k \text{ is non-decreasing on } [x_r,\xend]
        \end{align*}
    To understand why $k$ is non-decreasing on $[x_r, \xend]$, we refer to \Cref{thm:phi_bound}, first sentence in paragraph two of the proof. Thus, we have that the per-interval bound is $\le 0$ for $i\ge r$.

    \item If $i\le l-1$, then the proof is analogous to 1) by using the non-increasing property of $k$ on the left $[\xstart, x_l]$.
    \item If $l \le i \le r - 1$, then the bound is given by
    \begin{align*}
        \epsilon(x)
        & \le \max\{\tilde k(x_i), \tilde k(x_{i+1})\} - \phi \\
        & \le \max\{k_i^s, k_{i+1}^s\} - \phi \quad \text{since by construction $\tilde k(x_i) \le k_i^s$ and $\tilde k(x_{i+1}) \le k_{i+1}^s$} \\
        &= \max \{\phi, \phi \} - \phi = 0.
    \end{align*}
    Making the per-interval bound $l \le i \le r - 1$.
\end{enumerate}
\end{proof}

This construction can thus be viewed as shrinking the feasible set of $\tilde k$ to account for the worst-case between-grid behavior under the quasiconvexity assumptions of $k$.

\section*{Part III: Experiments}
\addcontentsline{toc}{part}{Part III: Experiments}
\section{Experimental Setup \& Additional Experiments} \label{apdx:AS:extra_experimetns}

\subsection{Setup} \label{apdx:AS:experimental_setup}
\subsubsection{Data and various experiment parameters.}
We consider observations generated from two multivariate distributions. The first consists of i.i.d.\ samples drawn uniformly from the unit ball, so that each observation satisfies $\lVert \vect{y} \rVert \le 1$. This setting deliberately favors confidence-sphere baselines, which explicitly assume a global norm bound of this form, whereas our method only requires coordinate-wise bounds $y_d \in [0,1]$. The second distribution is a Bernoulli mixture: half of the coordinates follow $y_d\sim\mathrm{Bernoulli}(0.5)$, while the remaining coordinates are shifted and rescaled to lie in $[0.5,1]$, again with success probability $0.5$. In this case, Bounded UP and our constructions are advantaged, as the spherical baselines require rescaling the data to satisfy the $\lVert\vect{y}\rVert \le 1$ constraint. For the experiments in~\Cref{AS:apdx:ablation_alpha_corelation}, we additionally consider a highly correlated distribution with correlation parameter $\rho=0.98$. Each coordinate is uniform on $[0,1]$, but the coordinates within an observation are nearly equal, producing a narrow cloud along the main diagonal.

Across all experiments, methods assuming a norm bound $\lVert \vect{y} \rVert \le 1$ (i.e., Banach Sphere and Conf Sphere) receive data that are appropriately shifted and scaled to lie within the unit ball, whereas for our method and Bounded UP, the data are scaled to lie in $[0,1]^D$. Additional details, such as method parameters and how volume is computed, are summarized in \Cref{apdx:AS:experimental_setup}.

Generally, we set the number of observations to 1,000, unless stated otherwise. For the volume-versus-time experiments in the main text, we fix the dimension to $D = 5$. For the small-scale experiments involving Bounded UP, we use $D = 2$. All portfolio weights and Bonferroni weights are chosen uniformly as $w_d = 1/D$. Empirical per-time-step coverage is evaluated using $200$ independent repetitions for each configuration. The error level $\alpha$ is set to 0.05 across all experiments.

\subsubsection{Hedged Capital Process and Capital Stakes} \label{apdx:AS:hedged_capital_stakes}
Following the discussion from \Cref{sec:AS:one_dim_conf_seq}, we summarize the authors' concrete suggestions for the stakes ($\lambda_i^+(\cdot)$ and $\lambda_i^-(\cdot)$), making them adaptive to the variance of the observed data:
\[
\widetilde\lambda_t
\;=\;
\sqrt{\frac{2\log(2/\alpha)}{\widehat{\sigma}_{t-1}^2\,t\,\log(t+1)}} ,
\]
where
\[
\widehat{\mu}_{t}
=\frac{\frac12 + \sum_{i=1}^{t}y_i}{t+1},
\quad
\widehat{\sigma}_{t}^2
=\frac{\frac14 + \sum_{i=1}^{t}(y_i-\widehat\mu_{i})^2}{t+1}.
\]
Finally, the stakes are then defined as:
\begin{equation}
    \lambda_t^+(m)
=\min\Bigl(\widetilde\lambda_t,\;\frac{c}{m}\Bigr),
\,
\lambda_t^-(m)
=\min\Bigl(\widetilde\lambda_t,\;\frac{c}{{1-m}}\Bigr), \label{eq:AS:apdx:stakes_min}
\end{equation}
where $c$ is a constant ($c=1/2$ or $3/4$). In our experiments, we set $c = 3/4$ and its parameter $\theta$ in Eq.~\eqref{Eq:AS:k_t_hedged_capital} to $1/2$.

The term stake is inspired by $k_t^\pm$'s interpretation as a \emph{betting process}. $k_t^{+}(m)$ represents the accumulated wealth after $t$ sequential bets on the hypothesis $\mu\ge m$. At round $i$, the bettor risks a fraction $\lambda_i^{+}(m)$ of their current wealth, resulting in the multiplicative update $k_i^{+}(m)=k_{i-1}^{+}(m)\left(1+\lambda_i^{+}(m)\left(y_i-m\right)\right) .$
Here, $\lambda_i^{+}(m) \in[0,1 / m)$ is the ``stake'', chosen based on the past $i-1$ observations. An analogous interpretation holds for $k_t^{-}(m)$, with stakes $\lambda_i^{-}(m) \in[0,1 /(1-m))$, corresponding to bets on $\mu<m$. This interpretation motivates the term \emph{betting-based} CS.

\subsubsection{Ellipsoid construction.}
The $\ell_p$-ellipsoid construction uses a fixed one-dimensional grid of size $N=2001$ for each coordinate. $N$ remains fixed independent of the interval width $[\xstart, \xend]$. The exponent $p$ is selected from the candidate set $\{6, \dots, 16\}$. The touching parameter is set to $\gamma_d = 1 / (\alpha \sqrt[p]{D})$.

\subsubsection{Volume estimation.}
When the volume of a region is not available in closed form, for example, for ellipsoid- or sphere-based regions intersecting the domain $[0,1]^D$, or for $\Bavg$, $\Bbboxellip$ and Bounded UP, we resort to Monte Carlo estimation. Specifically, volumes are estimated using $5$ independent runs of $25{,}000$ samples each. We report the mean volume estimate, and visualize uncertainty via a transparent overlay corresponding to a $95\%$ standard-deviation band.

Whenever possible, sampling is restricted to a tight outer bounding box. For $\Bavg$ and $\Bbboxellip$, we sample uniformly from the bounding box defined by $\Bbbox$. For $\Bellip$ and the sphere-based baselines, we use the minimal axis-aligned bounding box enclosing the ellipsoid. For Bounded UP, we are not aware of a tighter sampling region, and therefore sample uniformly from the full domain $[0,1]^D$.

Note that as the number of dimensions increases, the accuracy of the Monte Carlo volume estimate decreases, requiring more samples to compensate. We analyze this accuracy both theoretically and empirically in Ablation~\Cref{apdx:AS:ablation_monte_carlo_volume_estimation_accuracy}.

\subsubsection{Compute resources.} \label{apdx:AS:compute_resources}
All experiments were run on a single laptop, without GPU acceleration or specialized hardware. The machine has an Intel Core i7-13650HX processor and 16\,GB of RAM. We used this same environment for all reported runs in the paper and appendix.

\subsubsection{Evaluation time.}
\Cref{tab:AS:evaluation_time} reports mean wall-clock times for method construction and volume evaluation at $t=100$. $\Bavg$ and Bounded UP are not included due to their exponential computational cost.
But Bounded UP would require the same exponential number of calls as $\Bavg$ for its respective wealth function with the difference that each evaluation is much more costly in Bounded UP. 

\begin{table*}
    \centering
    \caption{Mean wall-clock evaluation time in seconds of computing the volume at $t=100$.}
    \label{tab:AS:evaluation_time}
    \scriptsize
    \setlength{\tabcolsep}{2.5pt}
    \renewcommand{\arraystretch}{0.95}
    \resizebox{\textwidth}{!}{%
    \begin{tabular}{c cccc ccc}
        \toprule
        & \multicolumn{4}{c}{Hedged-capital methods} & \multicolumn{3}{c}{External baselines} \\
        \cmidrule(lr){2-5}\cmidrule(lr){6-8}
        $D$ & $\Bbbox$ & $\Bbonf$ & $\BellipS$ & $\BbboxellipS$ & Banach Sphere & Conf Sphere & Sub-$\psi$ \\
        \midrule
        2  & 0.006036 & 0.006204 & 1.613209 & 1.667916 & 0.000830 & 0.000841 & 0.000812 \\
        3  & 0.009015 & 0.008886 & 2.383790 & 2.418203 & 0.000746 & 0.000908 & 0.000954 \\
        4  & 0.011913 & 0.011997 & 3.183447 & 3.186238 & 0.000731 & 0.000789 & 0.000742 \\
        5  & 0.015119 & 0.015322 & 4.041727 & 4.063602 & 0.000732 & 0.000796 & 0.000791 \\
        6  & 0.017873 & 0.017848 & 4.827297 & 4.798591 & 0.000723 & 0.008539 & 0.022119 \\
        7  & 0.020892 & 0.021427 & 5.656130 & 5.695633 & 0.000740 & 0.012678 & 0.071832 \\
        8  & 0.023990 & 0.023965 & 6.405785 & 6.409076 & 0.000774 & 0.013652 & 0.130039 \\
        9  & 0.027374 & 0.027274 & 7.164873 & 7.200621 & 0.000744 & 0.013114 & 0.144640 \\
        10 & 0.030424 & 0.031048 & 7.943246 & 7.972349 & 0.000745 & 0.013018 & 0.163916 \\
        \bottomrule
    \end{tabular}%
    }
\end{table*}
\FloatBarrier


\subsection{Small-Scale Experiments}
Because the Bounded UP CS's $O(t^{D+1})$ complexity of verifying membership is too costly to evaluate for larger settings (recall that we need to estimate its volume through Monte Carlo sampling, requiring another $5 \times 25{,}000$ samples at each $t$), we present a smaller-scale experiment to allow for comparison with this baseline as well.

\begin{figure*}[ht]
    \centering
    \includegraphics[width=\textwidth]{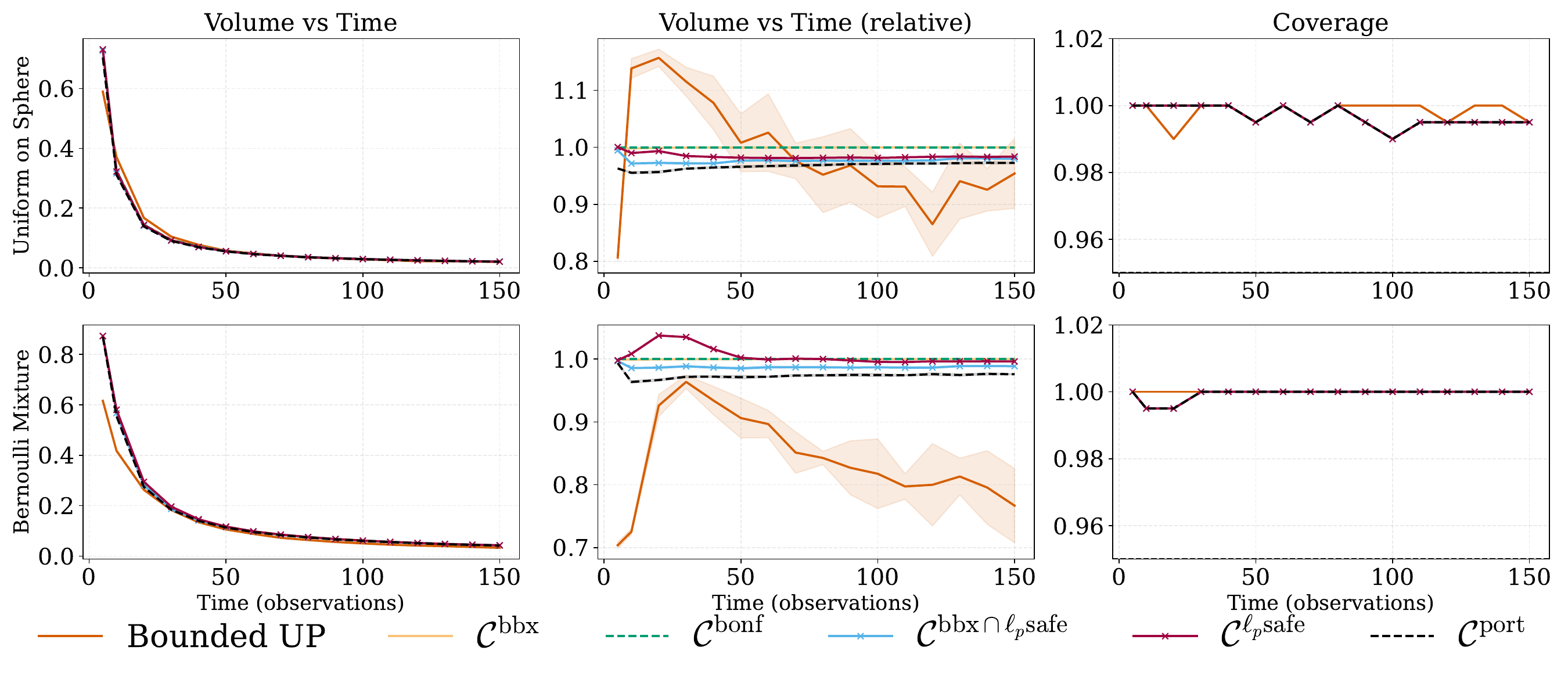}
    \caption{Small-scale comparison on synthetic data that includes Bounded UP.}
    \label{fig:AS:UP_small_scale}
\end{figure*}

\subsection{$\Bellip$ vs $\BellipS$}
\Cref{fig:AS:conservative_vs_non_conservative} compares the volumes obtained by the safe and non-safe variants, which perform remarkably similarly. The additional computational cost of the safe variant, measured in the number of $k^d_t(\cdot)$ calls, is quantified in \Cref{tab:AS:hedged_volume_comparison}.

\begin{figure*}[ht]
    \centering
    \includegraphics[width=\textwidth]{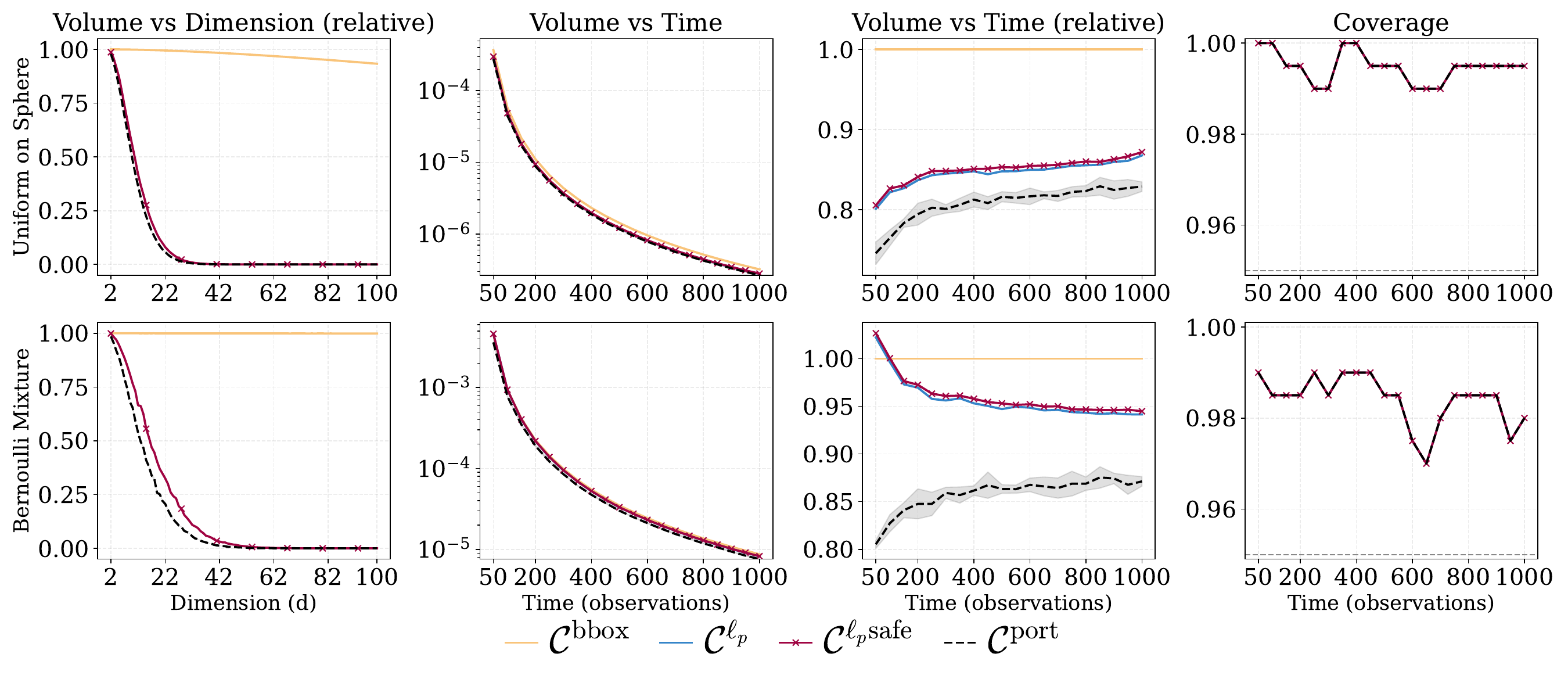}
    \caption{Comparison between the ``safe'' and non-safe $\ell_p$-ellipsoids.}
    \label{fig:AS:conservative_vs_non_conservative}
\end{figure*}

\begin{table}[t]
	\centering
	\caption{Safe vs non-safe approximations: number of $k^d_t(\cdot)$ calls to compute the volume at $t= 100$ (sorted from most to least efficient).}
	\label{tab:AS:hedged_volume_comparison}
	\scriptsize
	{%
		\begin{tabular}{c c c c c }
			\toprule
			$D   $
			& $\Bellip$
			& $\BellipS$
			& $\Bbboxellip$
			& $\BbboxellipS$
			\\
			\midrule
			2  & $1.07\!\times\!10^{5}$ & $1.60\!\times\!10^{5}$ & $1.08\!\times\!10^{5}$ & $1.61\!\times\!10^{5}$  \\
			3  & $1.62\!\times\!10^{5}$ & $2.40\!\times\!10^{5}$ & $1.63\!\times\!10^{5}$ & $2.41\!\times\!10^{5}$  \\
			4  & $2.15\!\times\!10^{5}$ & $3.20\!\times\!10^{5}$ & $2.16\!\times\!10^{5}$ & $3.21\!\times\!10^{5}$  \\
			5  & $2.69\!\times\!10^{5}$ & $4.00\!\times\!10^{5}$ & $2.70\!\times\!10^{5}$ & $4.01\!\times\!10^{5}$  \\
			6  & $3.23\!\times\!10^{5}$ & $4.80\!\times\!10^{5}$ & $3.25\!\times\!10^{5}$ & $4.82\!\times\!10^{5}$  \\
			7  & $3.77\!\times\!10^{5}$ & $5.61\!\times\!10^{5}$ & $3.79\!\times\!10^{5}$ & $5.63\!\times\!10^{5}$  \\
			8  & $4.31\!\times\!10^{5}$ & $6.41\!\times\!10^{5}$ & $4.33\!\times\!10^{5}$ & $6.43\!\times\!10^{5}$  \\
			9  & $4.85\!\times\!10^{5}$ & $7.22\!\times\!10^{5}$ & $4.87\!\times\!10^{5}$ & $7.24\!\times\!10^{5}$  \\
			10 & $5.40\!\times\!10^{5}$ & $8.03\!\times\!10^{5}$ & $5.43\!\times\!10^{5}$ & $8.06\!\times\!10^{5}$  \\
		\end{tabular}
	}
\end{table}

\subsection{Ablation Studies}
\subsubsection{Ellipsoid Parameters}
We evaluate three parameters: for the $\ell_p$-ellipsoid $\Bellip$ we consider effect of the touching parameter $\lambda_d$ and $p$ on the volume. For its safe construction $\BellipS$, we look at the impact of the grid size $N$.

\Cref{fig:AS:ellipse_vs_p} displays the impact of the parameter $p$, whereas the impact on the touch level $\gamma_d$ is reported in \Cref{fig:AS:ellipse_vs_c_touch}. Lastly, the effect on the grid size for the safe ellipsoid is shown in \Cref{fig:AS:conservative_vs_grid}. In all three plots $\Bbonf$ and $\Bavg$ are included for reference.

\begin{figure*}[ht]
    \centering
    \includegraphics[width=\textwidth]{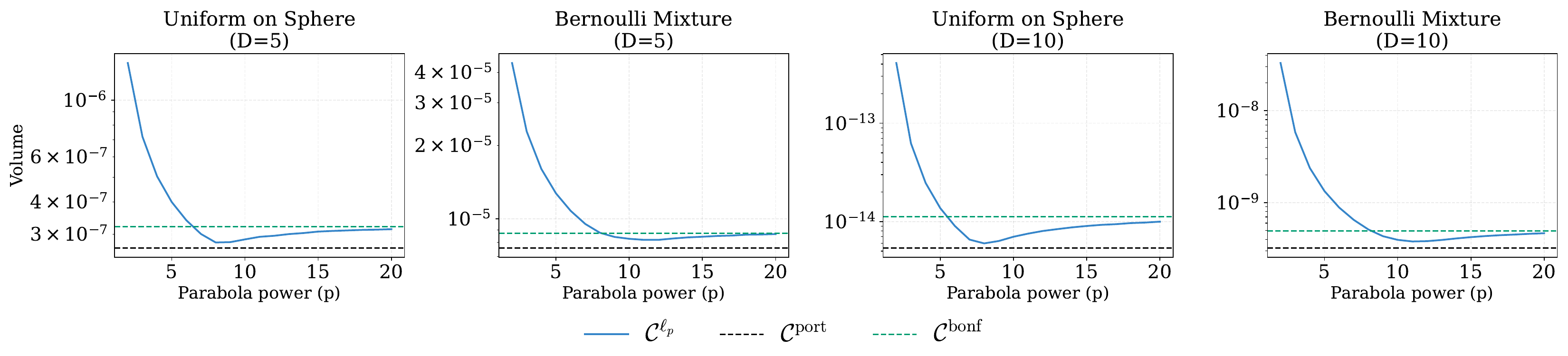}
    \caption{Impact on the exponent $p$.}
    \label{fig:AS:ellipse_vs_p}
\end{figure*}

\begin{figure*}[ht]
    \centering
    \includegraphics[width=\textwidth]{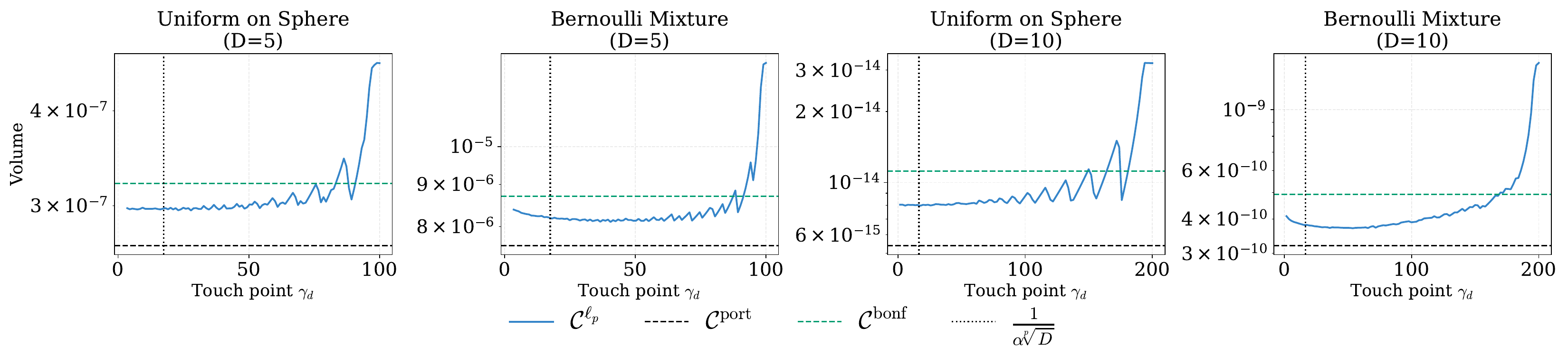}
    \caption{Impact on the touch level $\gamma_d$. The vertical indicator marks our default choice across all experiments. For an isolated comparison, the parameter $p$ is kept fixed at $p=12$.}
    \label{fig:AS:ellipse_vs_c_touch}
\end{figure*}

\begin{figure*}[ht]
    \centering
    \includegraphics[width=\textwidth]{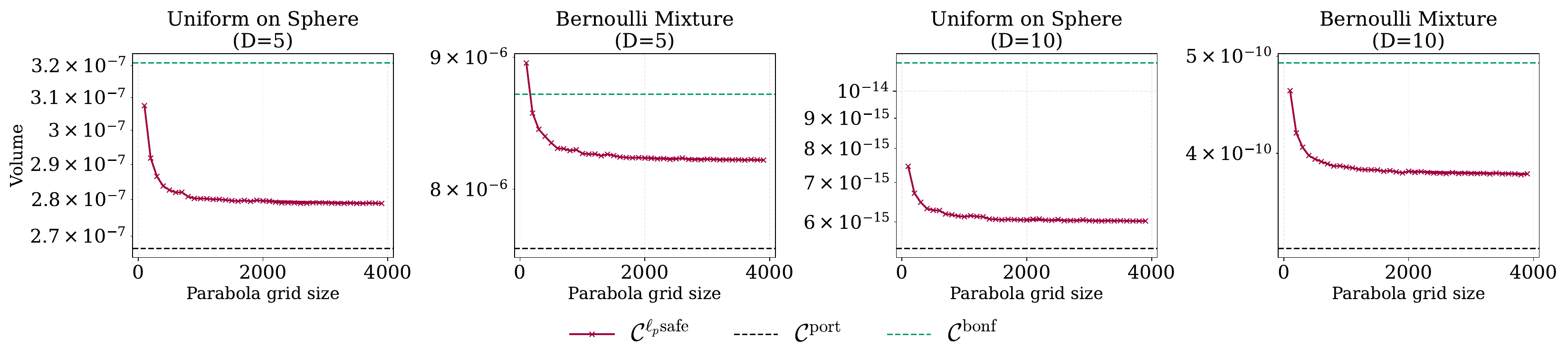}
    \caption{Impact on the grid size $N$.}
    \label{fig:AS:conservative_vs_grid}
\end{figure*}
\FloatBarrier

\subsubsection{Accuracy of Monte Carlo Volume Estimation} \label{apdx:AS:ablation_monte_carlo_volume_estimation_accuracy}

Since $\Bbboxellip$  in~\Cref{AS:sec:intersection_bbox_ellip} requires Monte Carlo sampling to estimate the volume, we analyze how accurate the estimate remains as the dimension $D$ increases.

Let
$$
\rho_D=
\frac{\operatorname{Vol}\!\left(
	\Bbbox\cap\Bellip
	\right)}
{\operatorname{Vol}\!\left(\Bbbox\right)}
$$
denote the acceptance probability when sampling uniformly from the bounding box. With $M$ independent Monte Carlo samples, the number of accepted points satisfies
$$
H\sim\operatorname{Binomial}(M,\rho_D),
\qquad
\mathbb E[H]=M\rho_D.
$$
The Monte Carlo estimate of  $\operatorname{Vol}(\Bbbox\cap
\Bellip)$ is:
$$
\widehat V =
\operatorname{Vol}(\Bbbox)
\frac{H}{M}.
$$
Its accuracy can be quantified through the relative standard error, defined as the standard deviation of $\widehat{V}$ divided by the target volume ${\operatorname{Vol}(\Bbbox\cap
	\Bellip)}$. Since $\operatorname{Var}(H) = M \cdot \rho_D (1- \rho_D)$, one can derive this standard error as: 
$$
\frac{\sqrt{\operatorname{Var}(\widehat V)}}
{\operatorname{Vol}(\Bbbox\cap
	\Bellip)} =
\sqrt{\frac{1-\rho_D}{M \cdot \rho_D}}.
$$
Hence, achieving relative standard error at most $\varepsilon$ requires
$$
M\geq\frac{1-\rho_D}{\rho_D\varepsilon^2}.
$$
Although $\rho_D$ need not decrease exponentially in every setting, it can become very small as $D$ grows, and our method provides no dimension-free bound on it. Thus, the number of samples required for an accurate volume estimate may grow rapidly with the dimension. 

\Cref{AS:apdx:tbl:vol_est_accept_rate_acc} validates empirically how the relative error is affected by the dimension for $M=25{,}000$ samples.

\begin{table}[t]
	\centering
	\caption{Monte Carlo acceptance rates and resulting estimation accuracy for different dimensions. 
		The number of accepted points is computed for a budget of $25{,}000$ Monte Carlo samples.}
	\label{tab:mc_volume_accuracy}
	\small
	\begin{tabular}{lrrrr}
		\toprule
		Distribution 
		& $D$ 
		& Estimated acceptance rate $\rho_D$ 
		& Accepted points
		& Rel. std. error \\
		\midrule
		Bernoulli mixture 
		& 20  & $0.359$              & $8{,}987$           & $0.844\%$ \\
		 
		& 50  & $1.12\times10^{-2}$  & $280$               & $5.94\%$ \\
		 
		& 100 & $3.13\times10^{-6}$  & $7.84\times10^{-2}$ & $357\%$ \\
		\addlinespace
		Uniform on Sphere
		& 20  & $0.118$              & $2{,}947$           & $1.73\%$ \\
		 
		& 50  & $7.17\times10^{-5}$  & $1.79$              & $74.7\%$ \\
		 
		& 100 & $2.90\times10^{-12}$ & $7.25\times10^{-8}$ & $3.71\times10^{5}\%$ \\
		\bottomrule
	\end{tabular}
	\label{AS:apdx:tbl:vol_est_accept_rate_acc}
\end{table}

Thus, box sampling remains effective at $D=20$ in these examples, but deteriorates substantially by $D = 50$, particularly for the uniform-on-sphere setting. It is nevertheless far better than sampling from $[0,1]^D$.
\FloatBarrier

\subsubsection{Influence on Significance Level $\alpha$} \label{AS:apdx:ablation_alpha_corelation}
\Cref{AS:apdx:fig:alpha0p01,AS:apdx:fig:alpha0p05,AS:apdx:fig:alpha0p1} compare the CSs for $\alpha\in\{0.01,0.05,0.10\}$, respectively. Across all settings, the estimated cumulative probability of excluding the true mean by $t=1000$ remains below the level $\alpha$. The relative-volume rankings are also stable across significance levels: our approximations remain close to the portfolio region and improve substantially over previous tractable baselines.
\begin{figure}
	\centering
	\includegraphics[width=\linewidth]{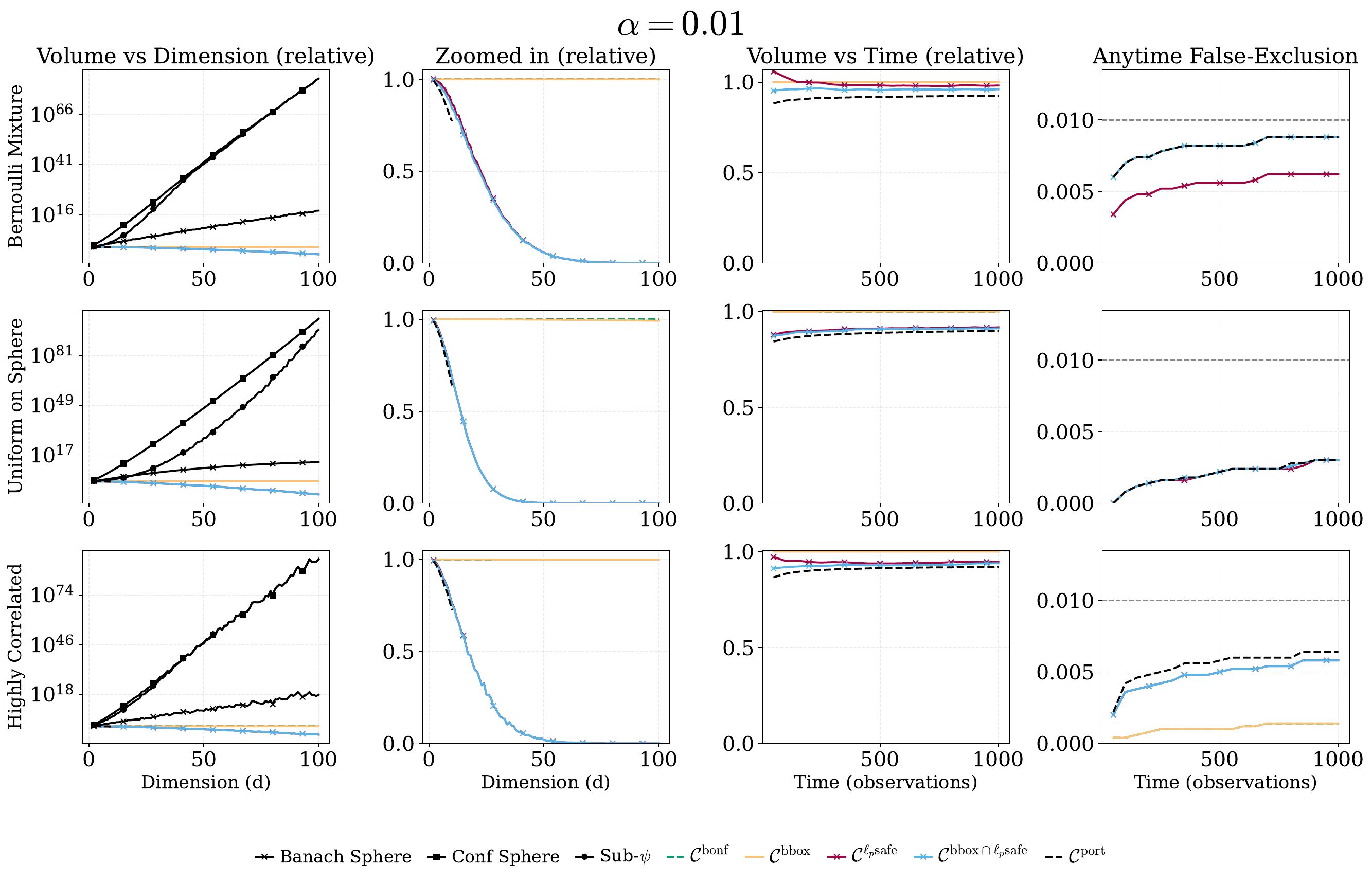}
	\caption{Results for $\alpha=0.01$. Rows correspond to the Bernoulli-mixture, uniform-sphere, and highly correlated settings. Columns show relative volume versus dimension (with baselines and a zoom-in), relative volume versus time, and the cumulative false-exclusion rate; the dashed horizontal line marks the nominal $\alpha$.}
	\label{AS:apdx:fig:alpha0p01}
\end{figure}
\begin{figure}
	\centering
	\includegraphics[width=\linewidth]{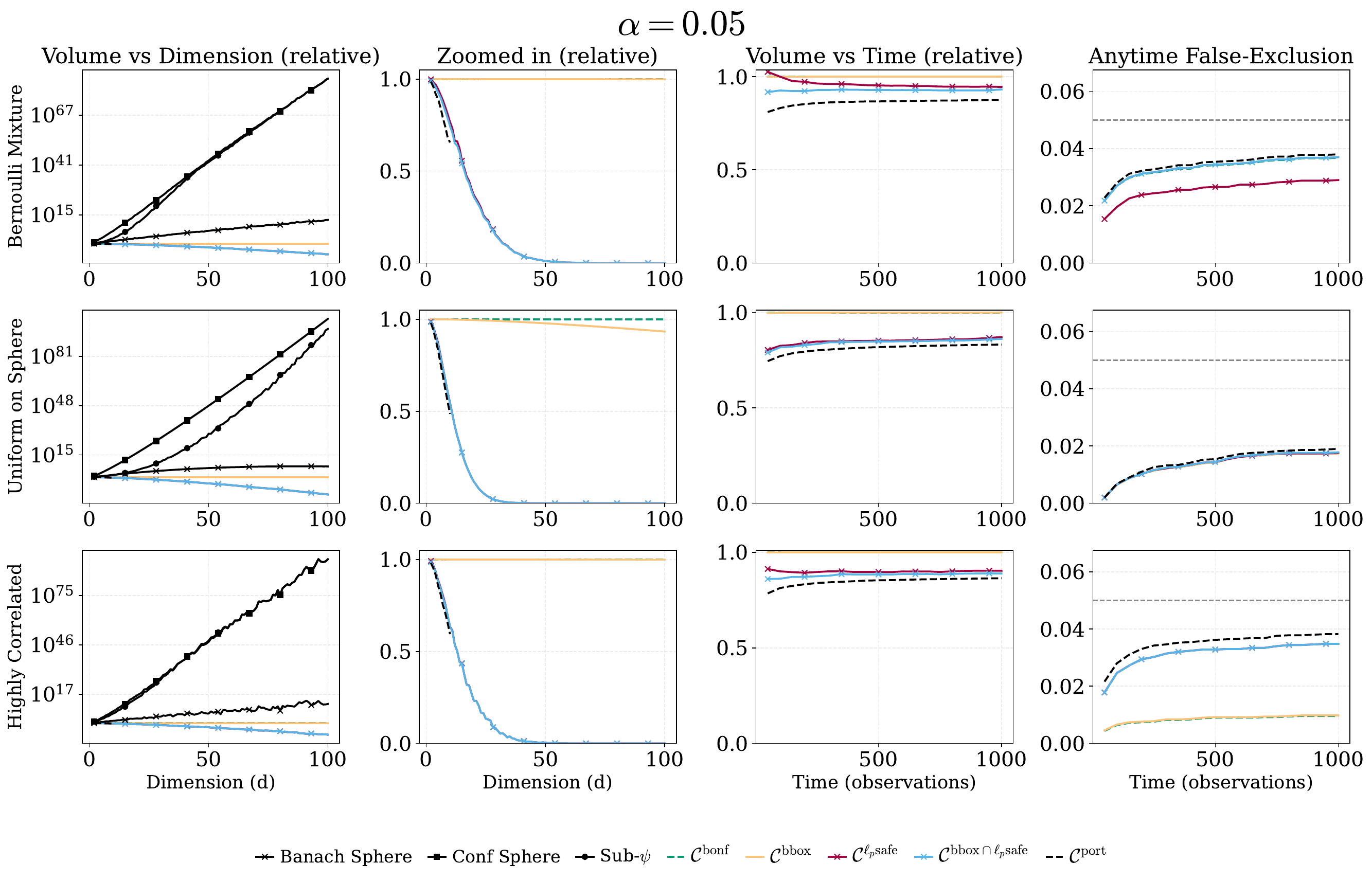}
	\caption{Results for $\alpha=0.05$. Rows and columns follow \Cref{AS:apdx:fig:alpha0p01}; the dashed horizontal line marks the nominal $\alpha$.}
	\label{AS:apdx:fig:alpha0p05}
\end{figure}
\begin{figure}
	\centering
	\includegraphics[width=\linewidth]{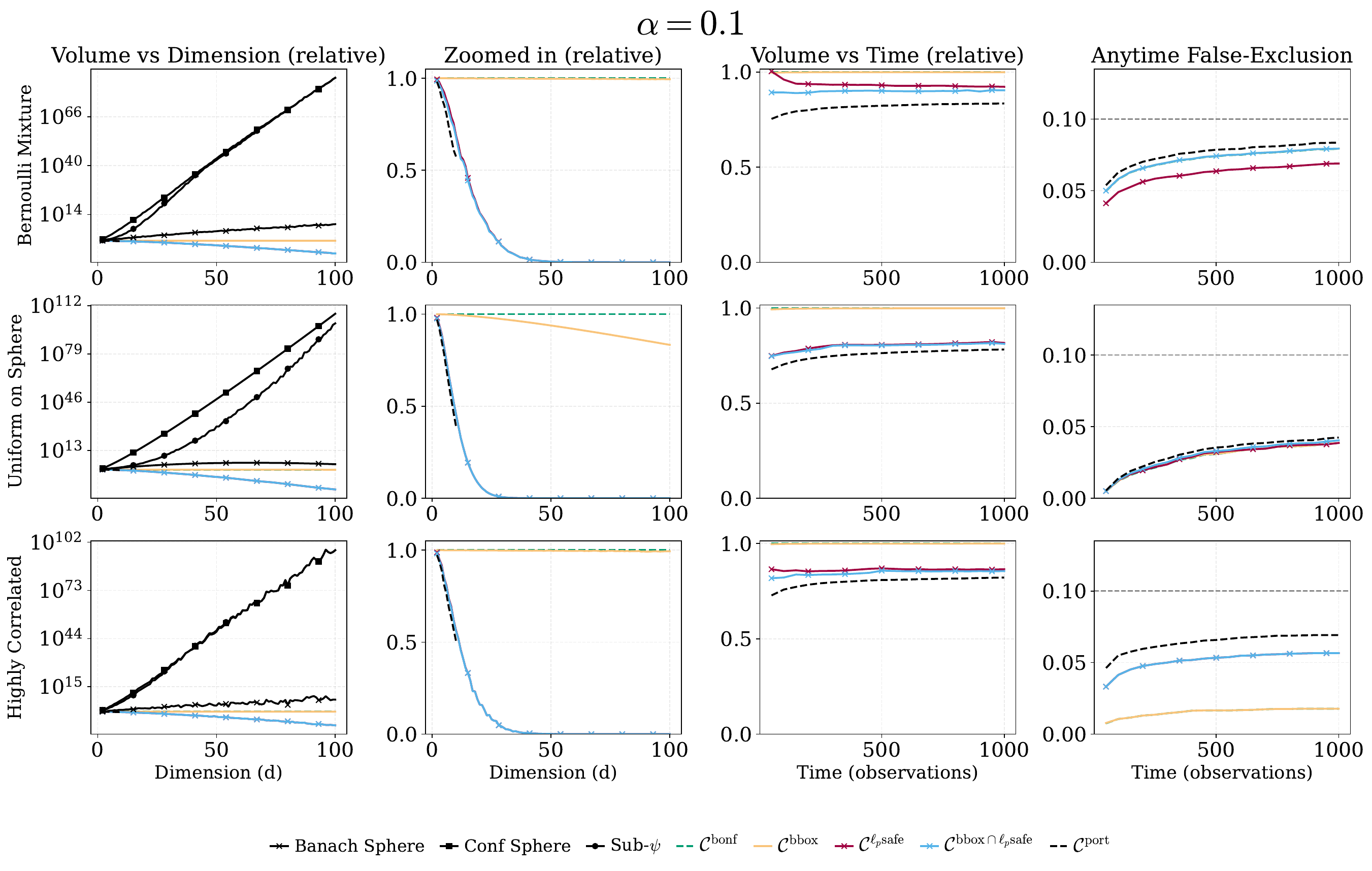}
	\caption{Results for $\alpha=0.10$. Rows and columns follow \Cref{AS:apdx:fig:alpha0p01}; the dashed horizontal line marks the nominal $\alpha$.}
	\label{AS:apdx:fig:alpha0p1}
\end{figure}
\FloatBarrier

\subsection{A/B-Testing Experiment: Synthetic Web Metrics}
\label{apdx:AS:ab_setup}
\textbf{Short Overview}
We simulate a sequential A/B test comparing two variants of a web shop to determine whether there is a statistically significant difference between them.  Let $\mathbf{y}_t^{(A)}$ denote the vector of observed metrics for website $A$ at step $t$, where $t$ may correspond to an individual user (and analogously for $B$), and define $\boldsymbol{\Delta}_t :=  \mathbf{y}_t^{(B)} - \mathbf{y}_t^{(A)}$. A CS then gives a region $\mathcal{C}_t$ for $\mathbb{E}[\boldsymbol{\Delta}_t]$ at every time $t$. If the interval for a particular coordinate lies entirely above or below $0$, we can conclude that the variants differ significantly on that metric. On the other hand, if the region becomes sufficiently small without excluding $0$ (measured through volume), we may conclude that there is no practically meaningful difference and stop collecting data.

\paragraph{Dataset generation.}
The simulation process is designed to mimic a real-world setting. Time is indexed by batches $t=1,\dots,T$. At each batch $t$, we sample $n_{t,A}$ independent sessions from variant $A$ and $n_{t,B}$ independent sessions from variant $B$, where $n_{t,A}$ and $n_{t,B}$ vary by $\pm 15\%$ around a nominal batch size $n=400$ (with a minimum of $100$ sessions per variant).

Each session is represented by a $D=10$-dimensional metric vector in $[0,1]^D$. The first eight coordinates mimic typical web-analytics metrics:
\textbf{(i)} time-on-site (seconds, clipped at $300$s and divided by $300$),
\textbf{(ii)} number of clicks (clipped at $20$ and divided by $20$),
\textbf{(iii)} number of pages viewed (clipped at $10$ and divided by $10$),
\textbf{(iv)} scroll depth (already in $[0,1]$),
\textbf{(v)} add-to-cart (Bernoulli),
\textbf{(vi)} purchase (Bernoulli),
\textbf{(vii)} non-bounce (Bernoulli; equals $1-\text{bounce}$ so that larger is better),
and \textbf{(viii)} 7-day return (Bernoulli).
The final two coordinates are added as independent noise variables sampled uniformly from $[0,1]$.
All coordinates are therefore bounded in $[0,1]$.

To create correlation between metrics, we draw a latent \emph{engagement} variable $E\in[0,1]$ per session. Conditional on $E$, continuous metrics are generated from simple parametric distributions (log-normal for time-on-site, Poisson for counts, Beta for scroll depth), and binary metrics are generated by Bernoulli draws whose success probabilities increase with $E$. This produces realistic correlations (e.g., sessions with larger $E$ tend to have more time, more clicks, deeper scroll, and higher conversion probability).

\Cref{fig:AS:ab_batch_distributions} visualizes what a single batch looks like (i.e., the $n=400$ samples). For each coordinate, we plot a histogram of samples for variants $A$ and $B$ within one batch (not the distribution of batch averages over time).  The latent engagement variable $E$  and its shifted treatment variant are shown in the top-left panel.

\paragraph{Two experiments: with effect vs.\ no effect.}
We run two variants of the simulation:
\begin{itemize}
\item \textbf{With effect:} variant $B$ is improved on a subset of metrics by shifting the engagement signal upward for those metrics only. Concretely, for the affected coordinates, we replace $E$ by $\tilde E = \min(E+\delta,1)$ when generating sessions from $B$, with shift $\delta=0.06$. This means that the two variants differ only in a small set of coordinates (time, clicks, scroll, and add-to-cart), while the remaining metrics (pages, purchase, non-bounce, return, and the two noise dimensions) are approximately unchanged.
\item \textbf{No effect:} we set $\delta=0$ and apply no coordinate-specific shift, so that variants $A$ and $B$ are generated from the same distribution.
\end{itemize}

\paragraph{From two variants to one bounded stream (batched pairing).}
Our CSs are defined for a single stream of bounded observations. We therefore combine the two variants within each batch. Let $\mathbf{y}_{t}^{(A)}, \mathbf{y}_{t}^{(B)}\in[0,1]^D$ denote the per-batch mean metric vectors for variants $A$ and $B$. We form the per-batch difference
\[
\boldsymbol{\Delta}_t \; = \; \mathbf{y}_{t}^{(B)} - \mathbf{y}_{t}^{(A)} \in[-1,1]^D,
\]
and rescale it coordinate-wise into $[0,1]$ by
\[
\mathbf{z}_t \;=\; \tfrac12(\boldsymbol{\Delta}_t+\mathbf{1}) \in [0,1]^D.
\]
Under the no-effect experiment, $\mathbb{E}[\mathbf{z}_t]=\tfrac12\mathbf{1}$, i.e., $\mathbb{E}[z_t^d]=0.5$ for every coordinate $d$. In the with-effect experiment, only the affected coordinates satisfy $\mathbb{E}[z_t^d]>0.5$; the unaffected and noise coordinates remain close to $0.5$.

We run all confidence region methods at level $\alpha=0.05$ on the sequential stream $\{\mathbf{z}_t\}_{t=1}^T$.

\paragraph{Stopping and decision rule.}
At each time $t$, let $\mathcal{C}_t$ be the confidence region for the mean of $\mathbf{z}_t$. We conclude coordinate-wise significance for metric $d$ as soon as $\mathcal{C}_t$ excludes $0.5$ on that axis:
\[
\underline{\mu}_{t,d} > 0.5 \quad \text{or}\quad \overline{\mu}_{t,d} < 0.5,
\]
where $[\underline{\mu}_{t,d},\overline{\mu}_{t,d}]$ denotes the coordinate interval for $\mu_d$. If no coordinate is significant, we stop and decide nothing once the region is sufficiently small in volume. Concretely, we stop without significance when $\sqrt[D]{V(\mathcal{C}_t)} \le r_{\mathrm{tol}}$.

In the main experiments, we use $D=10$, batch size $n=400$, $r_{\mathrm{tol}}=0.02$, and run up to $T=10{,}000$ batches. The main results are shown in~\Cref{tab:AS:ab_results}.

\paragraph{Ablation: sensitivity to batch size.}
We repeat the experiment for $n\in\{50,100,200,400,800\}$. As shown in \Cref{tab:AS:ab_batch_size}, the relative performance of the methods remains consistent. Interestingly, although larger batches reduce the variance of each batch mean, stopping times measured in batches change little: our methods stop after $170$--$181$ batches with an effect and $355$--$428$ batches without an effect. Consequently, smaller batches reach the same decisions with substantially fewer users in this experiment.

\begin{table*}[t]
	\centering
	\caption{\textbf{Sensitivity of A/B stopping outcomes to batch size.}
	Arrows report the results as the nominal batch size increases from $n=50$ to $n=800$; the intermediate values follow the same trend. Approximate users are reported per variant.}
	\label{tab:AS:ab_batch_size}
	\small
	\setlength{\tabcolsep}{7pt}
	\renewcommand{\arraystretch}{1.05}
	\begin{tabular}{@{}l *{4}{r@{\,$\to$\,}l}@{}}
		\toprule
		& \multicolumn{4}{c}{With effect} & \multicolumn{4}{c}{No effect} \\
		\cmidrule(lr){2-5}\cmidrule(lr){6-9}
		Method & \multicolumn{2}{c}{Stop$_t$} & \multicolumn{2}{c}{Approx. users} & \multicolumn{2}{c}{Stop$_t$} & \multicolumn{2}{c}{Approx. users} \\
		\midrule
		ConfSphere & $10{,}000$ & $10{,}000$ & $500{,}000$ & $8{,}000{,}000$ & $10{,}000$ & $10{,}000$ & $500{,}000$ & $8{,}000{,}000$ \\
		BanachSphere & $5{,}227$ & $5{,}046$ & $261{,}350$ & $4{,}036{,}800$ & $8{,}273$ & $8{,}036$ & $413{,}650$ & $6{,}428{,}800$ \\
		Sub-$\psi$ & $6{,}352$ & $5{,}243$ & $317{,}600$ & $4{,}194{,}400$ & $9{,}259$ & $6{,}650$ & $462{,}950$ & $5{,}320{,}000$ \\
		$\Bbonf$ & $171$ & $176$ & $8{,}550$ & $140{,}800$ & $428$ & $397$ & $21{,}400$ & $317{,}600$ \\
		$\Bbbox$ & $170$ & $175$ & $8{,}500$ & $140{,}000$ & $427$ & $396$ & $21{,}350$ & $316{,}800$ \\
		$\BellipS$ & $171$ & $176$ & $8{,}550$ & $140{,}800$ & $378$ & $355$ & $18{,}900$ & $284{,}000$ \\
		$\BbboxellipS$ & $\mathbf{170}$ & $\mathbf{175}$ & $\mathbf{8{,}500}$ & $\mathbf{140{,}000}$ & $\mathbf{378}$ & $\mathbf{355}$ & $\mathbf{18{,}900}$ & $\mathbf{284{,}000}$ \\
		\bottomrule
	\end{tabular}
\end{table*}

\begin{figure*}[ht]
	\centering
	\includegraphics[width=\textwidth]{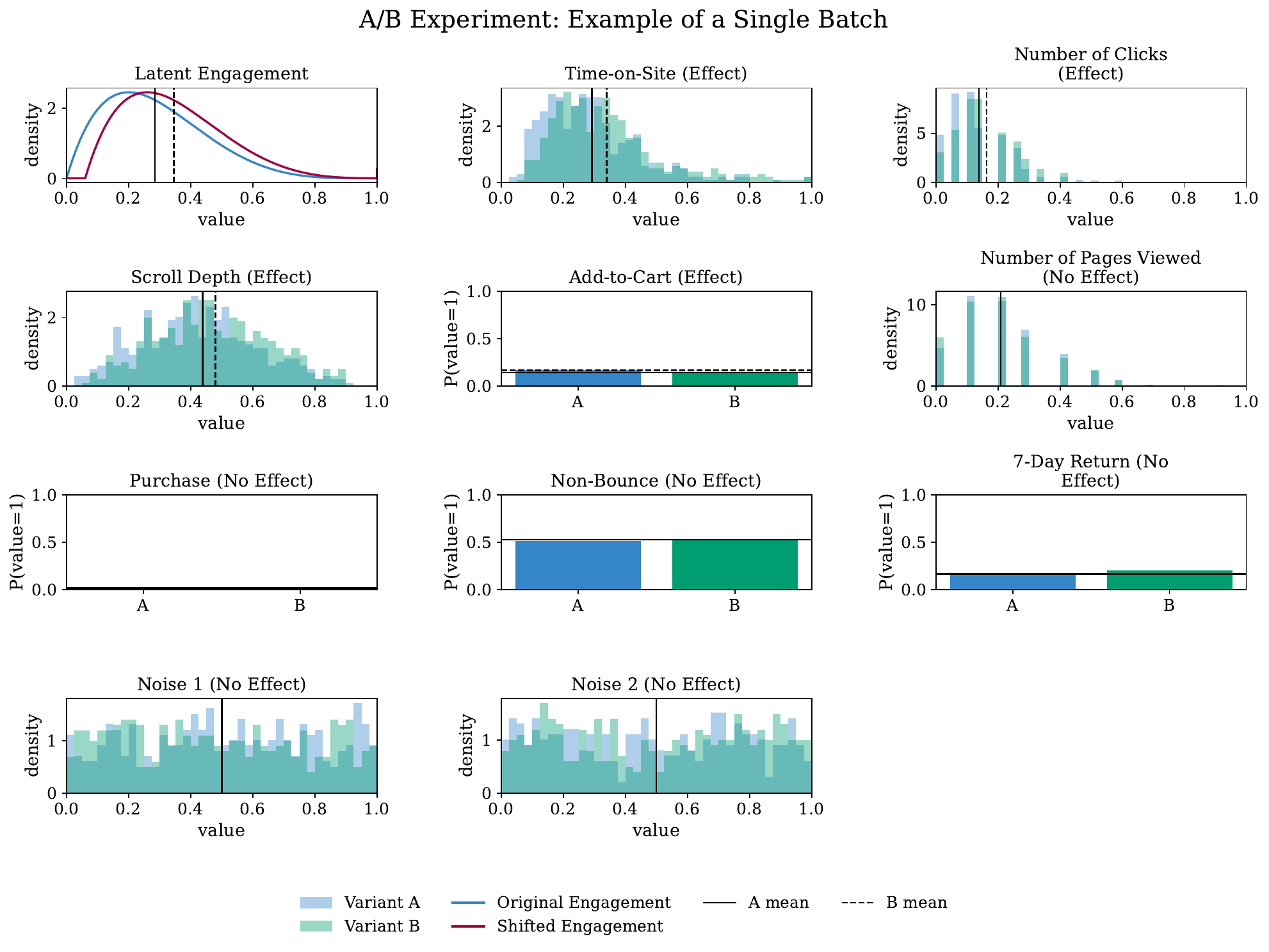}
	\caption{
    \textbf{Session-level distributions within one batch (synthetic A/B web metrics).}
		We show one representative batch from the \emph{with-effect} simulation. 
		Top-left:  A distribution of the latent engagement variable $E$, and its shifted variant. Other: Affected coordinates (time, clicks, scroll, add-to-cart) show a systematic shift from $A$ to $B$, while the remaining coordinates (pages, purchase, non-bounce, return-7d, noise$_1$, noise$_2$) show little to no shift. 
        }
	\label{fig:AS:ab_batch_distributions}
\end{figure*}

\FloatBarrier
\subsection{Model Comparison on Adult}
\label{apdx:AS:model_comparison}
\textbf{Short Overview}
We use the Adult dataset to illustrate how a multivariate CS can compare two classifiers while monitoring several performance quantities at once. Model $A$ is the baseline and model $B$ is the challenger. The target decision is whether $B$ improves overall accuracy; however, the same CS also reports uncertainty for the Brier score and for subgroup accuracies. The datasets' features include (among others) gender, education, and age. The task is to predict whether a person's income exceeds \$50{,}000.  \Cref{tab:AS:model_comparison_the_models_them_selves} summarizes the empirical performance of the two trained classifiers before applying any CS.

The setup is similar to the A/B experiment, where each $\boldsymbol{\Delta}_t$ is defined as the difference between two metric vectors, with $\mathbf{y}_t^{(A)}$ representing the metrics of model $A$ and $\mathbf{y}_t^{(B)}$ those of model B. The \textbf{main technical challenge lies in handling conditional accuracies}, such as $\Pr(f_B(\mathbf{x})=\text{correct}\mid \mathbf{x}=\text{woman})$ while preserving anytime validity. This is achieved by adding an extra coordinate to $\boldsymbol{\Delta}_t$ that represents the proportion of the relevant subgroup, for example, $\Pr(\mathbf{x}=\text{woman})$, and then dividing by this quantity.

\begin{table*}[t]
	\centering
	\caption{\textbf{Summary of the two classifiers on the Adult dataset.}
	Model $A$ is logistic regression and model $B$ is histogram gradient boosting. 
	The table reports the performance of both models on the tracked metrics, together with the mean gain of model $B$ over model $A$. For overall accuracy and subgroup-specific accuracies, the gain is computed as $B-A$.  For the Brier score, since lower values are better, we instead report the improvement as $A-B$, so that a positive value indicates that model $B$ outperforms model $A$.}
	\label{tab:AS:model_comparison_the_models_them_selves}
    \begin{tabular}{l c c c}
		\toprule
		Metric & Model $A$ (\%) & Model $B$ (\%) & Mean gain (\%) \\
		\midrule
        Overall accuracy                    & 85.09 & 87.30 & 2.24 \\
		Overall Brier score                 & 10.25 & 8.83  & 1.42 \\
		Female accuracy (conditional)       & 92.84 & 93.80 & 0.96 \\
		Male accuracy (conditional)         & 81.19 & 84.06 & 2.87 \\
		Age $< 40$ accuracy (conditional)   & 89.13 & 90.70 & 1.57 \\
		Age $\geq 40$ accuracy (conditional)& 79.89 & 82.98 & 3.09 \\
		\bottomrule
	\end{tabular}
	\vspace{2pt}
	\footnotesize
\end{table*}

\paragraph{Setup.}
The test set is split into $509$ batches of size $32$ (with the last batch smaller). For a batch $I_t$, we define the batch-level metric vector for model $A$ as
\begin{equation}
	\mathbf{y}_t^{(A)}
	=
	\frac{1}{|I_t|}
	\sum_{i\in I_t}
	\left(
	\begin{array}{c}
		\mathbf{1}\{ f_A(\mathbf{x}_i) = f_i \} \\
		(c_A(\mathbf{x}_i)-f_i)^2 \\
		\mathbf{1}\{ f_A(\mathbf{x}_i) = f_i \text{ and } \mathbf{x}_i \text{ is woman}\} \\
		\mathbf{1}\{ f_A(\mathbf{x}_i) = f_i \text{ and } \mathbf{x}_i \text{ is man}\} \\
		\mathbf{1}\{ f_A(\mathbf{x}_i) = f_i \text{ and } \mathbf{x}_i\text{'s age}<40\} \\
		\mathbf{1}\{ f_A(\mathbf{x}_i) = f_i \text{ and } \mathbf{x}_i\text{'s age}\ge 40\} \\
		\mathbf{1}\{\mathbf{x}_i \text{ is woman}\} \\
		\mathbf{1}\{\mathbf{x}_i\text{'s age}<40\}
	\end{array}
	\right)
	\in[0,1]^8.
	\label{AS:eq:model_comparision_defin_y_A}
\end{equation}

Here $f_A(\mathbf{x}_i)$ is model $A$'s predicted class and $c_A(\mathbf{x}_i)$ is its predicted probability for the true class; $\mathbf{y}_t^{(B)}$ is defined analogously. The first coordinate is accuracy, the second is Brier score, the next four are joint subgroup accuracies, and the final two are subgroup proportions. We only track the proportions of women and individuals with age $<40$, since the proportions of men and individuals with age $\ge 40$ are their complements.

We compare model $B$ to model $A$ by subtracting these batch metrics:
\[
\boldsymbol{\Delta}_t
=
\left(
\begin{array}{c}
y_{t,1}^{(B)}-y_{t,1}^{(A)} \\
y_{t,2}^{(A)}-y_{t,2}^{(B)} \\
y_{t,3}^{(B)}-y_{t,3}^{(A)} \\
\vdots \\
y_{t,6}^{(B)}-y_{t,6}^{(A)} \\
y_{t,7}^{(A)} \\
y_{t,8}^{(A)}
\end{array}
\right).
\]
Note that $A$ and $B$ are flipped in the Brier-score coordinate, since larger is better for accuracy but smaller is better for Brier score. Thus positive values in the first six coordinates favor model $B$. The gain coordinates lie in $[-1,1]$ and are normalized to $[0,1]$ before applying the CS, via $(\Delta_{t,d}+1)/2$. The subgroup-proportion coordinates already lie in $[0,1]$.

The expectations of the first six coordinates correspond to the overall accuracy gain, the Brier-score improvement, and the joint subgroup-accuracy gains, whereas the final two coordinates correspond to subgroup proportions.
For example, for the ``women gain'' coordinate (coordinate 3), we have
\[
\mathbb{E}[\Delta_{\mathrm{gain~woman}}]
=
\Pr(f_B(\mathbf{x})=f,\,\mathbf{x}\text{ is woman})
-
\Pr(f_A(\mathbf{x})=f,\,\mathbf{x}\text{ is woman}),
\]
and for the woman subgroup proportion (coordinate 7), we have
\[
\mathbb{E}[\Delta_\mathrm{prop~woman}]
=
\Pr(\mathbf{x}\text{ is woman}).
\]
Therefore, the conditional accuracy gain is obtained by dividing by the subgroup probability:
\begin{equation}
    \Pr(f_B(\mathbf{x})=f\mid \mathbf{x}\text{ is woman})
-
\Pr(f_A(\mathbf{x})=f\mid \mathbf{x}\text{ is woman})
=
\frac{\mathbb{E}[\Delta_{\mathrm{gain~woman}}]}{\Pr(\mathbf{x}\text{ is woman})}. \label{eq:AS:apdx:ratio_joint_women_gain_div_women_ratio}
\end{equation}
In practice, the quantities in~\eqref{eq:AS:apdx:ratio_joint_women_gain_div_women_ratio} are unknown, so we replace them by their CS bounds. Let
\[
[\underline{\mu}_{\mathrm{gain~woman}}, \overline{\mu}_{\mathrm{gain~woman}}]
:=
\mathcal{I}_t(\Delta_{\mathrm{gain~woman}})
\]
denote the CS interval for $\mathbb{E}[\Delta_{\mathrm{gain~woman}}]$, and analogously let
\[
[\underline{\mu}_{\mathrm{prop~woman}}, \overline{\mu}_{\mathrm{prop~woman}}]
:=
\mathcal{I}_t(\Delta_{\mathrm{prop~woman}})
\]
denote the interval for $\Pr(\mathbf{x}\text{ is woman})$. Since the denominator is positive, we report the interval for~\eqref{eq:AS:apdx:ratio_joint_women_gain_div_women_ratio} as
$$
\left[
\min_{\substack{a\in\{\underline{\mu}_{\mathrm{gain~woman}},\overline{\mu}_{\mathrm{gain~woman}}\}\\[1pt]
b\in\{\underline{\mu}_{\mathrm{prop~woman}},\overline{\mu}_{\mathrm{prop~woman}}\}}}
\frac{a}{b},
\;
\max_{\substack{a\in\{\underline{\mu}_{\mathrm{gain~woman}},\overline{\mu}_{\mathrm{gain~woman}}\}\\[1pt]
b\in\{\underline{\mu}_{\mathrm{prop~woman}},\overline{\mu}_{\mathrm{prop~woman}}\}}}
\frac{a}{b}
\right].
$$
For men, we use the complement $1-\mathcal{I}_t(\Delta_{\mathrm{prop~woman}})$. The same construction is used for the other subgroup gains. Since this is a deterministic transformation of a jointly valid CS, it preserves anytime validity.

 \paragraph{Training Setup.}
Both models are implemented using sklearn's implementation. For logistic regression, categorical features are one-hot encoded and numeric features are standardized; we use \texttt{LogisticRegression(solver='lbfgs', max\_iter=2000, random\_state=7)}. For the boosted-tree model, categorical features are ordinal-encoded with unknown categories mapped to $-1$, and we use \texttt{HistGradientBoostingClassifier(max\_depth=6, learning\_rate=0.05, max\_iter=300, min\_samples\_leaf=20, random\_state=7)}.

\Cref{fig:AS:model_comparison_intervals_all} displays the resulting intervals. The hedged-capital regions all detect a positive overall-accuracy gain after roughly $356$ batches, while ConfSphere, BanachSphere, and Sub-$\psi$ remain inconclusive after all $509$ test batches.

\begin{figure*}[ht]
	\centering
	\includegraphics[width=\textwidth]{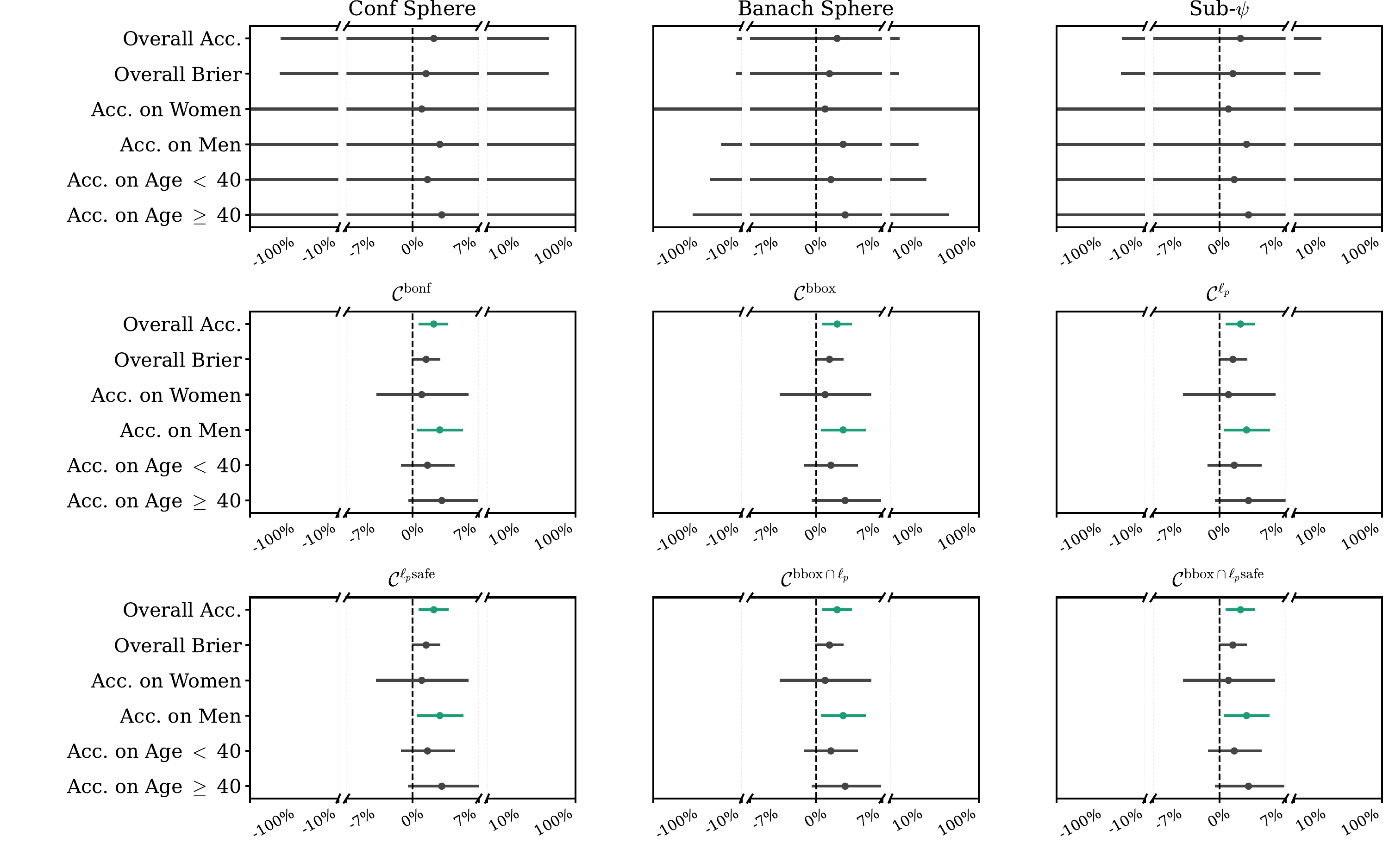}
	\caption{
    \textbf{Model-comparison intervals on the Adult dataset.}
    Each panel shows one CS method and each row shows one reported metric. Intervals represent the gain for that particular metric, with the dashed vertical line marking zero. Positive values favor the histogram-gradient-boosting challenger (model $B$) over the logistic-regression baseline (model $A$). The hedged-capital regions certify a positive overall-accuracy gain, whereas the sphere and normalized-ellipsoid baselines still include zero after all available test batches.
	}
	\label{fig:AS:model_comparison_intervals_all}
\end{figure*}

\FloatBarrier
\section*{Part IV: Varia}
\addcontentsline{toc}{part}{Part IV: Varia}
\section{Additional Related Work}
In the context of sequential testing for multi-armed bandits, \citet{cho_peeking_with_PEAK} similarly argue for combining wealth processes by averaging rather than through Bonferroni corrections.
Their goal, however, is different: they design sequential tests for bandit decision problems such as best-arm and threshold identification. In doing so, they modify the stake choices (i.e., $\lambda_t(\cdot)$ in~\Cref{sec:AS:one_dim_conf_seq}) so that the corresponding rejection region is convex, enabling the use of convex optimization tools. In contrast, our focus is on explicitly describing the resulting confidence region through tractable outer approximations that make quantities such as volume or maximum width easy to compute. Convexity alone still does not make these properties straightforward to quantify. Our constructions are also more general, as they apply to any legal stake choice $\lambda_t$.

In the context of trustworthy machine learning, Bayesian neural networks (BNNs) provide a principled framework for representing predictive uncertainty, which is particularly valuable in high-stakes applications. Their predictions, however, are commonly approximated using a fixed number of Monte Carlo samples per input, without controlling the resulting error due to the finite number of samples \citep{daxberger2021laplace,DBLP:conf/nips/MaddoxIGVW19,DBLP:conf/icml/GalG16,denoodt2025varbasedsmoothing,krause2025probabilisticAiBook}. Future work could investigate whether CSs could be applied to control this error as well.

\FloatBarrier

\section{Broader Impacts} \label{apdx:AS:broader_impacts}
This work aims to improve transparency and reliability in sequential statistical inference, which can support more trustworthy decision-making in applications such as online experimentation and model evaluation. By enabling valid anytime monitoring, the proposed methods can reduce unnecessary data collection and help practitioners detect effects (or lack of effects) earlier. We foresee no direct negative societal impacts from the methods themselves, as they are purely statistical tools that do not interact with individuals or make decisions on their own. However, as with any statistical method, there is a risk of misuse or misinterpretation.


\end{document}